\documentclass[letterpaper]{article} 
\usepackage{aaai2026}  
\usepackage{times}  
\usepackage{helvet}  
\usepackage{courier}  
\usepackage[hyphens]{url}  
\usepackage{graphicx} 
\usepackage{natbib}  
\usepackage{caption} 
\usepackage{algorithm}
\usepackage{algorithmic}

\usepackage{amsmath, amssymb, amsfonts, bm, amsthm}
\usepackage{pifont}
\usepackage{amssymb}
\usepackage{subfigure, subcaption}
\usepackage{booktabs, multirow}
\usepackage{siunitx}
\usepackage{colortbl,xcolor}

\newtheorem{theorem}{Theorem}        
\newtheorem{lemma}{Lemma}

\newtheorem{remark}{Remark}
\renewenvironment{proof}[1][Proof: ]{\noindent \textit{#1}}{\qed\medskip}

\definecolor{bestbg}{RGB}{230,242,255}

\newcommand{\B}[1]{\multicolumn{1}{>{\columncolor{bestbg}}r}{#1}}

\newcommand{\yes}{$\checkmark$}
\newcommand{\no}{$\times$}

\usepackage{newfloat}
\usepackage{listings}
\DeclareCaptionStyle{ruled}{labelfont=normalfont,labelsep=colon,strut=off} 
\floatstyle{ruled}
\newfloat{listing}{tb}{lst}{}
\floatname{listing}{Listing}
\title{LPPG-RL: Lexicographically Projected Policy Gradient Reinforcement Learning with Subproblem Exploration}
\author{
    Ruiyu Qiu\textsuperscript{\rm 1}\equalcontrib, Rui Wang\textsuperscript{\rm 2}\equalcontrib, Guanghui Yang\textsuperscript{\rm 1,\rm 3}, Xiang Li\textsuperscript{\rm 1}, Zhijiang Shao\textsuperscript{\rm 1,\rm 3}\thanks{Corresponding author.}
}
\affiliations{
    \textsuperscript{\rm 1}College of Control Science and Engineering, Zhejiang University, Hangzhou, China \\
    \textsuperscript{\rm 2}School of Mechanical Engineering and Mechanics, Ningbo University, Ningbo, China \\
    \textsuperscript{\rm 3}Huzhou Institute of Industrial Control Technology, Huzhou, China \\

    ryqiu@zju.edu.cn,wangrui2@nbu.edu.cn,\{ghyang,li.xiang,szj\}@zju.edu.cn
}

\begin{document}

\maketitle

\begin{abstract}
Lexicographic multi-objective problems, which consist of multiple conflicting subtasks with explicit priorities, are common in real-world applications. Despite the advantages of Reinforcement Learning (RL) in single tasks, extending conventional RL methods to prioritized multiple objectives remains challenging. In particular, traditional Safe RL and Multi-Objective RL (MORL) methods have difficulty enforcing priority orderings efficiently. Therefore, Lexicographic Multi-Objective RL (LMORL) methods have been developed to address these challenges. However, existing LMORL methods either rely on heuristic threshold tuning with prior knowledge or are restricted to discrete domains. To overcome these limitations, we propose Lexicographically Projected Policy Gradient RL (LPPG-RL), a novel LMORL framework which leverages sequential gradient projections to identify feasible policy update directions, thereby enabling LPPG-RL broadly compatible with all policy gradient algorithms in continuous spaces. LPPG-RL reformulates the projection step as an optimization problem, and utilizes Dykstra's projection rather than generic solvers to deliver great speedups, especially for small- to medium-scale instances. In addition, LPPG-RL introduces Subproblem Exploration (SE) to prevent gradient vanishing, accelerate convergence and enhance stability. We provide theoretical guarantees for convergence and establish a lower bound on policy improvement. Finally, through extensive experiments in a 2D navigation environment, we demonstrate the effectiveness of LPPG-RL, showing that it outperforms existing state-of-the-art continuous LMORL methods.
\end{abstract}

\begin{links}
    \link{Code}{https://github.com/qiuruiyu/LPPG-RL}
\end{links}

\section{Introduction}
Reinforcement learning (RL) has achieved remarkable success in diverse domains, including games \cite{silver2014, mnih2015a}, robotics \cite{tang2025}, and autonomous driving \cite{coelho2024}. In real-world applications, complex tasks are naturally decomposed into multiple subtasks, often organized according to explicit, predefined priorities. For example, in autonomous driving, the lane-changing task can be divided into subtasks with descending priorities: collision avoidance, lane keeping, and speed regulation. These subtasks may conflict with each other, making the overall task challenging, even though each subtask is individually simple.

In such scenarios, safety is typically assigned the highest priority. Prior work on Safe RL formulates constraints within the framework of Constrained Markov Decision Processes (CMDPs) \cite{altman2021}, where the agent maximize reward while minimizing cost. However, existing approaches exhibit significant limitations. First, policy optimization is generally conducted over an unordered set of constraints, which results in infeasibility when constraints are conflicting. Second, traditional Safe RL methods are not designed to optimize multiple objectives concurrently.

If the safety cost is treated as an additional objective, Safe RL can be regarded as a special case of Multi-Objective RL (MORL). MORL seeks to identify optimal policies that balance trade-offs among multiple objectives \cite{qiu2025}, typically by approximating the Pareto front. Under this perspective, Safe RL corresponds to selecting a Pareto-optimal policy that satisfies constraints. Although this concept is appealing, MORL methods usually aggregate objectives into a scalar, disregarding explicit priorities. Moreover, extending MORL approaches to densely approximate the Pareto front is computationally expensive, as it requires repeated training with varying weights \cite{parisi2014}. When priorities are predefined, practitioners must manually assign extreme weights and search through the Pareto front, resulting in sample inefficiency and offering no guarantees that higher-priority objectives are maintained \cite{liu2025}.

To address these limitations, we focus on directly finding a feasible policy by leveraging Lexicographic Multi-Objective Reinforcement Learning (LMORL). Building upon the framework introduced in \cite{skalse2022}, we extend it from a policy gradient perspective. Specifically, we propose the Lexicographically Projected Policy Gradient RL \textbf{(LPPG-RL)}, which sequentially searches for a feasible policy update by identifying the intersection of high-priority gradients, thereby enforcing the priority structure. Furthermore, we introduce several key improvements to enhance both practicality and efficiency. First, recognizing the computational complexity associated with large network parameters, we formulate an optimization problem and employ a problem-specific iterative projection method rather than relying on generic solvers. This approach significantly accelerates gradient projection, particularly in small- to medium-scale problems. Second, to avoid excessive hyperparameters and gradient vanishing, we encourage Subproblem Exploration (SE), which expedites the training and convergence of individual subtasks. Importantly, our method is naturally applicable to continuous spaces with standard policy gradient RL algorithms such as Proximal Policy Optimization (PPO) \cite{schulman2017} and Soft Actor-Critic (SAC) \cite{haarnoja2018}, whereas most previous methods are restricted to discrete domains. Finally, we provide theoretical analysis of the convergence and the policy improvement lower bound for our algorithm.

We evaluate our algorithm using PPO (LPPG-PPO) on variants of a 2D navigation environment. We compare the computational efficiency of our projection method with generic solvers within CVXPY \cite{diamond2016cvxpy}, and benchmark our performance against two state-of-the-art LMORL methods for continuous spaces: Lexicographic Projection Algorithm (LPA) \cite{tercan2024} and Lexicographic PPO (LPPO) \cite{skalse2022}. Experiments show that our method achieves higher efficiency with equivalent results when the number of objectives is below 100, and LPPG-RL outperforms baseline methods in both performance and stability. Ablation studies further demonstrate the effectiveness of our proposed enhancements.

Our contributions of this work are three-fold:
\begin{enumerate}
    \item We propose LPPG-RL, a novel LMORL algorithm for continuous spaces that guarantees strict priority ordering while training a policy end-to-end. With an ultra-light Dykstra's projection core, we achieve up to 20$\times$ greater computational efficiency compared to generic solvers.
    \item We introduce Subproblem Exploration (SE), a uniform rollout scheduler that dynamically focuses on the active priority layer, requiring no handcrafted weights or prior knowledge. SE achieves superior performance and training stability compared to existing baseline methods.
    \item We provide theoretical guarantees of convergence, including a principled stepsize bound that ensures monotonic policy improvement, and establish a lower bound on performance improvement at each update.
\end{enumerate}

\begin{figure*}[!t]
\centering
\includegraphics[width=0.9\linewidth]{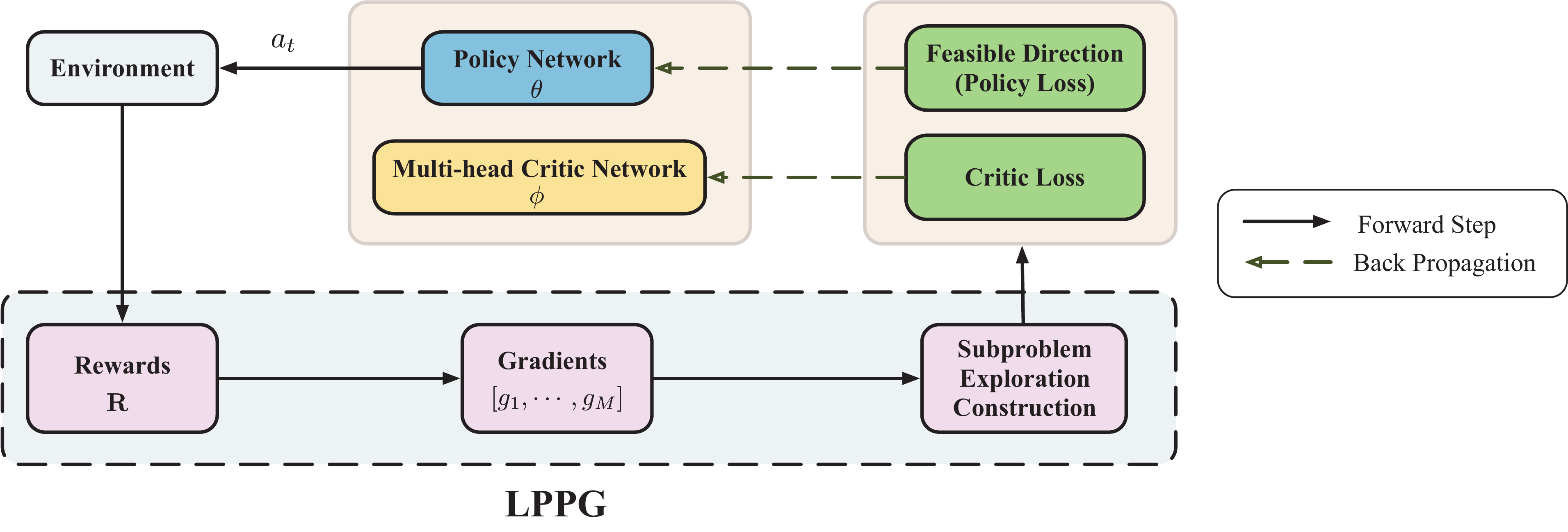} 
\caption{The overall workflow of LPPG-RL. The agent consists of a policy network and a multi-head critic network, where the policy network learns the global optimal lexicographic action. The lower part is LPPG. Policy gradients are calculated each time with rollouts, and a subproblem is drawn to get a feasible update direction so that each subtask can be trained uniformly.}
\label{fig:workflow}
\end{figure*}

\section{Related Works}
\subsection{Safe RL} 
Safe RL is typically formulated as a constrained optimization problem, and various techniques have been developed to handle safety constraints. Primal-dual (Lagrangian) methods are widely used by introducing penalty terms with tunable coefficients \cite{tessler2018, stooke2020}. However, selecting appropriate coefficients remains challenging in practice. Constrained Policy Optimization (CPO) \cite{achiam2017} addresses this issue by solving a trust-region problem at each step, but relies on accurate cost estimation and brings high computational complexity. To reduce this burden, Projection-based CPO (PCPO) \cite{yang2020c} simplifies the process by directly projecting policy gradients onto the feasible constraint set. In addition, explicit Lyapunov functions \cite{chow2018, chow2019} and control barrier functions \cite{dawson2023, deng2024} can be employed to provide formal safety guarantees. Despite the significant progress in Safe RL, explicitly ordered constraints remain largely unsolved. In such cases, constraints must be handled sequentially, which substantially increases computational costs.

\subsection{Multi-Objective RL}
MORL methods can be categorized into single-policy and multi-policy approaches. Single-policy methods combine multiple objectives into a single scalar reward, using weighted sums \cite{roijers2014, abels2018}, which require significant domain knowledge and extensive tuning \cite{vanmoffaert2013a}. Multi-policy methods approximate the Pareto front by repeatedly applying single-policy methods with different weight configurations \cite{mossalam2016, zuluaga2016}. To improve efficiency, evolutionary algorithms \cite{xu2020a} and Pareto Set Learning \cite{liu2025} have also been explored to help approximate the Pareto front. However, when explicit priorities are required, these methods often fall short. Lexicographic optimality demands filtering policies along a densely sampled and accurate Pareto front, which leads to computational and scalability issues as the number of objectives increases \cite{parisi2014, pirotta2015}. Therefore, existing MORL methods excel at balancing objectives but remain inefficient for strictly prioritized problems. 

\subsection{Lexicographic Multi-Objective RL}
LMORL was first explored through lexicographic optimization by \cite{gabor1998}. Later, a unified thresholded framework with theoretical analysis was proposed for both value-based and policy-based RL methods \cite{skalse2022}. However, the value-based framework—followed by works in autonomous driving\cite{li2019a, zhang2023}—filters actions by enumerating all candidates in each subtask, which restricts applications to discrete action spaces. \cite{rietz2024} extended the policy-based framework to Q-value, but this approach can fail when subtask Q-values differ in scale, and actually relies on a costly grid search to estimate continuous values. Recently, \cite{tercan2024} introduced Lexicographic Projection Algorithm (LPA) to address continuous problems. However, to our understanding, both LPA and the original policy-based framework requires setting  optimization and criteria thresholds according to the reward scale, demanding prior knowledge and repeated tuning, which makes methods quite sensitive and less practical. In contrast, we overcome these limitations by recasting lexicographic learning as a sequence of constrained gradient projections solved with an efficient Dykstra–style update, and by introducing SE to automatically balance subtask convergence instead of tuning thresholds with prior knowledge. 
 
\section{Problem Definition and Preliminaries}
Suppose there are $M$ subtasks, including objectives and constraints, defined in a lexicographic order and denoted by the task set $\mathcal{K} = \{K_1, \cdots, K_M\}$, where the subscript indicates priority (with $K_1$ the highest priority).

An LMORL problem can be formulated with Multi-Objective Markov Decision Processes (MOMDPs) \cite{skalse2022}, represented as $\mathcal{M} := \langle S, A, \mathbf{R}, P, \bm{\gamma} \rangle$. Here, $S$ and $A$ denote the state and action spaces, respectively, and $P$ is the state transition probability. The reward function $\mathbf{R}: S \times A \times S \rightarrow \mathbb{R}^M$ produces a rewards vector $[r_1, \cdots, r_M]^T$ of $\mathcal{K}$. $\bm{\gamma}$ is the vector of discount factors.

Following the definition, lexicographic relationships between subtasks are enforced through constraints. Formally, a policy $\pi$ is said to be lexicographically $\epsilon$-optimal within the policy set $\Pi_i$, if subtasks are finished and it satisfies:
\begin{equation} \label{eq:lexico pi def}
\Pi_{i+1} := \left\{ \pi \in \Pi_{i} \mid \max_{\pi' \in \Pi_{i}} J_i(\pi') - J_i(\pi) \leq \epsilon_i \right\}
\end{equation}
where $\Pi_0$ is the initial policy set and $\Pi_i$ is the feasible policy set for subtasks $\{K_1, \ldots, K_i \}$. When optimizing a subtask $K_{i+1}$, all higher priorities $\{K_1, \cdots, K_i\}$ must remain sufficiently close to their previous optima within a threshold $\epsilon_i$. For any subtask index $i$ and policy $\pi$, $J_i(\pi)$ is the performance criterion, such as the value function. It is important to note that, directly computing $\Pi_{i+1}$ via Equation~\ref{eq:lexico pi def} is generally intractable, especially in continuous domains.

\section{LPPG-RL}
In this section, we present LPPG-RL, an efficient end-to-end framework for solving LMORL problems in continuous spaces. The overall workflow of LPPG-RL is illustrated in Figure~\ref{fig:workflow}. Given an LMORL problem with an ordered subtask set $\mathcal{K}$, our framework proceeds as follows:  

\begin{itemize}
    \item \textbf{Actor-Critic Architecture}: We adopt the actor-critic structure, where a policy network generates the final feasible action and a multi-head critic network estimates the value for each subtask.
    \item \textbf{Data Collection}: Rollouts are collected following standard RL procedures, with the reward for each subtask stored separately.
    \item \textbf{Subtask Gradients}: The policy gradient for each individual subtask is computed separately.
    \item \textbf{Subproblem Extraction}: To promote systematic exploration and expedite learning, SE is introduced as a uniform rollout scheduler to extract a subproblem and allocate trajectories evenly across priority levels, reducing the chance of local optima.
    \item \textbf{Feasible Direction and Update}: A feasible policy update direction $d^*$ is computed based on the SE scheduler to update the policy. Instead of invoking generic solvers,  we employ a Dykstra projection loop, an ultra-light iterative algorithm that provably converges to the optimum with significantly reduced computational burden.
\end{itemize}
In the following, we first construct the LMORL problem with policy gradients. Next, we discuss critical challenges during training and introduce techniques to address issues and enhance performance. Finally, we instantiate our framework with LPPG-PPO and provide theoretical analysis.

\subsection{Projected Gradient for Lexicographic Tasks}
Policy gradient methods in standard RL aim to maximize the expected cumulative reward by repeatedly estimating the gradient, which can be formulated as \cite{schulman2015}:
\begin{equation} \label{eq:policy gradient form}
g = \mathbb{E}\left[
\sum_{t=0}^\infty \Psi_t \nabla_\theta \log \pi_{\theta}(a_t \mid s_t)
\right]
\end{equation}
where $\Psi_t$ can be the reward, Q-value, advantage, etc. A key advantage of using gradients is that it requires only the evaluation and differentiation of the policy log-density rather than enumerating all possible actions. This property facilitates a natural extension to multi-objective settings by directly operating on subtask-specific gradients. Since all gradients share a same parameter space, lexicographic order over the objectives can be straightforwardly imposed.

For each subtask $K_i \in \mathcal{K}$, let $J_i$ denote its optimization objective and $g_i = \nabla_\theta J_i$ the corresponding gradient. At each iteration, the policy parameters update as follows:
\begin{equation} \label{eq:network update}
\theta_{k+1} = \theta_{k} + \alpha d
\end{equation}
where $\alpha > 0$ is the learning rate and $d$ is the chosen update direction. Then, interpreting Equation~\ref{eq:lexico pi def} from a policy gradient perspective, we require that the objective $J_i$ does not decrease when optimizing for lower-priority subtasks. This can be captured by a first-order (linear) approximation:
\begin{equation} \label{eq:gd condition}
\begin{aligned}
& && J_i(\theta_{k+1}) \approx J_i(\theta_k) + \alpha g_i^Td \geq J_i(\theta_k) - \alpha\epsilon_i \\ 
& \Rightarrow && g_i^T d \geq -\epsilon_i \quad (\epsilon_i=0 \text{ by default)}\
\end{aligned}
\end{equation}
where $\theta_k$ is the network parameter at iteration $k$, and $\epsilon_i \geq 0$ serves as a small threshold for allowable degradation and helping convergence in the $i$-th subtask \cite{rietz2024}. Clearly, if the chosen update direction $d$ has a negative component along the original gradient $g_i$, the corresponding performance may deteriorate. Therefore, to achieve a lexicographically optimal update, we seek a direction $d$ that satisfies as many conditions in Equation~\ref{eq:gd condition} as possible. Note that, we set $\epsilon_i=0$ by default and rely on other techniques to maintain convergence. The threshold $\epsilon_i$ is only adjusted in specific cases where relaxing the constraint to address practical considerations.

Let $\mathcal{C}_M$ denote the intersection of half-spaces defined by Equation~\ref{eq:gd condition} for $i = \{1, \cdots, M \}$:
\begin{equation} \label{eq:direction set def}
\mathcal{C}_M = \bigcap_{i=1}^{M} \left\{ d \mid g_i^T d \ge -\epsilon_i \right\}
\end{equation}
Here, $\mathcal{C}_M$ forms a (polyhedral) cone, and any direction $d$ within this cone is considered feasible. As long as the intersection set $\mathcal{C}_M$ is not reduced to the trivial space $\{\bm{0}\}$, there always exists $d \in \mathcal{C}_M$ serving as a feasible update direction.

In practice, instead of simply selecting an arbitrary feasible direction, we formulate the following convex optimization problem to determine the optimal direction $d^*$. Specifically, we seek the direction closest to the gradient of the lowest-priority subtask, $g_M$, with all higher-priority subtasks within their prescribed thresholds:
\begin{equation}
    \begin{aligned}
        &\min_{d}          && \| d - g_M \|^2 \\
        &\text{ s.t.} && d \in \mathcal{C}_M
    \end{aligned}
\label{eq:min d g}
\end{equation}

An illustrative example of the optimal solution is shown in Figure~\ref{fig:optimal projection}.

\subsection{Existing Problems and Enhancements}

\begin{figure}[!t]
\centering
\includegraphics[width=0.9\linewidth]{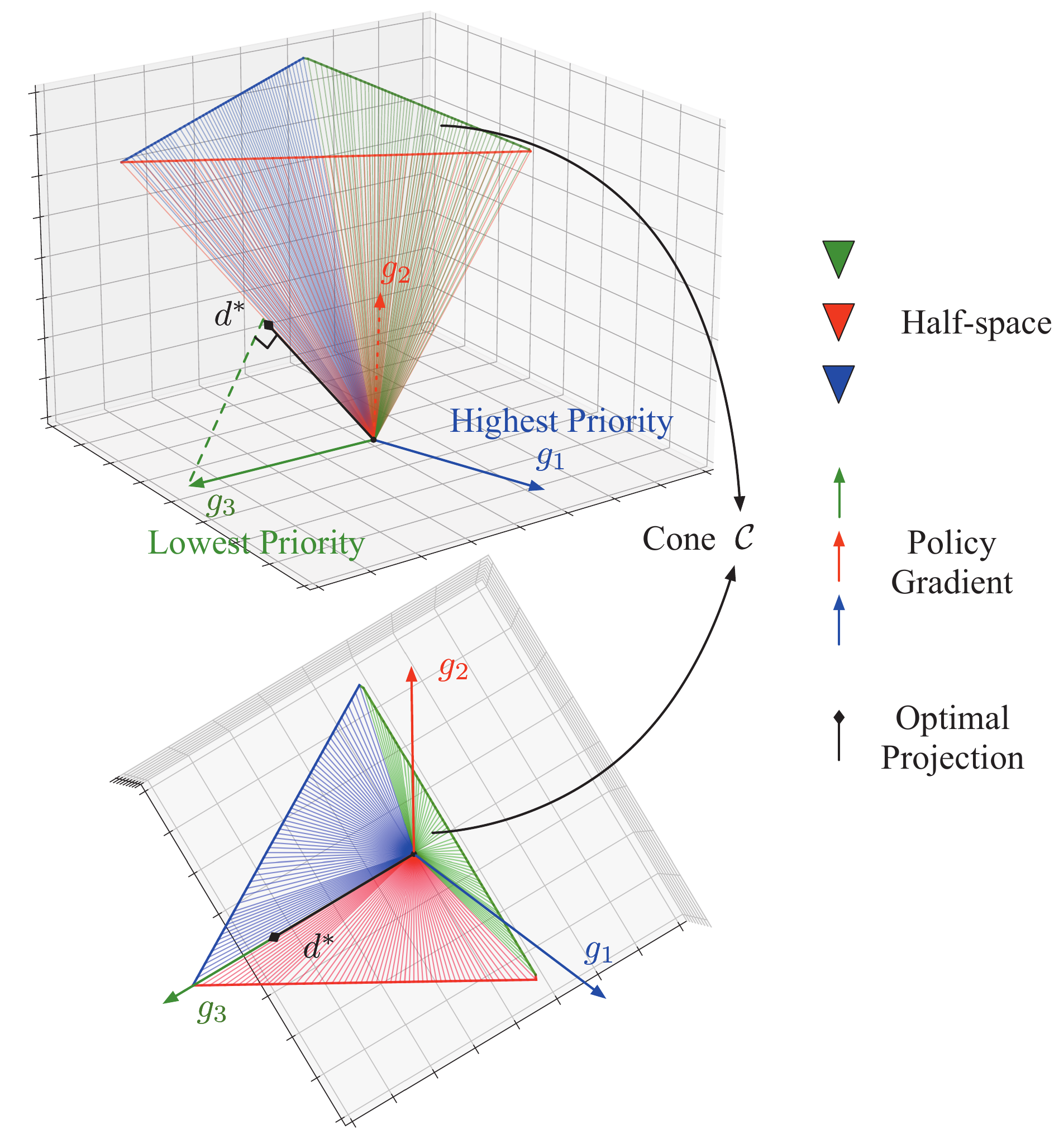}
\caption{Illustration of the optimal solution of Equation~\ref{eq:min d g} with three gradients. The blue, red and green vectors represent the policy gradients $g_1,g_2,g_3$, ordered from high to low priority, serving as the normal vectors of the respective half-spaces. The intersection of these half-spaces forms a cone. The optimal solution, shown as the black arrow $d^*$, is the vector within the intersection $\mathcal{C}$ that is closest to $g_3$.}
\label{fig:optimal projection}
\end{figure}

\subsubsection{Optimization Efficiency}
One common approach to directly solve problem~\ref{eq:min d g} is to employ standard convex optimization solvers, such as those provided in CVXPY. However, repeatedly invoking these solvers during update steps can incur significant computational burden.

To address this issue, we propose replacing general-purpose solvers with \textbf{Dykstra’s projection algorithm} \cite{boyle1986}, an efficient iterative method that, to our knowledge, has not been previously applied in LMORL. Dykstra’s projection extends the classical Sequential Orthogonal Projection (SOP) method. SOP is typically employed to find a feasible $d \in \mathcal{C}$ by iteratively removing components that violate the constraints, as follows:
\begin{equation} \label{eq:sop}
    d_{i+1} = 
        \begin{cases}
            d_i - \dfrac{g_i^Td_i}{\|g_i\|^2}g_i \quad & \text{if} \quad  g_i^Td_i\leq -\epsilon_i\\ 
            d_i & \text{else}
        \end{cases}
\end{equation}
In contrast to SOP, Dykstra's projection algorithm iteratively projects the gradient $g_M$ onto the intersection set, introducing auxiliary variables $r_i$ to track the residuals from previous projections. This mechanism ensures the final convergence to the optimal solution of Proposition~\ref{eq:min d g}. An iterative update proceeds as follows:
\begin{equation} \label{eq:dykstra iteration}
    \begin{aligned}
        & d_{i+1}^{(t)} = d_i^{(t)} - \dfrac{g_i^T(d_i^{(t)}+r_i^{(t)})}{\|g_i\|^2}g_i \\ 
        & r_{i}^{(t)} = d_i^{(t)} + r_{i}^{(t)} - d_{i+1}^{(t)}
    \end{aligned}
\end{equation}
where $t$ is the iteration number. It has been proven that the sequence $\{ d_i^{(t)} \}$ generated by Dykstra’s projection converges to the solution of Equation~\ref{eq:min d g} as long as the intersection set is non-empty \cite{gaffke1989}. As noted previously, $\bm{0}$ constitutes a trivial solution in our context, which means the convergence of Dykstra's projection is always guaranteed for our optimization.

\subsubsection{Gradient Vanishing}
Another challenge during training is gradient vanishing, which arises under two different conditions. First, as illustrated in Figure~\ref{fig:optimal projection}, when the optimal direction $d^*$ is on the boundary of $\mathcal{C}$, there exists at least one higher-priority policy gradient $g_i$ such that $g_i^T d^* = 0$, causing some subtasks trapped in the local optima. Second, when $\mathcal{C}$ contains only the trivial solution $\bm{0}$, $d^*$ makes no contribution to any subtask, and training stagnates.

To address gradient vanishing, we introduce \textbf{SE} as a rollout scheduler during training. Specifically, we extract a subproblem containing the top-$N$ subtasks:
\begin{equation}
    \begin{aligned}
        &\min_{d} && \| d - g_N \|^2 \\
        &\text{ s.t. } && d \in \mathcal{C}_N, \quad 1\leq N\leq M \\ 
    \end{aligned}
    \label{eq:sub min d g}
\end{equation}
where $N \sim \text{Uniform}(\{1,\cdots, M\})$ is sampled uniformly. For any $N$, if $\|d^*\| = 0$, we recursively set $N \leftarrow N-1$ and solve the higher-level subproblem, until a non-zero feasible solution is found. This approach ensures each subtask is selected and trained equally often, thereby reducing the risk of local optima caused by gradient vanishing. Besides, a major advantage over previous methods is that SE does not require accurate prior knowledge or additional hyperparameters.

Finally, the general LPPG framework for policy gradient RL algorithms is summarized in Algorithm~\ref{alg:lexicographic_policy_gradient}. Specific instantiations of LPPG-PPO and LPPG-SAC, are provided in Appendix A. 

\subsection{Theoretical Analysis}
In this subsection, we provide a theoretical analysis of the convergence properties of LPPG-RL and establish an exact lower bound for each subtask during a policy update.

\begin{theorem} \label{thm:lexico convergence}
    Consider an LMORL problem with subtask set $\mathcal{K} = \{K_1,\cdots, K_M\}$ and LPPG-RL algorithm in a general actor-critic framework. Let $\theta_t$ be the actor parameters and $\phi_t$ be the multi–head critic parameters. Assume, 
    \begin{enumerate}
    \item Two–timescale stepsizes.\;
          The actor and critic learning rates $\alpha_{\theta,t},\alpha_{\phi,t}>0$
          satisfy  
          $\sum_t\alpha_{\theta,t}=\sum_t\alpha_{\phi,t}=\infty,\;
            \sum_t(\alpha_{\theta,t}^2+\alpha_{\phi,t}^2)<\infty,\;
            \lim_{t\rightarrow\infty}\alpha_{\theta,t}/\alpha_{\phi,t}=0$
    \item $L$–smooth objectives.\;
          Each sub-objective $J_i(\theta)$ is twice differentiable w.r.t $\theta$, with Hessian $L_i$-Lipschitz continuous.
    \end{enumerate}
    Then, we have our parameters converge to a local or global lexicographic optimum $(\theta^*, \phi^*(\theta^*))$ with a stepsize $\alpha_\theta \leq \min\left\{2g_i^Td_i/(L_i\|d\|^2),i \in \{1,\cdots,M\}\right\}$
\end{theorem}

The proof of Theorem~\ref{thm:lexico convergence} is provided in Appendix E. 

\begin{theorem} \label{thm:update lower bound}
    Consider an LMORL problem, with the policy update direction $d$ by LPPG, and the stepsize $\alpha_\theta$. Then, for each policy update from $\pi_\theta$ to $\pi_{\theta'}$, we have the lower bound for the improvement of each performance $J_i$,
    \begin{equation} \notag
        J_i(\pi') - J_i(\pi) \geq \frac{\alpha_\theta \delta_i}{1-\gamma} - \frac{2\gamma C_i^{\pi',\pi}}{(1-\gamma)^2}\sqrt{\frac{\eta}{2}}
    \end{equation}
\end{theorem}

The proof of Theorem~\ref{thm:update lower bound} is provided in Appendix F. 

\begin{figure*}[!t]
    \centering
    \subfigure[Simulation trajectories]{
        \begin{minipage}{0.65\columnwidth} \label{fig:single goal traj}
            \centering
            \includegraphics[width=\linewidth]{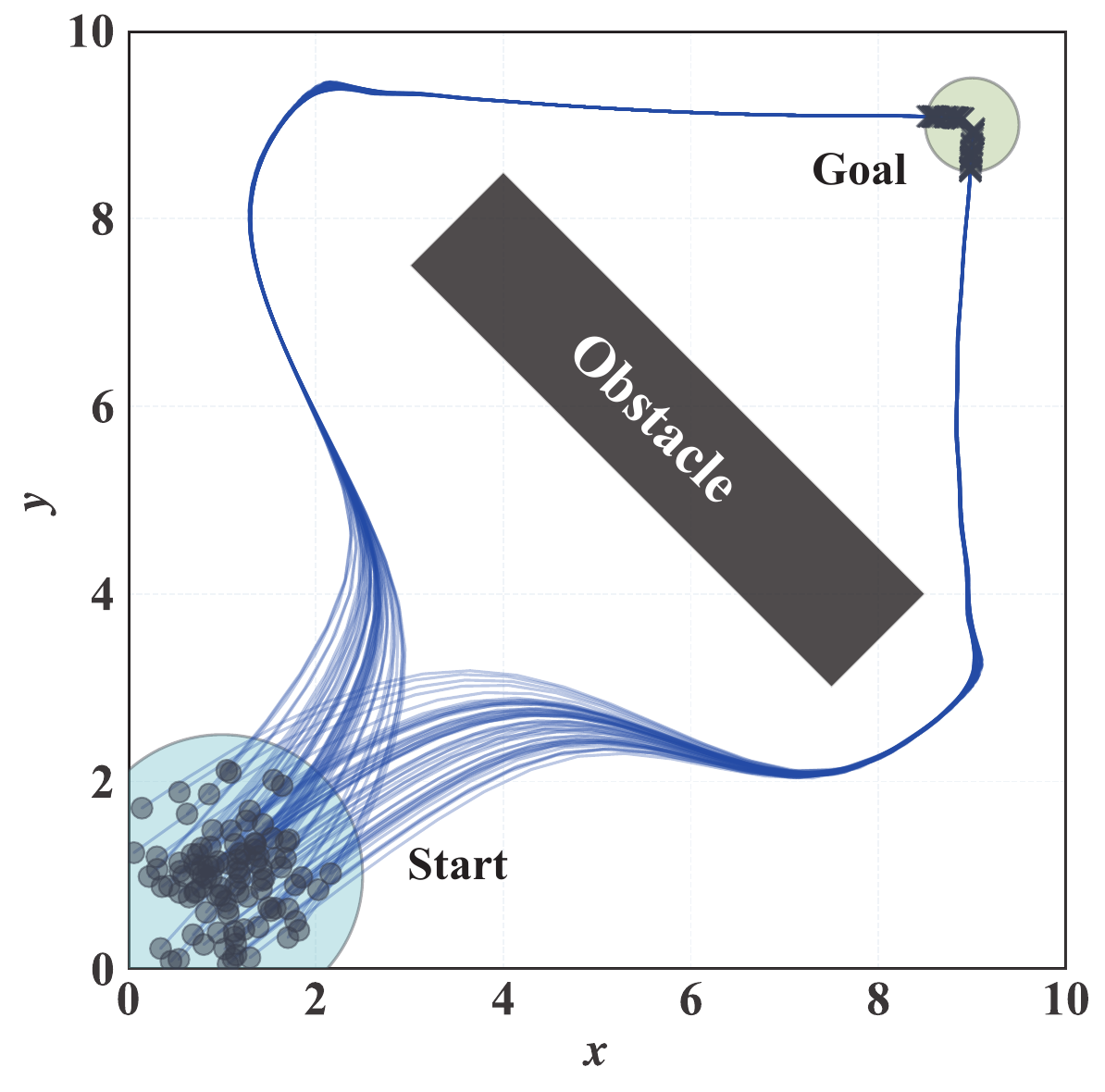}
        \end{minipage}
    }
    \subfigure[95\% Confidence corridor with noise]{
        \begin{minipage}{0.65\columnwidth} \label{fig:single goal traj 95 estimated}
            \centering
            \includegraphics[width=\linewidth]{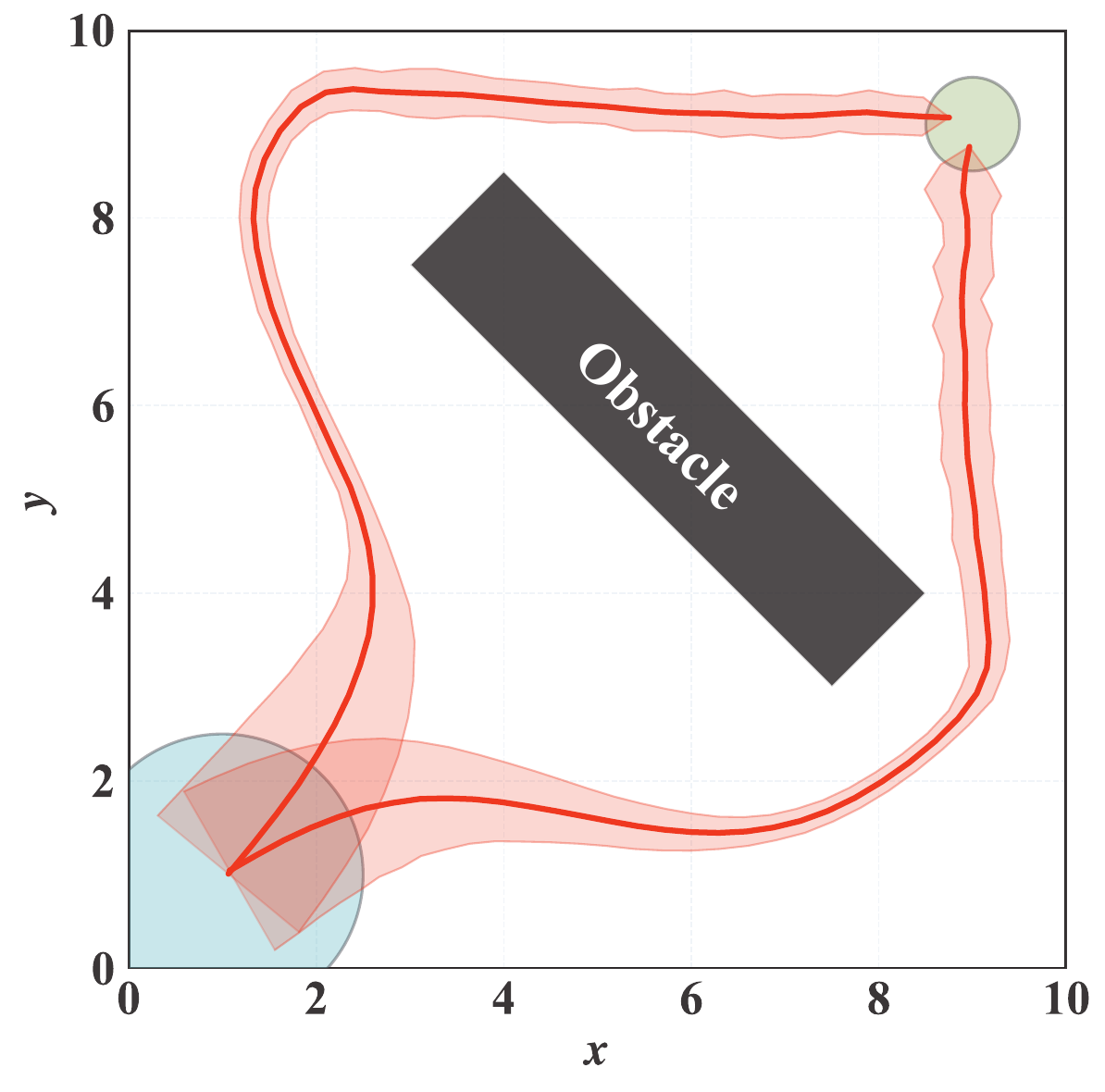}
        \end{minipage}
    }
    \subfigure[$\pi^*(a\mid s)$]{
        \begin{minipage}{0.65\columnwidth} \label{fig:single goal a distribution}
            \centering
            \includegraphics[width=\linewidth]{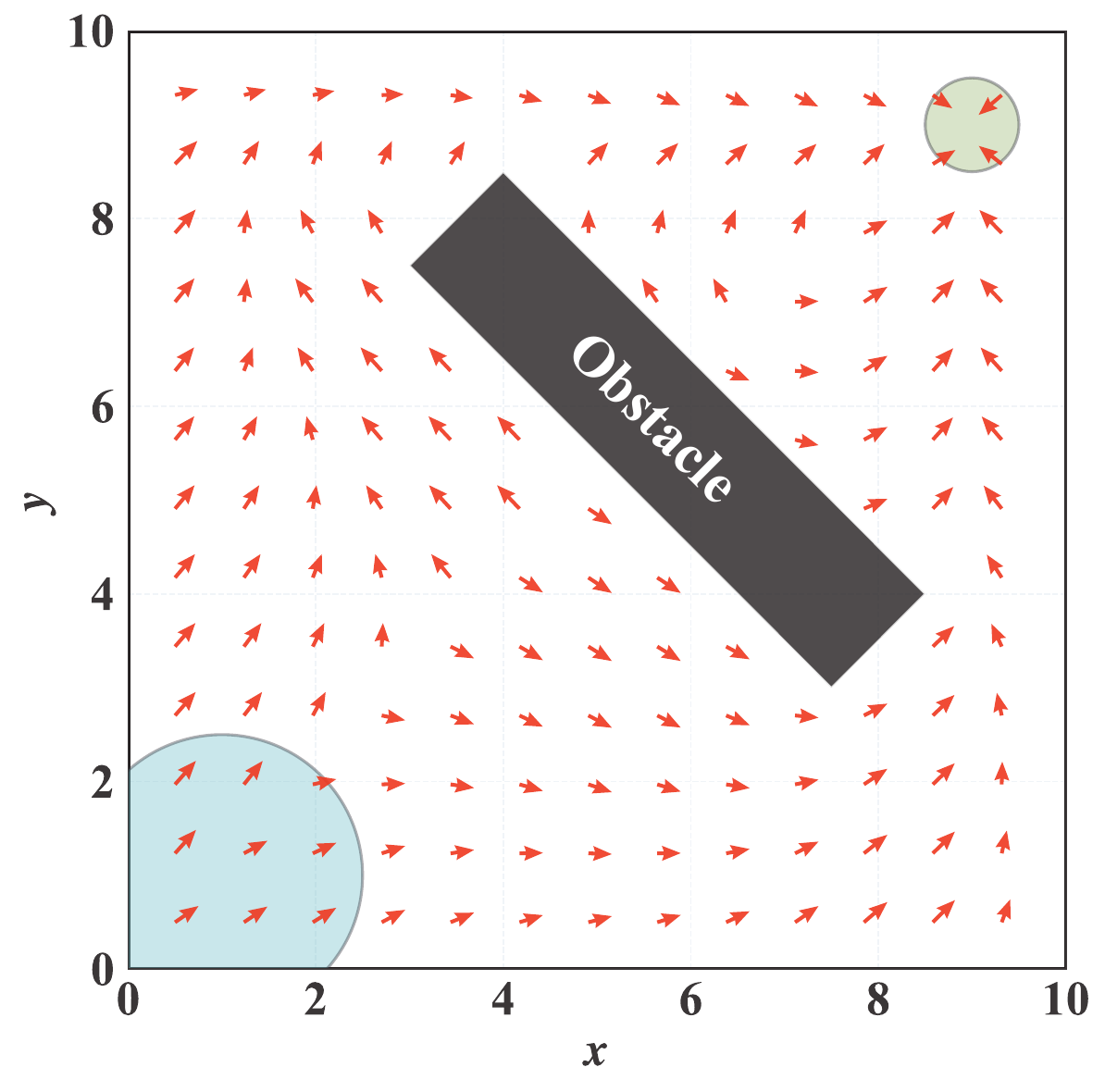}
        \end{minipage}
    }
    \caption{\textbf{1 Goal experiment} in the 2D navigation environment. \ref{fig:single goal traj} shows the map and agent trajectories under 50 Monte-Carlo simulations with different seeds. \ref{fig:single goal traj 95 estimated} displays the 95\% statistical confidence corridor of trajectories under $\mathcal{N}(0, 0.1)$ Gaussian noise applied to the state. \ref{fig:single goal a distribution} illustrates of the deterministic policy direction for different agent locations in the map, where the length of arrows corresponds to the action magnitude.}    
    \label{fig:single goal sim result}
\end{figure*}

\begin{figure*}[!t]
    \centering
    \subfigure[Green$\rightarrow$Red]{
        \begin{minipage}{0.65\columnwidth} \label{fig:double_goal}
            \centering
            \includegraphics[width=\linewidth]{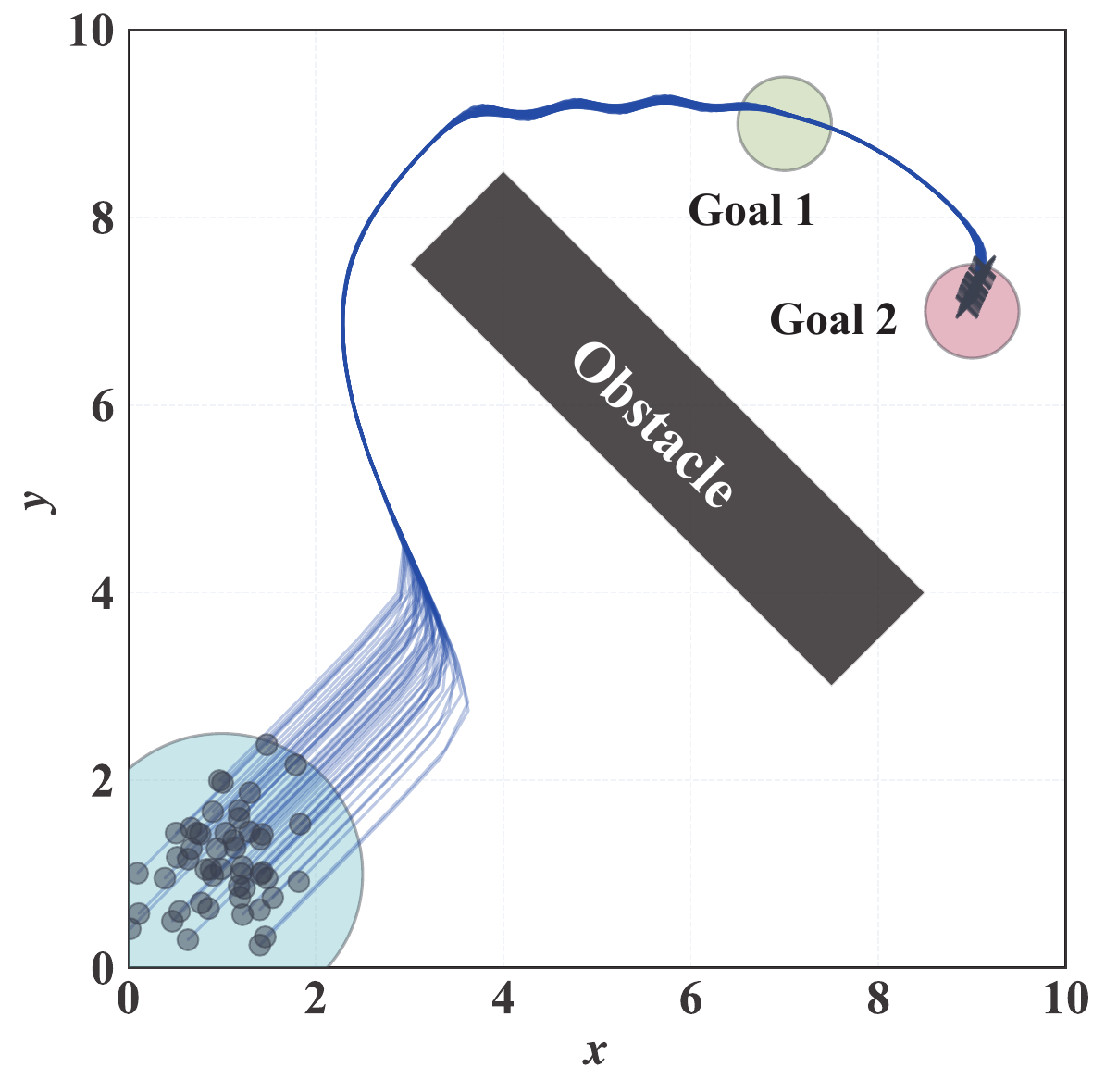}
        \end{minipage}
    }
    \subfigure[Red$\rightarrow$Green]{
        \begin{minipage}{0.65\columnwidth} \label{fig:double_goal_reversed}
            \centering
            \includegraphics[width=\linewidth]{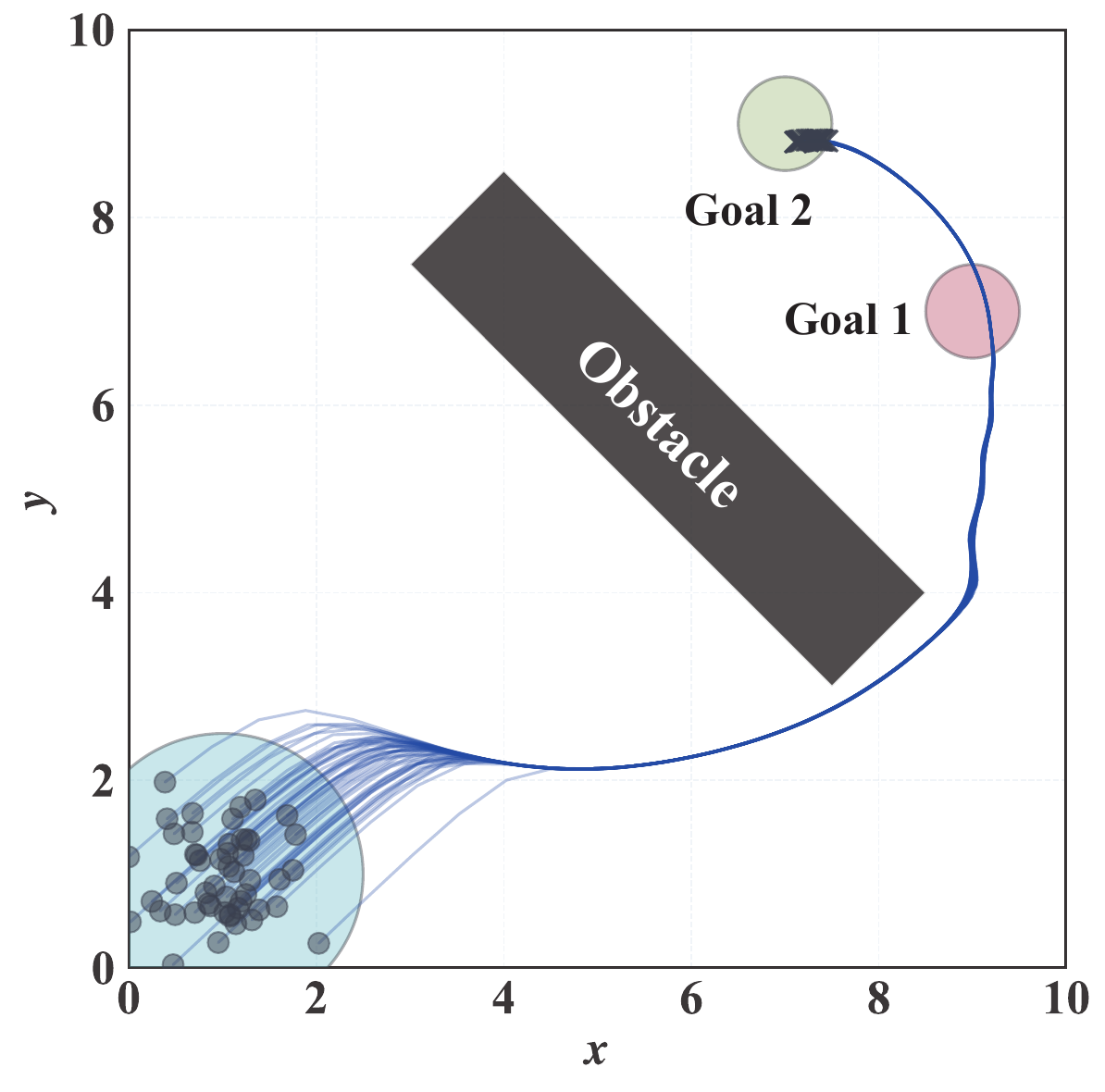}
        \end{minipage}
    }
    \subfigure[Returns comparison]{
        \begin{minipage}{0.65\columnwidth} \label{fig:double_goal_returns_compare}
            \centering
            \includegraphics[width=\linewidth]{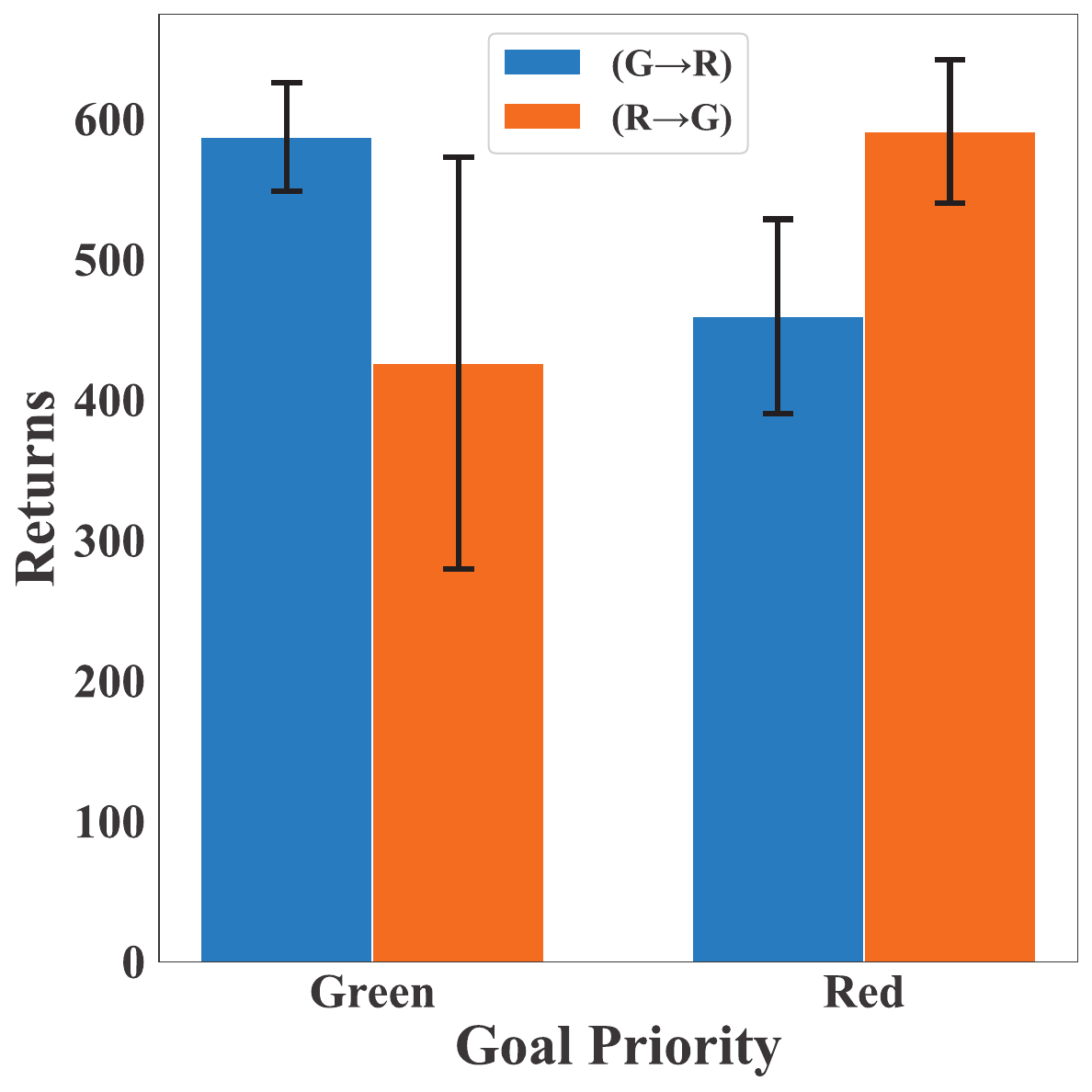}
        \end{minipage}
    }
    \caption{\textbf{2 Goal experiment} in the 2D navigation environment. The start region is initialized as before and two symmetric goal regions are designated as different priority subtasks. Figure~\ref{fig:double_goal} and Figure~\ref{fig:double_goal_reversed} show the trajectories under 50 Monte-Carlo simulations with two different subtask priority configurations. The green goal region has higher priority in Figure~\ref{fig:double_goal} while it has lower priority in Figure~\ref{fig:double_goal_reversed}. As shown in Figure~\ref{fig:double_goal_returns_compare}, a comparison of the returns for the two goal regions under two different priority configurations is presented, demonstrating that our method strictly preserve lexicographic priorities. The legend "G$\rightarrow$R" and "R$\rightarrow$G" correspond to Figure~\ref{fig:double_goal} and Figure~\ref{fig:double_goal_reversed}, respectively.}
    \label{fig:double goal sim result}
\end{figure*}

\begin{algorithm}[tb]
\caption{LPPG}
\label{alg:lexicographic_policy_gradient}
\begin{algorithmic}[1]
\REQUIRE: policy network $\pi(\cdot|s)$, prioritized subtasks set $\mathcal{K}$ \\
\STATE Compute gradient $g_i$ for each subtask
\STATE Sample subtask index $N \sim \mathrm{Uniform}(\{1,\ldots,|\mathcal{K}|\})$
\FOR{$n = N,\cdots,1$}
    \STATE $\mathcal{C}_n \leftarrow \bigcap\limits_{i=1}^{n} 
        \left\{\, d \,\middle|\, g_i^\top d \ge -\epsilon_i \,\right\}$
    \IF{$\mathcal{C}_n = \{\bm{0}\}$}
        \STATE \textbf{continue}
    \ELSE
        \STATE \textbf{break}
    \ENDIF
\ENDFOR
\STATE Solve $d \in \mathcal{C}_n$ closest to $g_n$,
\RETURN Optimal gradient $d^*$
\end{algorithmic}
\end{algorithm}

\section{Experiments}

\begin{table*}[!t]
    \centering
    \begin{tabular}{lccccc}
        \toprule
        Algorithm & \textbf{Nav2D-1G (ms)} & \textbf{Nav2D-10G (ms)} & \textbf{Nav2D-20G (ms)} & \textbf{Nav2D-50G (ms)} & \textbf{Nav2D-100G (ms)} \\
        \midrule
        \textbf{OSQP} & 18.30$\pm$1.61 & 41.46$\pm$3.56 \ \ & \ \ 87.37$\pm$14.66 & 701.61$\pm$158.83 & \textbf{2907.73$\pm$801.74} \\
        \textbf{SCS} & 21.13$\pm$2.31 & 44.47$\pm$12.04 & 113.97$\pm$25.54 & 803.13$\pm$255.26  & 3910.14$\pm$926.18 \\ 
        \textbf{CLARABEL} & 26.79$\pm$1.41 & 79.17$\pm$3.91 \ \  & 154.44$\pm$10.83 & 994.35$\pm$120.79 & 4162.98$\pm$554.81 \\
        \textbf{Dykstra(Ours)} & \textbf{\ \ 0.93$\pm$0.34} & \textbf{5.33$\pm$2.38} & \textbf{\ \ 22.15$\pm$10.46} & \textbf{685.61$\pm$146.68} & 2983.20$\pm$370.53 \\
        \bottomrule
    \end{tabular}
    \caption{Ablation study of Dykstra's projection method with varying numbers of goal regions in the 2D navigation environment. All results are reported in milliseconds (ms). For \textbf{Nav2D-nG} environment, we generate $\textbf{n}$ different goal regions as lexicographic subtasks. Each setting is trained with 100k steps, 5 seeds, starting with the lowest priority subtask $K_M$. We compare the average runtime of policy-gradient search between Dykstra and several commonly used CVXPY solvers. All solvers yields the same projected gradient within a $10^{-6}$ relative error. See Appendix C for additional details.}
    \label{tab:compare sampling}
\end{table*}

\begin{table*}[!t]
\centering
\setlength{\tabcolsep}{1mm}  
\begin{tabular}{
  l
  *{3}{r c}   
  *{4}{r c}   
  *{4}{r c}   
}
\toprule
\multirow{2}{*}{Algorithm} &
\multicolumn{6}{c}{\textbf{Nav2D-1G}} &
\multicolumn{8}{c}{\textbf{Nav2D-2G}} &
\multicolumn{8}{c}{\textbf{Nav2D-2G-rev}} \\
\cmidrule(lr){2-7}\cmidrule(lr){8-15}\cmidrule(lr){16-23}
& \multicolumn{2}{c}{$K_1$} & \multicolumn{2}{c}{$K_2$} & \multicolumn{2}{c}{$K_3$}
& \multicolumn{2}{c}{$K_1$} & \multicolumn{2}{c}{$K_2$} & \multicolumn{2}{c}{$K_3$} & \multicolumn{2}{c}{$K_4$}
& \multicolumn{2}{c}{$K_1$} & \multicolumn{2}{c}{$K_2$} & \multicolumn{2}{c}{$K_3$} & \multicolumn{2}{c}{$K_4$} \\
\midrule
\textbf{LPA} &
100 & \yes & -19 & \no & 714 & \yes &
98 & \yes &   0 & \yes & 631 & \yes & 246 & \no &
97 & \yes &   0 & \yes & 603 & \yes & -39 & \no \\

\textbf{LPPO} &
100 & \yes & -20 & \no & 740 & \yes &
100 & \yes & -10 & \no & 753 & \yes & 382 & \no &
100 & \yes & -10 & \no & 785 & \yes & 197 & \no \\

\textbf{LPPG-PPO w/o SE} &
100 & \yes & -12 & \no & 704 & \yes &
98 & \yes &   0 & \yes & 199 & \no & 604 & \yes &
98 & \yes &   0 & \yes & 258 & \no & 612 & \yes \\

\textbf{LPPG-PPO (Ours)} &
\B{98}  & \B{\yes} & \B{0} & \B{\yes} & \B{427} & \B{\yes} &
\B{96} & \B{\yes} & \B{0} & \B{\yes} & \B{587} & \B{\yes} & \B{459} & \B{\yes} &
\B{97} & \B{\yes} & \B{0} & \B{\yes} & \B{591} & \B{\yes} & \B{426} & \B{\yes} \\
\bottomrule
\end{tabular}
\caption{Comparison of baseline algorithm LPA \cite{tercan2024} and LPPO \cite{skalse2022}, as well as an ablation study of SE (LPPG-PPO w/o SE), on performances in three environments under 10 different seeds. The metric is the average returns after 1M steps, where higher values are better. \yes and \no, indicate whether the subtask is completed. For each environment, the lexicographic order is verified first. The light blue block marks the optimal policy that satisfies the priority order and achieves the highest returns. The completed result table with standard deviations is provided in Appendix C.}
\label{tab:compare rewards}
\end{table*}
In this section, we present our experimental setup and results. We focus on two main aspects: (1) demonstrating that LPPG-RL achieves zero constraint violation with high performance in safety-critical tasks. (2) illustrating explicitly how LPPG-RL distinguishes and handles different priorities when permuting lexicographic orders of subtasks. 

All experiments are conducted using LPPG-PPO in a 2D navigation environment with varying numbers of prioritized subtasks. We consider three scenarios: \textbf{(1) Nav2D-1G}: The agent is required to navigate to a goal located behind a rectangular obstacle. \textbf{(2) Nav2D-2G}: Two goals are placed symmetrically behind the obstacle. The agent must reach the green goal first, followed by the red goal. \textbf{(3) Nav2D-2G-rev}: The priorities of the two goals is reversed. In all scenarios, two common high-priority constraints are imposed: the agent must stay within the map boundary, and must not enter the obstacle area.

\subsection{Nav2D with 1 Goal} \label{sec:nav2d g1}
In the following experiment, a navigation subtask with two strict constraints are given. As shown in Figure~\ref{fig:single goal sim result}, a map is setup, with a start region and a goal region locating on opposite sides of an obstacle. The set $\mathcal{K}$ contains three elements for a time-limited episode. $K_1$: "In Boundary", $K_2$: "Avoid Collision", $K_3$: "Reach Green Goal". The initial location is sampled from a Gaussian distribution, $x \sim \mathcal{N}(1,0.5), y \sim \mathcal{N}(1, 0.5)$. Figure~\ref{fig:single goal traj} displays 50 Monte-Carlo simulation trajectories with two different seeds. Due to the symmetry, the agent explores both sides of the obstacle with equal probability. In Figure~\ref{fig:single goal traj 95 estimated}, we introduce $\mathcal{N}(0, 0.1)$ Gaussian noise to the state, and plot the 95\% confidence corridor, indicating consistent safety performance under perturbations. Figure~\ref{fig:single goal a distribution} illustrates the deterministic policy in the map, where the arrow length denotes action magnitude. The results demonstrate that our policy reliably navigates around the obstacle and towards the goal region with high speed. 

\subsection{Nav2D with 2 Goals} \label{sec:nav2d g2}
In the following experiment, two different goal regions are set symmetrically behind the obstacle, and two comparative conditions are designed to evaluate the ability to distinguish priority orderings. The subtask set $\mathcal{K}$ contains four elements, with $K_1$ and $K_2$ same as before. In the condition 1, the remaining priorities are, $K_3$: "Reach Green Goal", $K_4$: "Reach Red Goal". In condition 2, $K_3$ and $K_4$ are reversed. Figure~\ref{fig:double_goal} shows the simulation trajectories of condition 1 (\textbf{Nav2D-2G}), where the agent tends to navigate around the obstacle from the left side to reach the green goal region more efficiently, as it is relatively closer. In contrast, Figure~\ref{fig:double_goal_reversed} presents the trajectories of condition 2 (\textbf{Nav2D-2G-rev}), where switching the priority causes the agent to prefer the path around the obstacle from the right side. Figure~\ref{fig:double_goal_returns_compare} compares the returns from both goal regions under two conditions. In condition 1 ("G$\rightarrow$R"), the agent achieves higher returns from the green, whereas in condition 2 ("R$\rightarrow$G"), the agent obtains higher returns from the red. These results demonstrate that our trained agents are able to solve the LMORL problem with priorities strictly satisfied.

It is important to note that, to ensure the comparative experiments are not influenced by environmental factors (such as state or reward function design), all things are kept identical for the green and red goal regions.

\subsection{Comparison and Ablation Study}
Experiments above demonstrate the effectiveness of LPPG-RL, which employs Dykstra's projection and SE scheduler. This motivates three key questions: (1) How does Dykstra improve the efficiency of policy-gradient search compared to other solvers? (2) How does the SE scheduler help avoid local optimal? (3) How does LPPG-RL perform relative to baseline methods in LMORL?

We conduct two comparative studies to answer these questions. (1) we generate \textbf{n} different goal regions as lexicographic subtasks in \textbf{Nav2D-nG} environment, and compare the average runtime of policy-gradient search between Dykstra and other solvers. (2) We adapt the same experimental settings as before and compare the final performance of four methods: LPA \cite{tercan2024}, LPPO \cite{skalse2022}, LPPG-PPO w/o SE, and full LPPG-PPO. 

The first set of results is presented in Table~\ref{tab:compare sampling}. where we compare Dykstra with several common solvers in CVXPY, including OSQP, SCS, and CLARABEL. The results show that, for problems of practical size ($\textbf{n} \leq 50$), Dykstra achieves up to $20\times$ faster than generic solvers. The gap gradually diminishes with \textbf{n} increasing. For larger problems, CVXPY solvers begin to match Dykstra's performance when $\textbf{n} \approx 100$. Importantly, all solvers converge to the same solution within a relative error of $10^{-6}$, resulting in indistinguishable final policy performances.

The second set of results, presented in Table~\ref{tab:compare rewards}, includes a comparison with the baseline algorithms LPA and LPPO, as well as an ablation study on SE (LPPG-PPO w/o SE). For each environment, we report the returns and subtask completion status after 1M training steps. 

For Nav2D-1G, we have $K_1-K_3$. All other methods achieves higher returns for $K_3$, but fail to avoid the obstacle, as they primarily focus on $K_3$ during early training and get trapped in suboptimal policies. In contrast, our method ensures uniform sampling, and successfully completes all subtasks. Note that, our method achieves an average score of 98 on $K_1$, slightly below perfect due to the residual exploration noise in limited training steps. This noise introduces an irreducible variance floor without affecting the optimality of the policy. This observation also applies to subsequent results. 

For Nav2D-2G and Nav2D-2G-rev, we have $K_1-K_4$. LPA get stuck in $K_3$, due to the difficulty of setting appropriate thresholds without prior knowledge. LPPO still fails to complete $K_2$, because the performance is sensitive to Lagrange multiplier hyperparameters. LPPG-PPO w/o SE terminates with $K_4$ higher than $K_3$. This is because, in the early stage, rollouts cover only limited states, resulting in unreliable policy gradients due to biased advantage estimations, and rendering practically infeasible gradients feasible. Therefore, the agent collects rollouts primarily near the low-priority region, further reinforcing biased estimations and leaving high-priority areas unexplored. SE addresses this issue by forcing the agent to visit high-priority states, thereby reducing bias and restoring the intended lexicographic order.

\section{Conclusion}
In this paper, we introduce \textbf{LPPG-RL}, which reformulates LMORL problem as a convex optimization of gradients. By finding the intersection of higher-priority half-spaces, LPPG-RL obtains a lexicographically feasible policy gradient direction. We also propose using Dykstra’s projection to accelerate gradient search for small- to medium-scale problems, and employ SE scheduler to avoid local optima and facilitate convergence. Experimental results demonstrate that LPPG-RL outperforms previous methods. Notably, our approach does not require additional hyperparameters or prior knowledge, making it more adaptive and practical. Furthermore, we provide theoretical guarantees on convergence and a lower bound for policy updates. A limitation of our work is that we only consider a fixed priority order among subtasks. In future work, it would be interesting to explore the use of transfer learning or experience replay buffers to handle dynamic or changing priorities.

\bigskip

\section*{Acknowledgements}
This work was supported by the National Natural Science Foundation of China (62120106003, 62173301). We thank the anonymous reviewers for their constructive suggestions that improved this paper. 

\bibliography{aaai2026}

\newpage

\appendix

\twocolumn[
\begin{center}
\textbf{\Large Appendix}
\end{center}
\hfill \break
\hfill \break
]

\section{A. Relevant Algorithms}
In this section, we first provide the pseudo codes of some relevant algorithms in our paper, including Dykstra's projection algorithm and the instantiation of SAC version of our method, LPPG-SAC. 

Dykstra's projection is shown as Algorithm~\ref{app:alg:dykstra}. We start from a initial point $x^{(0)}$, and allocate a residual variable $r_i$ for each set $\mathcal{C}_i$. For each iteration, we project $x+r_i$ to the next convex set, and update the residual part. When the norm of $x$ between iterations is small than the tolerance, Dykstra's projection gets converged, and return the projection of $x^{(0)}$ onto the intersection of closed convex sets. 

LPPG-PPO is presented as Algorithm~\ref{app:algo:LPPG-PPO}. PPO is an on-policy algorithm using actor-critic architecture. First, rollouts are collected, and advantage functions are estimated (using GAE here). Second, gradient of each subtask is calculated, and LPPG is applied to help find the optimal policy gradient $d^*$ satisfying all priorities. Then, we update policy and critic network step by step. 

LPPG-SAC is presented as Algorithm~\ref{app:alg:lppg-sac}. SAC is an off-policy algorithm using twin Q networks with targets. First, LPPG-SAC collects rollouts and stores them into replay buffer $\mathcal{D}$. When training starts, a batch is sampled from the buffer. We compute the gradient for each subtask, using our LPPG algorithm to find the optimal policy gradient $d^*$ satisfying all priorities, and update the policy network. Then, we update Q networks for minimizing the soft Bellman residual, adjust the temperature $\alpha$, and update target networks. 

There can be some extra tricks to accelerate LPPG further. First, when the sample the subtask index $N$, and obtain $\|d^*\|=0$, we can do $N\sim \operatorname{Uniform}(1,\cdots,N-1)$ instead of directly set $N \leftarrow N-1$. This can reduce the extra iterative time of trial and error. Second, when iterative convergence becomes more difficult with the number of subtasks growing, we can relax the threshold of convergence in Dykstra's projection a bit (e.g. $1\times10^{-6} \rightarrow 1\times10^{-4}$). This helps converge faster may at a cost of the final performance. Please note that, the threshold cannot be adjusted too large. Otherwise, it can lead to biased rollouts distribution as shown in our experiments. 

\begin{algorithm}[tb]
\caption{Dykstra's Projection Algorithm}
\label{app:alg:dykstra}
\begin{algorithmic}[1]
    \REQUIRE initial point $x^{(0)}\in \mathbb{R}^d$, closed convex sets $\{\mathcal{C}_i\}_{i=1}^M$, projection method $\operatorname{Proj}()$, tolerance $\epsilon > 0$, maximum iterations $T_{\max}$
    \STATE $x \gets x^{(0)}$ 
    \STATE $r_i \gets 0$ for $i = \{1, \cdots, m \}$
    \FOR{$t = 0, 1, \cdots, T_{\max}-1$}
        \FOR{$i = 1, \cdots, M$}
            \STATE $y \gets x + r_i$
            \STATE $p \gets \operatorname{Proj}_{\mathcal{C}_i}(y)$ 
            \STATE $r_i \gets y - p$
            \STATE $x \gets p$
        \ENDFOR
        \IF{$\|x - x^{(t)}\|_2 \leq \epsilon $}
            \STATE \textbf{break}
        \ELSE
            \STATE $x^{(t+1)} \gets x$
        \ENDIF
    \ENDFOR
    \RETURN $x \approx \operatorname{Proj}_{\cap_i\mathcal{C}_i}\bigl(x^{(0)}\bigr)$
\end{algorithmic}
\end{algorithm}

\begin{algorithm}[tb]
    \caption{LPPG-PPO}
    \label{app:algo:LPPG-PPO}
    \begin{algorithmic}[1]
    \REQUIRE policy network $\pi_{\theta}(\cdot|s)$, critic network $V_{\phi}(s)$, prioritized subtasks set $\mathcal{K}$ 
        \FOR{iteration $k = 1, 2, \cdots, $}
            \STATE Collect rollouts $\{s_t, a_t, \bm{r}_t, \log \pi_{\theta}(a_t|s_t)\}$
            \STATE Compute rewards-to-go $\hat{R}_t$ for $K \in \mathcal{K}$
            \STATE Compute advantage estimates (e.g. using GAE) based on current value function $V_{\phi}$ for $K \in \mathcal{K}$
            \FOR{each update epoch}
                \STATE Find optimal policy gradient $d^*$ with Algorithm~\ref{alg:lexicographic_policy_gradient}
                \STATE Compute critic loss $L_{\phi}$
                \STATE Update policy network
                \STATE Update critic network 
            \ENDFOR
        \ENDFOR
        \RETURN Optimal policy network $\pi_\theta^*$
    \end{algorithmic}
\end{algorithm}

\begin{algorithm}[tb]
\caption{LPPG-SAC}
\label{app:alg:lppg-sac}
\begin{algorithmic}[1]
\REQUIRE policy network $\pi_\theta(\cdot|s)$, twin Q networks $\{Q_{\phi_1},Q_{\phi_2}\}$ with targets $\{Q_{\phi_{\text{targ},1}},Q_{\phi_{\text{targ},2}}\}$, learning start step $t_s$, target entropy $\mathcal{H}$, temperature $\alpha$, target smoothing coefficient $\rho$, replay buffer $\mathcal{D}$, prioritized subtasks set $\mathcal{K}$
\STATE $\phi_{\text{targ},1} \gets \phi_1$, $\phi_{\text{targ},2} \gets \phi_2$
\FOR{environment step $t=1,2,\dots$}
    \STATE Collect rollouts $\{s, a, \bm{r}, s', d\}$
    \STATE Add rollouts in replay buffer $\mathcal{D}$
    \IF{$t \geq t_s$}
        \STATE Sample batch $B = \{(s,a,\bm{r},s', d)\} \sim \mathcal{D}$
        \STATE Compute targets for Q functions
        \STATE Compute policy gradient $g_{\theta,K}$ for $K \in \mathcal{K}$ 
        \STATE Find optimal policy gradient $d^*$ with Algorithm~\ref{alg:lexicographic_policy_gradient}
        \STATE Update Q networks with gradients $g_{\phi_i}$
        \STATE Update temperature with gradient $g_\alpha$
        \STATE Update target networks with 
            \begin{equation} \notag
                \phi_{\text{targ},i} \gets \rho\phi_{\text{targ},i} + (1-\rho)\phi_i, \quad i=\{1,2\}
            \end{equation}
    \ENDIF

\ENDFOR
\RETURN Optimal policy network $\pi^*_\theta$
\end{algorithmic}
\end{algorithm}

\section{B. Experiment Details}
In this section, we will present more details about our experiment and present additional results.

\subsection{B.1 Map Details}
Here, we first provide some specific details about how to establish our map as the environment in experiments. 

In all experiments, The map is constructed with a 10$\times$10 rectangle. An rectangle obstacle is put in the middle of the start region and the goal region. The locations of the obstacle consists of four points, $[(3, 7.5), (4, 8.5), (8.5, 4), (7.5, 3)]$. The Start region is also generated with a Gaussian distribution as the following, 
\begin{equation} \notag \label{app:eq:start region dist}
    \begin{aligned}
        x \sim \mathcal{N}(1, 0.5) \\ 
        y \sim \mathcal{N}(1, 0.5) \\ 
    \end{aligned}
\end{equation}
In \textbf{Nav2D-1G} environment, the green goal region is set as a circle, with center $(9,9)$ and radius 0.5. In \textbf{Nav2D-2G} environments, the centers of the green and red goal region are $(7,9)$ and $(9,7)$, respectively. 

\subsection{B.2 Environment Details}
Here, we present the design of our environment, including the state space, action space, and specific reward design for each subtask. 
\begin{itemize}
    \item \textbf{State} The state is constructed with the agent current location $(x_t, y_t)$ and all goal center location $(x^G_1, y^G_1, \cdots, x^G_n, y^G_n)$. In some more complex environment, frame stack technique can be applied to stack past states together, to make the training more stable. 
        \begin{equation} \notag
            s_t = [x_t, y_t, x^G_1, y^G_1, \cdots, x^G_n, y^G_n]
        \end{equation}

    \item \textbf{Action} The action is constructed with two independent components along the $x$ and $y$ direction. The policy action $a_t$ is $(-1, 1)$ to ensure the training stability, and the actual action to the environment $\tilde{a}_t$ will be rescaled to the max action bound $v_{max}$. In our experiments, $v_{max} = 0.5$.
        \begin{equation} \notag
            \begin{aligned}
                & && a_t = (-1, 1) \\ 
                \Rightarrow & && \tilde{a}_t = (-v_{max}, v_{max})
            \end{aligned}
        \end{equation}

    \item \textbf{Termination Condition} Here, we show the termination conditions of our environment. There are two termination conditions. (1) The agent gets out of the boundary of the map. (2) The agent step exceeds the limit of the episode. The limit episode length is set to be 100 in our experiments.  

    \item \textbf{Reward: In Boundary} This reward is designed to keep the agent not going out of the map boundary. Otherwise, episode is terminated. The reward is designed to be an binary reward with values $\{0, 1\}$ as the following. 
        \begin{equation} \notag
            r^{boundary}_t = 
                \begin{cases}
                    1 \quad & \text{if} \quad  \{x_t \in [0, 10], y \in [0, 10]\} \\ 
                    0 & \text{else}
                \end{cases}
        \end{equation}

    \item \textbf{Reward: Avoid Obstacle} This reward is designed to help agent avoid obstacle areas. Instead of finding a barrier function as the reward to cover the whole map, we only give penalty when the agent is detected inside the obstacle area. Generally, for any polygonal obstacle area, we get all vertices, and calculated the minimal distance from these vertices to the agent current location. Then, a quadratic penalty is given. 
        \begin{equation} \notag
            r^{obstacle}_t = 
            \begin{cases}
                    \min(d_1,\cdots, d_i)^2 \quad & \text{if} \quad  \text{inside} \\ 
                    0 & \text{else}
                \end{cases}
        \end{equation}
    where $\min(d_1, \cdots, d_i)$ means the minimal Euclidean distance to all vertices, and $(x_{obs}, y_{obs})$ refers to the general vertices of the obstacle. 

    \item \textbf{Reward: Goal Tracking} This reward is designed for reaching the goal region. To be more specific, the reward is no relevant with other subtasks such as boundary or collision, and all rewards are post-processed with LPPG to deal with the priority.
        \begin{equation} \notag
            r^{goal}_t = 
            \begin{cases}
                c^{goal} \quad & \text{if} \quad d \leq 0.5 \\
                -\lambda d^2 \quad & \text{else} 
            \end{cases}
        \end{equation}
    where $d$ is the Euclidean distance to the goal center, $\lambda$ is a constant coefficient to control the reward scale, and $c^{goal}$ is a constant reward for reaching the goal. In our experiment, $\lambda = 100$ and $c^{goal} = 10$.

    Specially, when multiple goal regions are involved, we set a reaching flag $\bm{0}_{G}$ for each goal. For each goal region, the flag is set to be $\bm{1}_{G}$ after the agent reach the goal for the first time, to make the agent get continuous reward from the goal. In this way, the agent go for the next goal region rather than stay still. 

    \begin{equation} \notag
            r^{goal}_t = 
            \begin{cases}
                c^{goal} \quad & \text{if} \quad d \leq 0.5 \ \text{or} \ \bm{1}_G \\
                -\lambda d^2 \quad & \text{else} 
            \end{cases}
        \end{equation}
\end{itemize}

\subsection{B.3 Training Details}
The hyperparameter settings of our experiments are presented with Table~\ref{app:tab:hyperparameters of lppg-ppo}. The specific structure of the actor and critic varies between different environment due to the dimension of states and number of subtasks. The specific information of environments, including the state dimension, action dimension and priority settings are presented.

\begin{remark}
    Please note that, for hyperparameter \textbf{Lexicographic relaxation value} \bm{$\epsilon_i$} for each subtask, we set all as 0. This is because we use the SE technique as a rollout scheduler to control a uniform rollout distribution for each subtask. Therefore, compared with past methods like LPA, our proposed LPPG-RL does not require any specific prior environment knowledge or tunable hyperparameters to help regulate the training.  
    
    For LPPG-RL, these hyperparameters are only required to be tuned when subtasks needs to be changed for specific environmental reasons (e.g. some constraints need to be relaxed with a threshold $\epsilon$) rather than computational reasons.
\end{remark}

\begin{table*}[t]
    \centering
    \begin{tabular}{lccc}
        \toprule
         & \textbf{Nav2D-1G} & \textbf{Nav2D-2G} & \textbf{Nav2D-2G-rev} \\
        \midrule
        \textbf{Total number of steps} & $1\times10^6$ & $1\times10^6$ & $1\times10^6$ \\ 
        \textbf{Learning rate for actor} & $5\times10^{-5}$ & $5\times10^{-5}$ & $5\times10^{-5}$ \\
        \textbf{Learning rate for critic} & $1\times10^{-4}$ & $1\times10^{-4}$ & $1\times10^{-4}$ \\
        \textbf{Discount factor} & $0.99$ & $0.99$ & $0.99$ \\ 
        \textbf{GAE discount factor} & $0.95$ & $0.95$ & $0.95$ \\ 
        \textbf{Batch size} & $2048$ & $2048$ & $2048$ \\ 
        \textbf{Mini batch size} & $64$ & $64$ & $64$ \\ 
        \textbf{Update epoch} & $10$ & $10$ & $10$ \\
        \textbf{Actor hidden layer numbers} & $3$ & $3$ & $3$ \\
        \textbf{Actor hidden neuron numbers} & $64$ & $64$ & $64$ \\
        \textbf{Critic hidden layer numbers} & $3$ & $3$ & $3$ \\
        \textbf{Critic hidden neuron numbers} & $64$ & $64$ & $64$ \\
        \textbf{Lexicographic relaxation value }  \bm{$\epsilon_i$} & $[0,0,0]$ & $[0,0,0,0]$ & $[0,0,0,0]$ \\
        \textbf{Dykstra's convergence tolerance} & $1\times10^{-6}$ & $1\times10^{-6}$ & $1\times10^{-6}$ \\ 
        \textbf{Dykstra's maximum iteration} & 500 & 500 & 500 \\
        \bottomrule
    \end{tabular}
    \caption{Hyperparameter settings of LPPG-PPO in different 2D Navigation environments.}
    \label{app:tab:hyperparameters of lppg-ppo}
\end{table*}

\begin{table*}[t]
    \centering
    \begin{tabular}{cccc}
        \toprule
        \textbf{Environment} & \textbf{Dim of $\mathcal{S}$} & \textbf{Dim of $\mathcal{A}$} & \textbf{Objective Priority (High to Low)} \\ 
        \midrule
        Nav2D-1G    &  4                   & 2                    & [In boundary, Avoid collision, Reach green goal] \\ 
        Nav2D-2G    &  6                   & 2                    & [In boundary, Avoid collision, Reach green goal, Reach red goal] \\ 
        Nav2D-2G-rev&  6                   & 2                    & [In boundary, Avoid collision, Reach red goal, Reach green goal] \\ 
        \bottomrule
    \end{tabular}
    \caption{Details of different 2D Navigation environments, including dimensions of state space and action space, and defined objective priorities.}
    \label{app:tab:environment details}
\end{table*}

\section{C. Additional Results}
In this section, we will present detailed additional results for our experiments. 

\subsection{C.1 Training Results}
To illustrate the policy evolution, Figure~\ref{app:fig:training snapshots of nav2d-1g} plots eight training snapshots with the rollouts from 10k to 1M steps in \textbf{Nav2D-1G}. At the initial 10k steps, the agent learns to go towards to goal but across the obstacle and outside the boundary. From 30k-80k steps, the agent starts to find obstacle-free but suboptimal detour. Then, from 150k to 1M steps, the agent learns to leverage the narrow passage around the obstacle and further optimize the trajectory efficiency. Similarly, Figure~\ref{app:fig:training snapshots of nav2d-2g} and Figure~\ref{app:fig:training snapshots of nav2d-2g-rev} plots the training evolution snapshots for \textbf{Nav2D-2G} and \textbf{Nav2D-2G-rev} environments, respectively. At the first 10k steps, the agent also gets out of the boundary. At 30k steps, the agent learns to get close to the first goal from one side. Then, from 80k to 500k steps, the agent successfully reach the first go and get closer and closer to the second goal. Finally, from 750k to 1M steps, the agent learns to reach two goals with priority with a high efficiency trajectory. 

We also present the reward training curves of each subtask of different algorithms in different environments, including two baseline algorithms: LPA \cite{tercan2024} and LPPO \cite{skalse2022}, and another ablation experiment algorithm: LPPO-PPG w/o SE. Figure~\ref{app:fig:reward nav2d-1g} shows the comparison in Nav2D-1G environment. We observe that, all 4 algorithms meets subtask 1 (In boundary). For LPA and LPPO, $r_1$ (Avoid Collision) reaches the threshold $0$ at an early stage, so $r_2$ becomes the main optimized target, which makes $r_1$ get stuck and the agent cannot find the narrow passage around the obstacle. For LPPG-PPO w/o SE, due to the lack of subproblem exploration, the agent also stay in a local optimal, and has a small chance to avoid collision. In comparison, LPPG-PPO balances the training frequency of different subtasks well, allowing the agent to jump out of the local optimum in a timely manner. Figure~\ref{app:fig:reward nav2d-2g} and Figure~\ref{app:fig:reward nav2d-2g-rev} present the training curves in Nav2D-2G and Nav2D-2G-rev environments, respectively. Due to the difficulty of finding a appropriate threshold for each subtask, in most cases, LPA algorithm usually stop optimizing for lower priority subtasks. Similarly, LPPO always stop in a local optimum because too much hyperparameters are given, which sometimes requires grid search to find the best group. For LPPG-PPO w/o SE, it has the best optimization for the lowest priority. However, in our map location, the rollouts distribution is severely biased due to premature and excessive exploration for lower priority subtask. Therefore, there is little chance for high priority subtask to fit the correct output distribution. Besides, because of the policy gradient constraints we have for lexicographic orders, the training will be obviously unstable and the variance is larger for some higher priority subtasks. 

\begin{figure*}[!t]
    \centering
    \subfigure[10000]{
        \begin{minipage}{0.48\columnwidth}
            \centering
            \includegraphics[width=\linewidth]{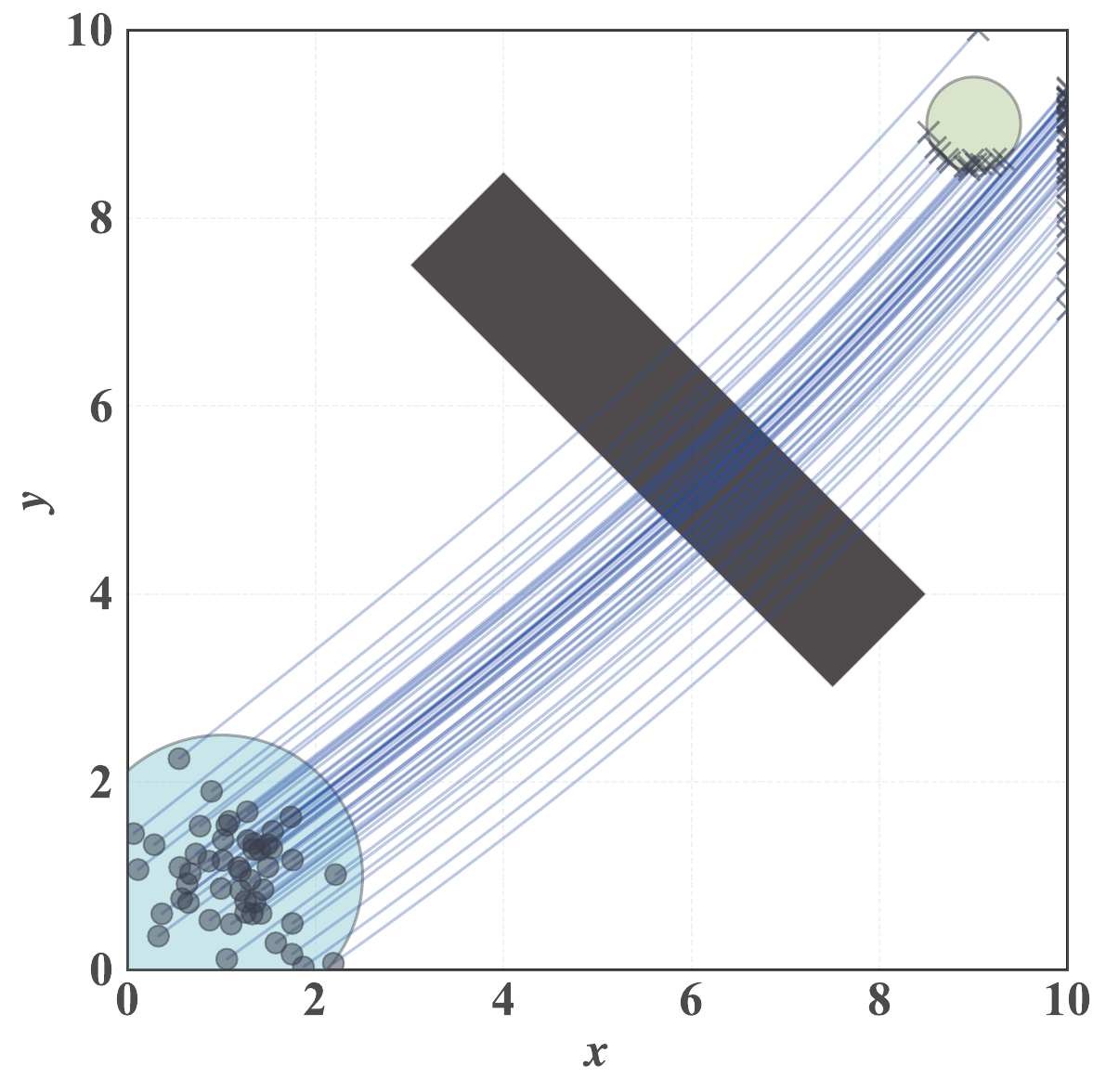}
        \end{minipage}
    }
    \subfigure[30000]{
        \begin{minipage}{0.48\columnwidth}
            \centering
            \includegraphics[width=\linewidth]{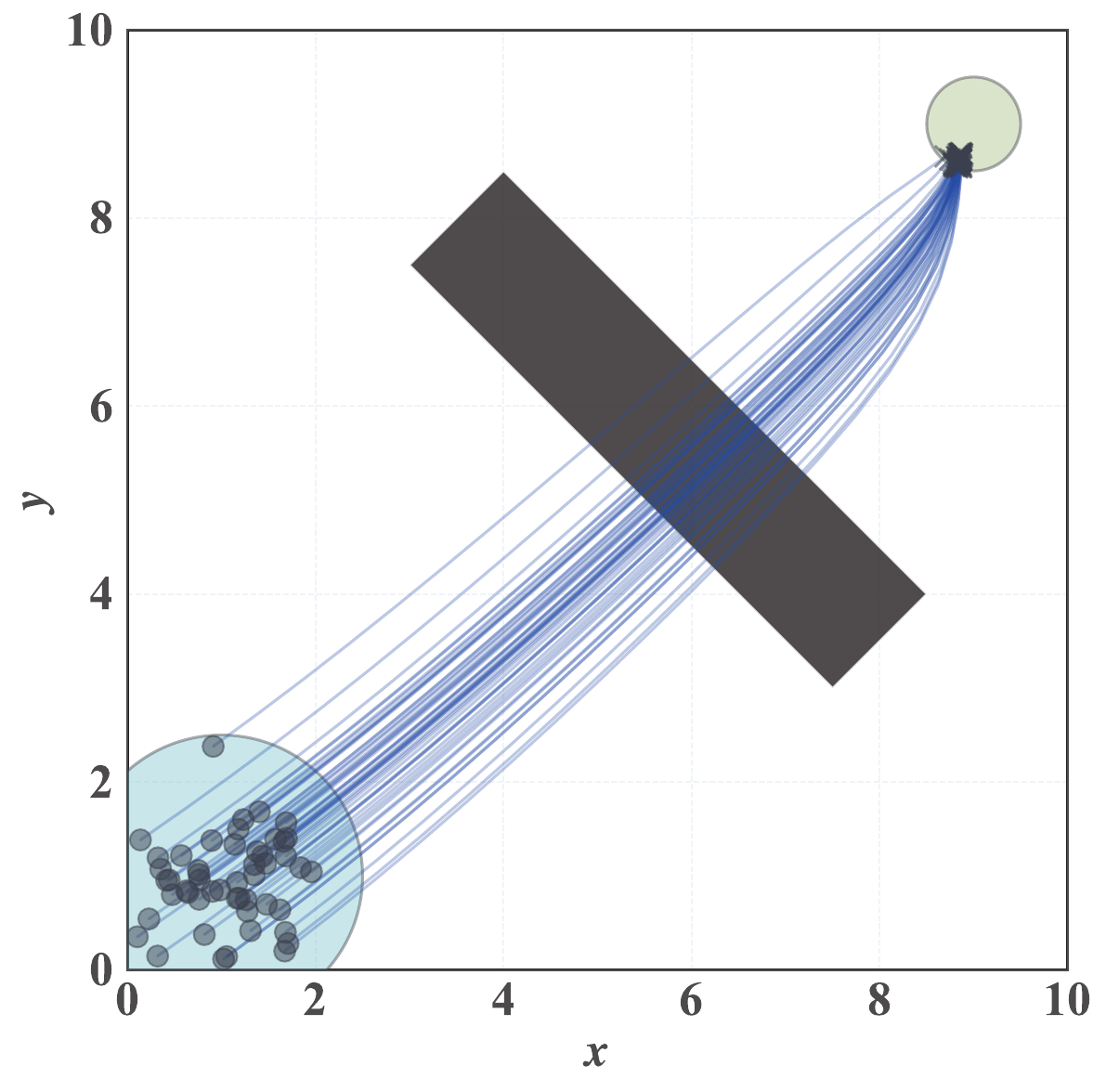}
        \end{minipage}
    }
    \subfigure[80000]{
        \begin{minipage}{0.48\columnwidth}
            \centering
            \includegraphics[width=\linewidth]{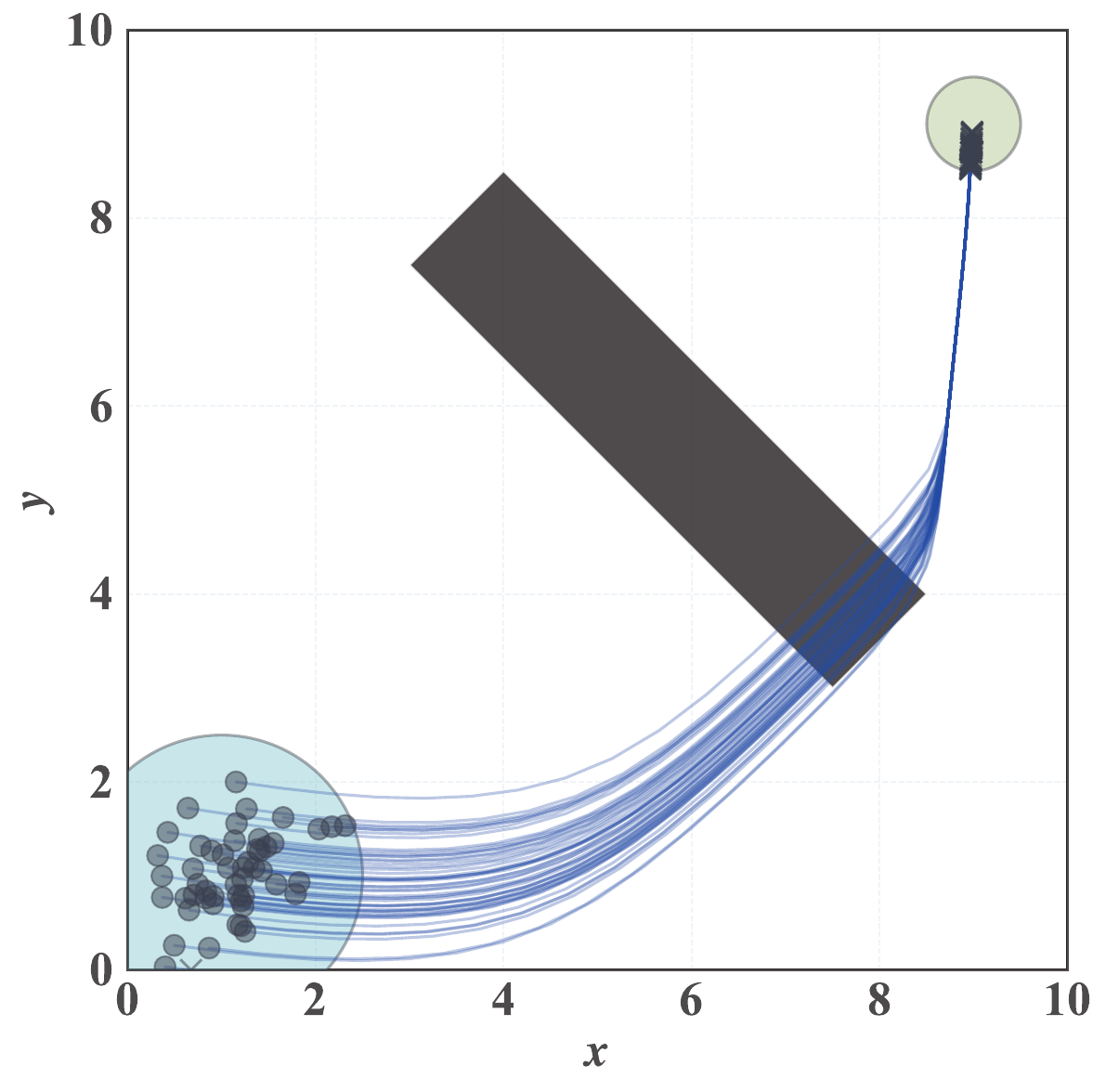}
        \end{minipage}
    }
    \subfigure[150000]{
        \begin{minipage}{0.48\columnwidth}
            \centering
            \includegraphics[width=\linewidth]{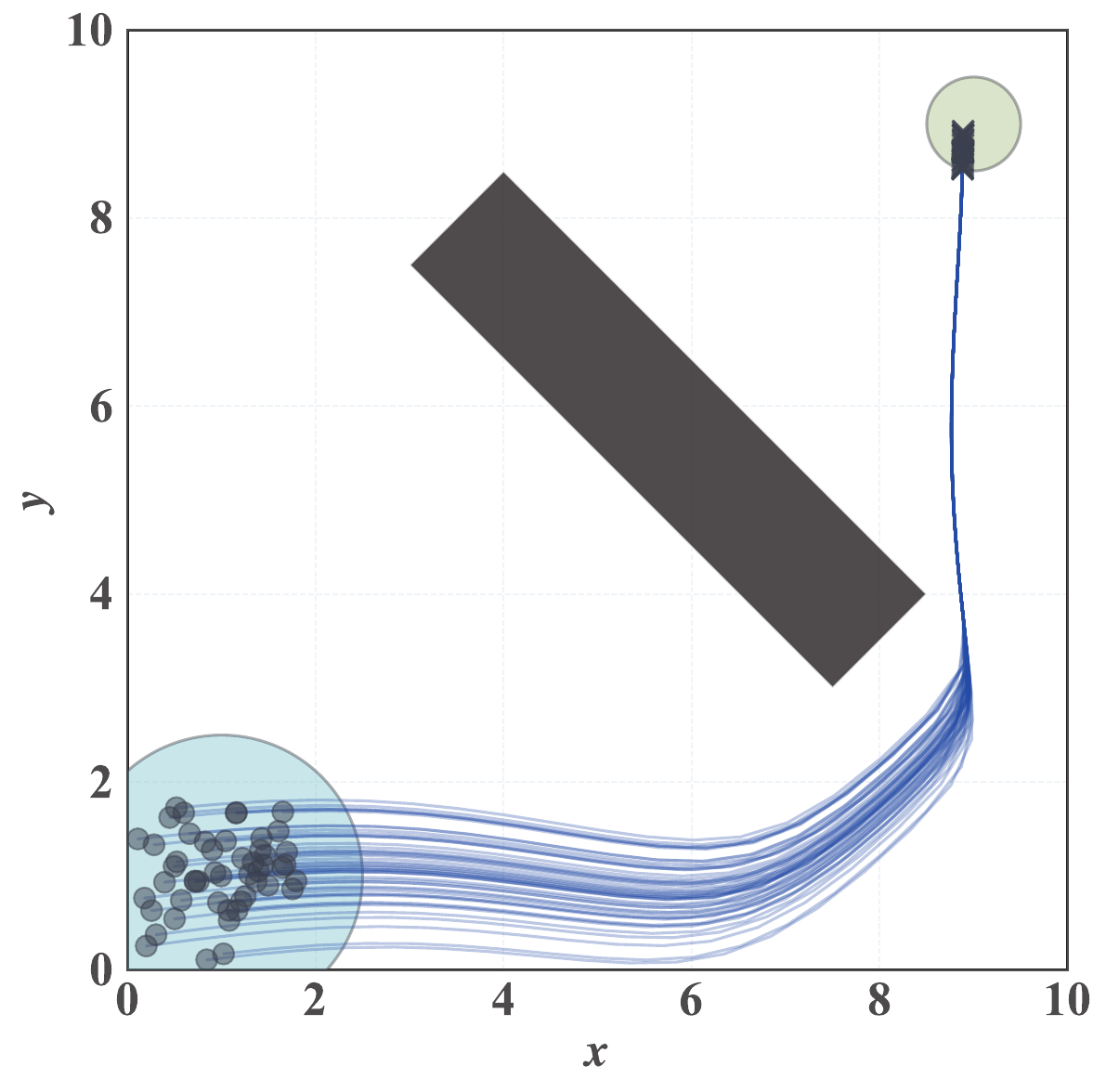}
        \end{minipage}
    }
    \subfigure[300000]{
        \begin{minipage}{0.48\columnwidth}
            \centering
            \includegraphics[width=\linewidth]{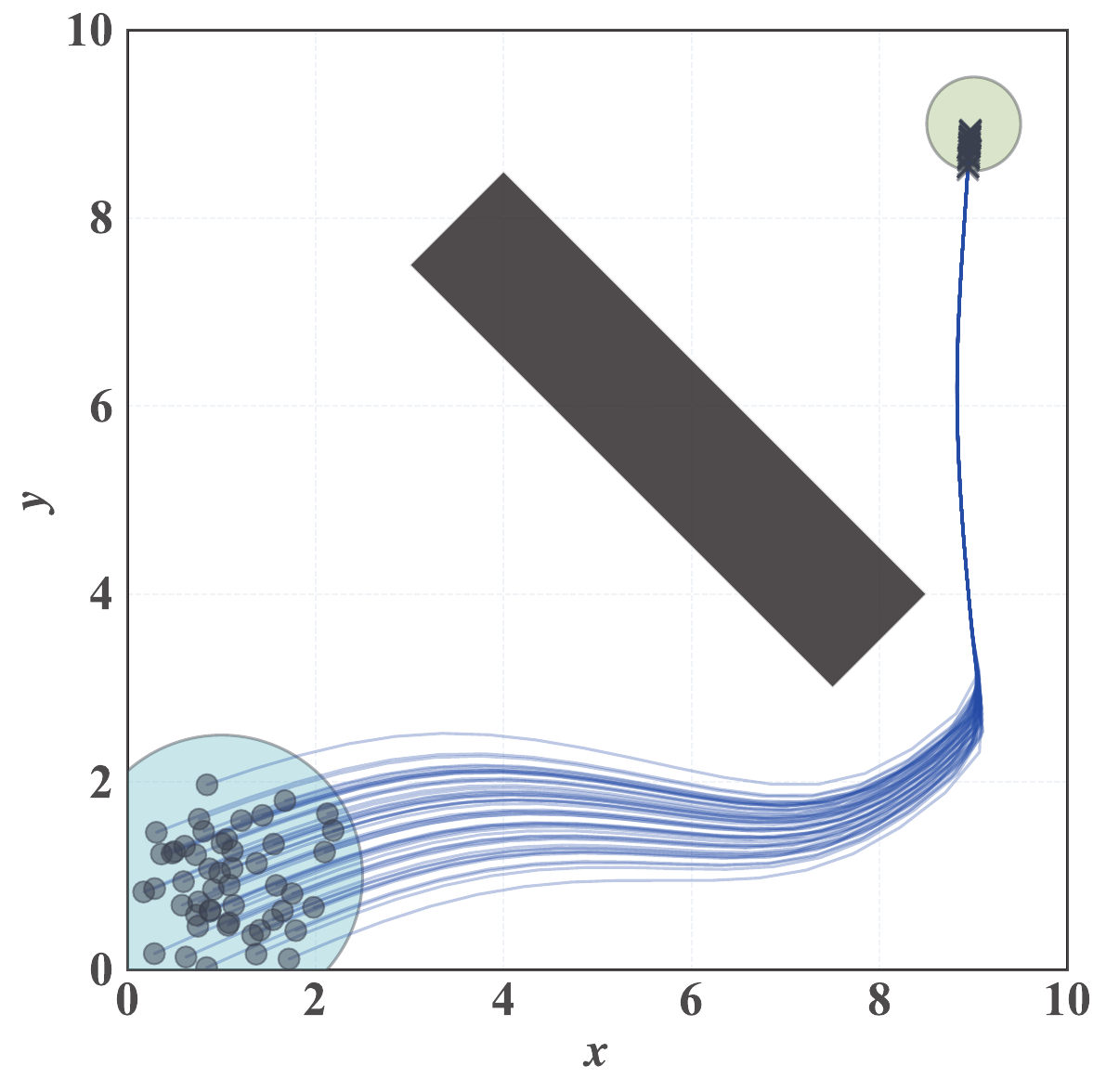}
        \end{minipage}
    }
    \subfigure[500000]{
        \begin{minipage}{0.48\columnwidth}
            \centering
            \includegraphics[width=\linewidth]{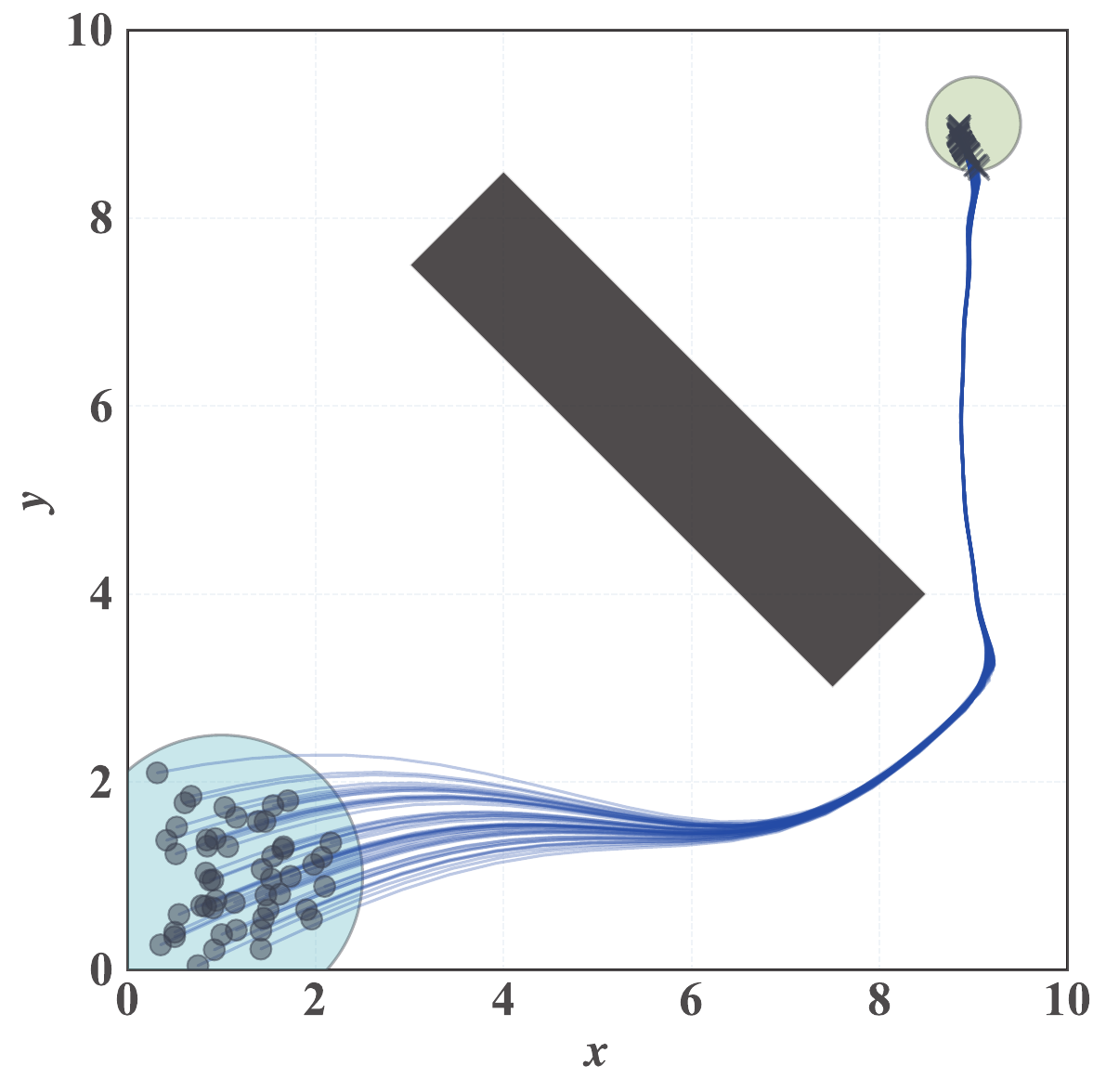}
        \end{minipage}
    }
    \subfigure[750000]{
        \begin{minipage}{0.48\columnwidth}
            \centering
            \includegraphics[width=\linewidth]{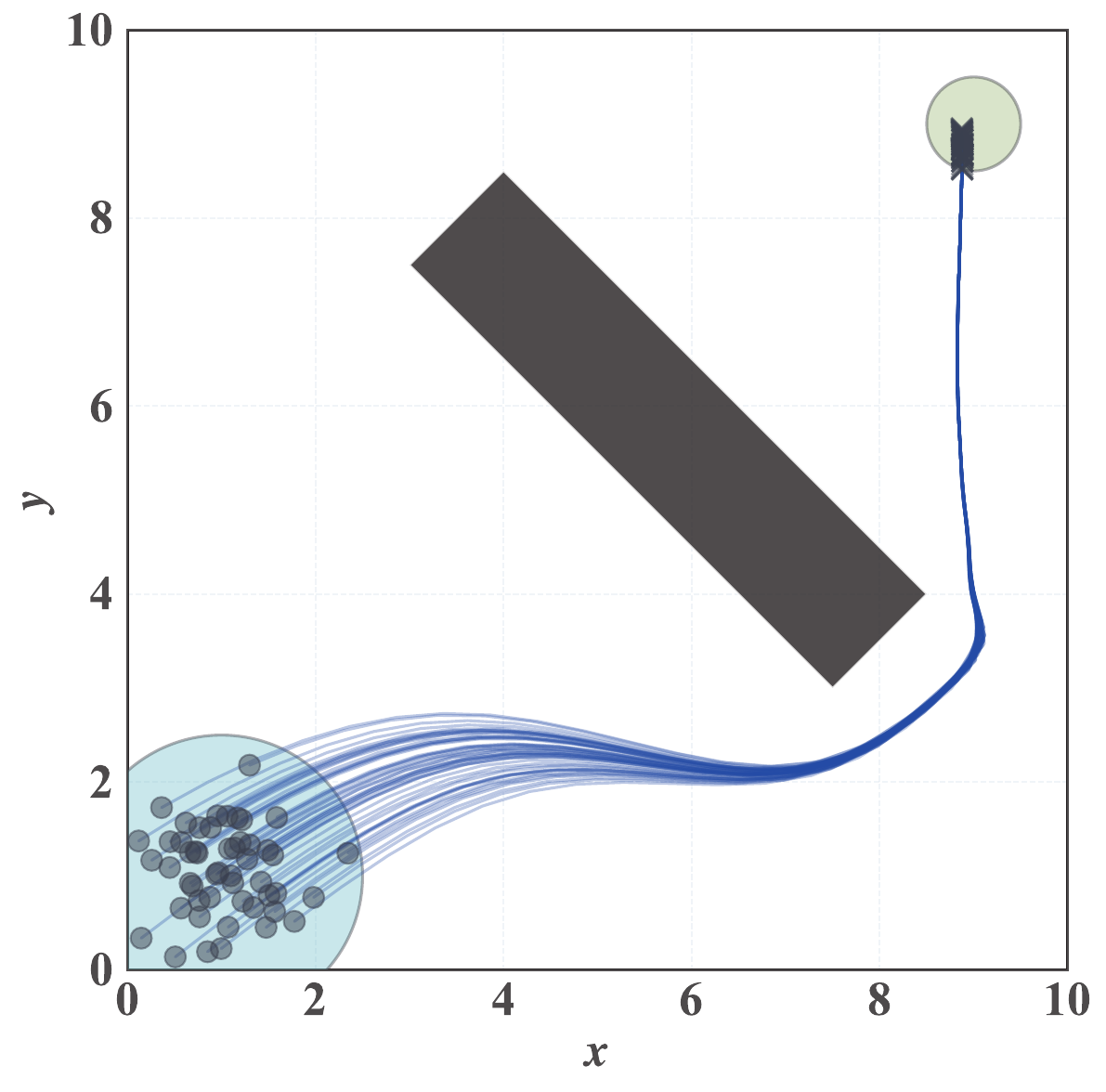}
        \end{minipage}
    }
    \subfigure[1000000]{
        \begin{minipage}{0.48\columnwidth}
            \centering
            \includegraphics[width=\linewidth]{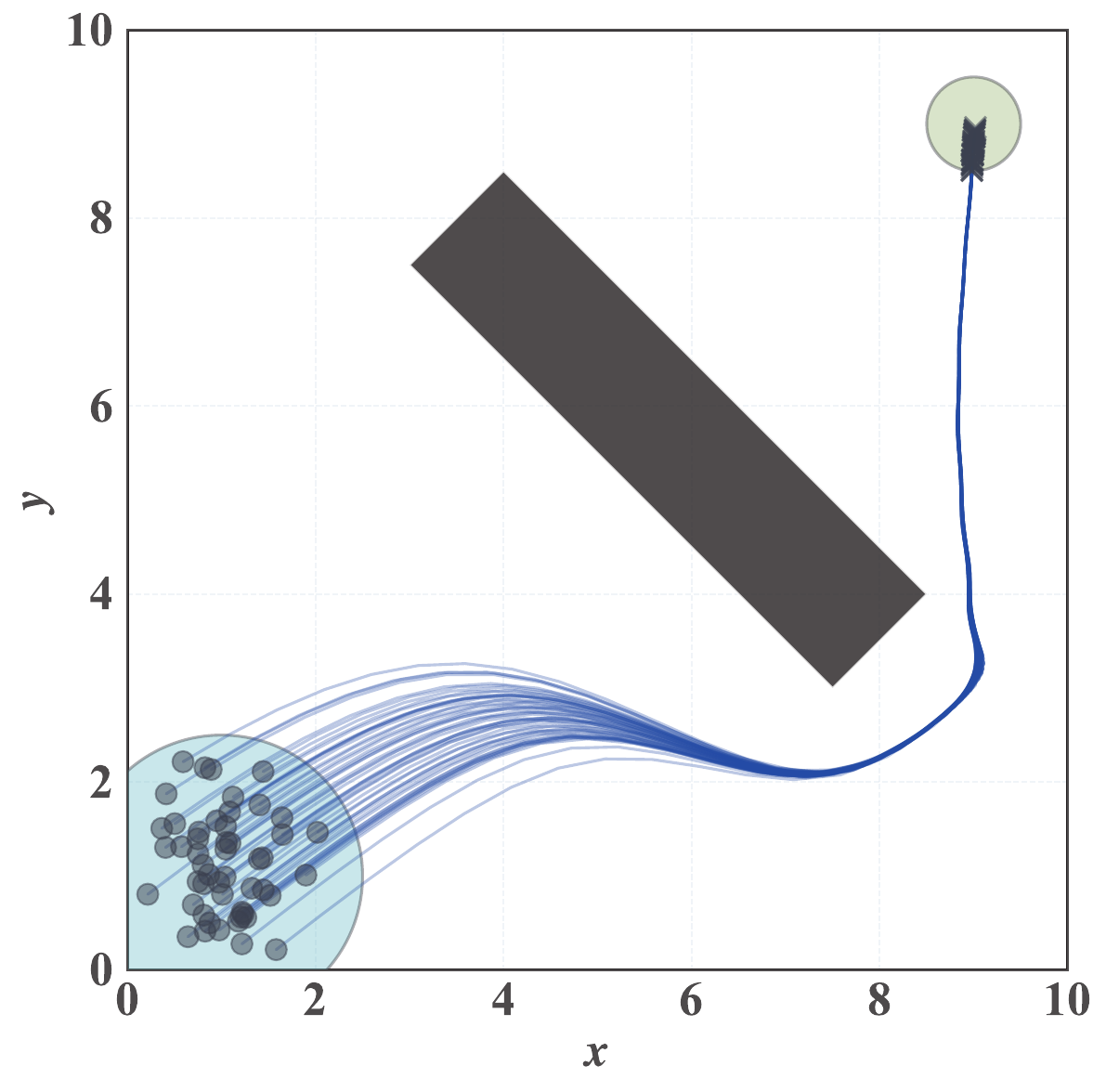}
        \end{minipage}
    }
    \caption{Training snapshots of \textbf{Nav2D-1G} environment from 10k steps to 1M steps. Blue region is the start region, green goal region is the only target.}    
    \label{app:fig:training snapshots of nav2d-1g}
\end{figure*}

\begin{figure*}[!t]
    \centering
    \subfigure[10000]{
        \begin{minipage}{0.48\columnwidth}
            \centering
            \includegraphics[width=\linewidth]{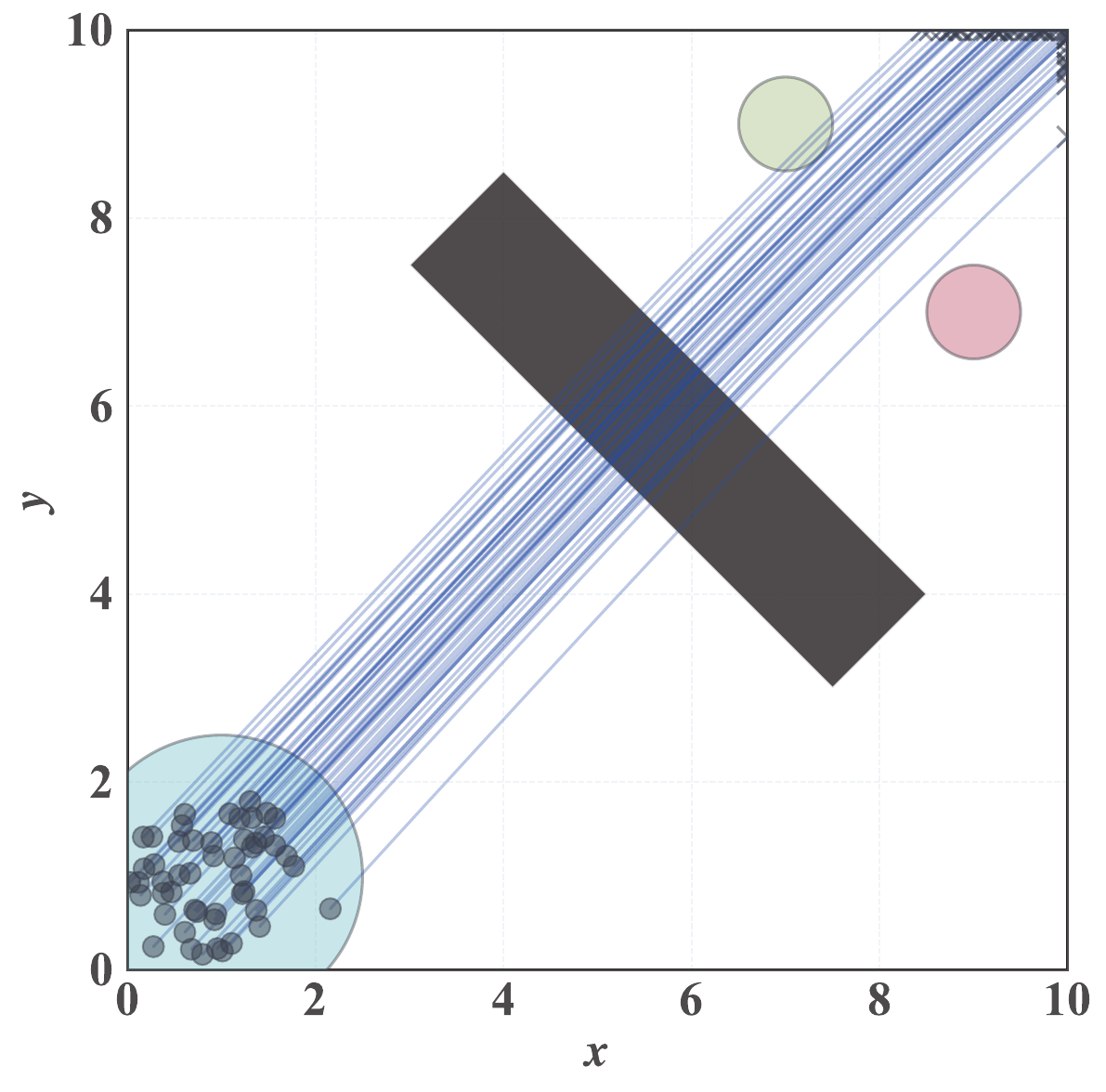}
        \end{minipage}
    }
    \subfigure[30000]{
        \begin{minipage}{0.48\columnwidth}
            \centering
            \includegraphics[width=\linewidth]{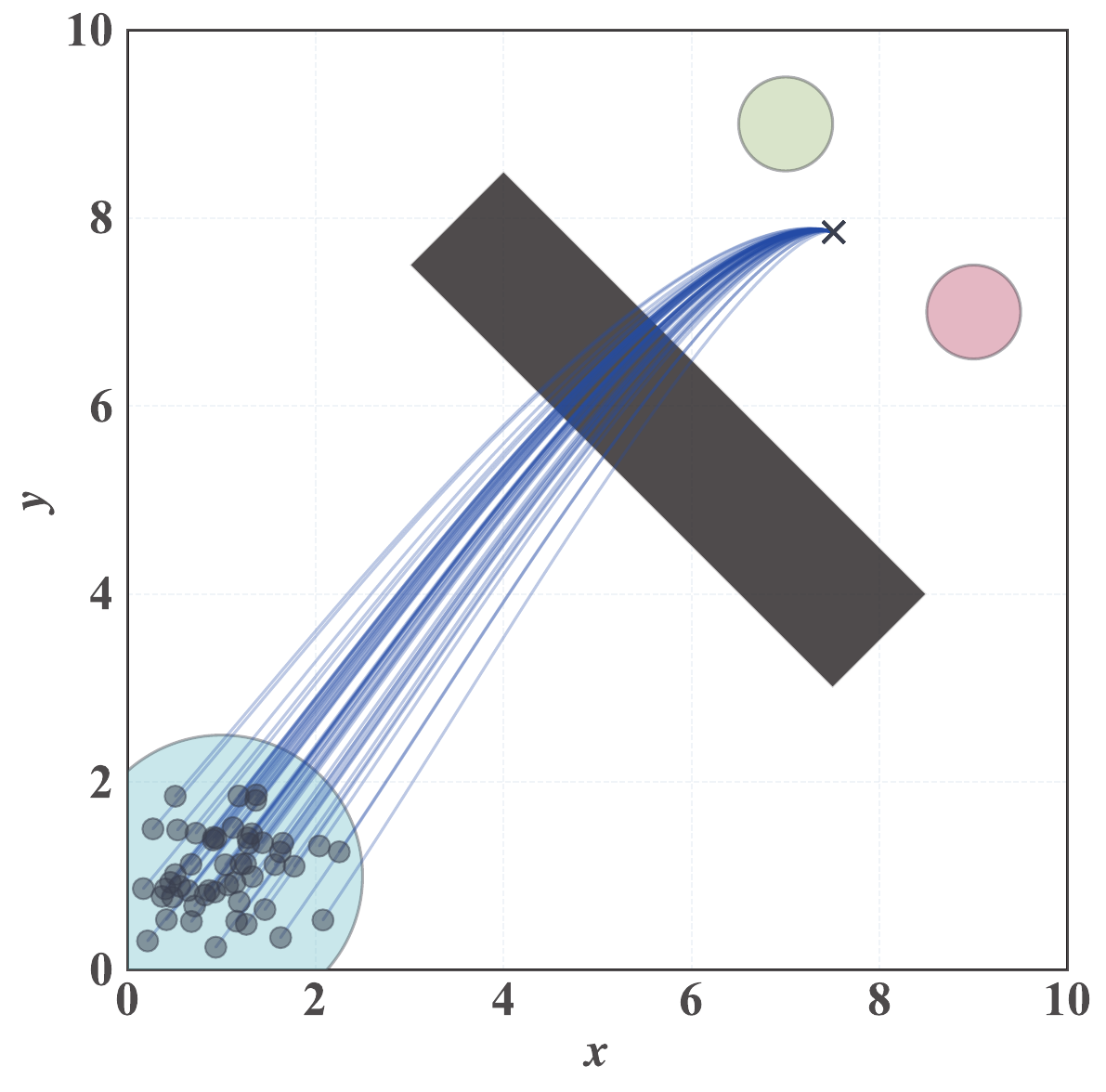}
        \end{minipage}
    }
    \subfigure[80000]{
        \begin{minipage}{0.48\columnwidth}
            \centering
            \includegraphics[width=\linewidth]{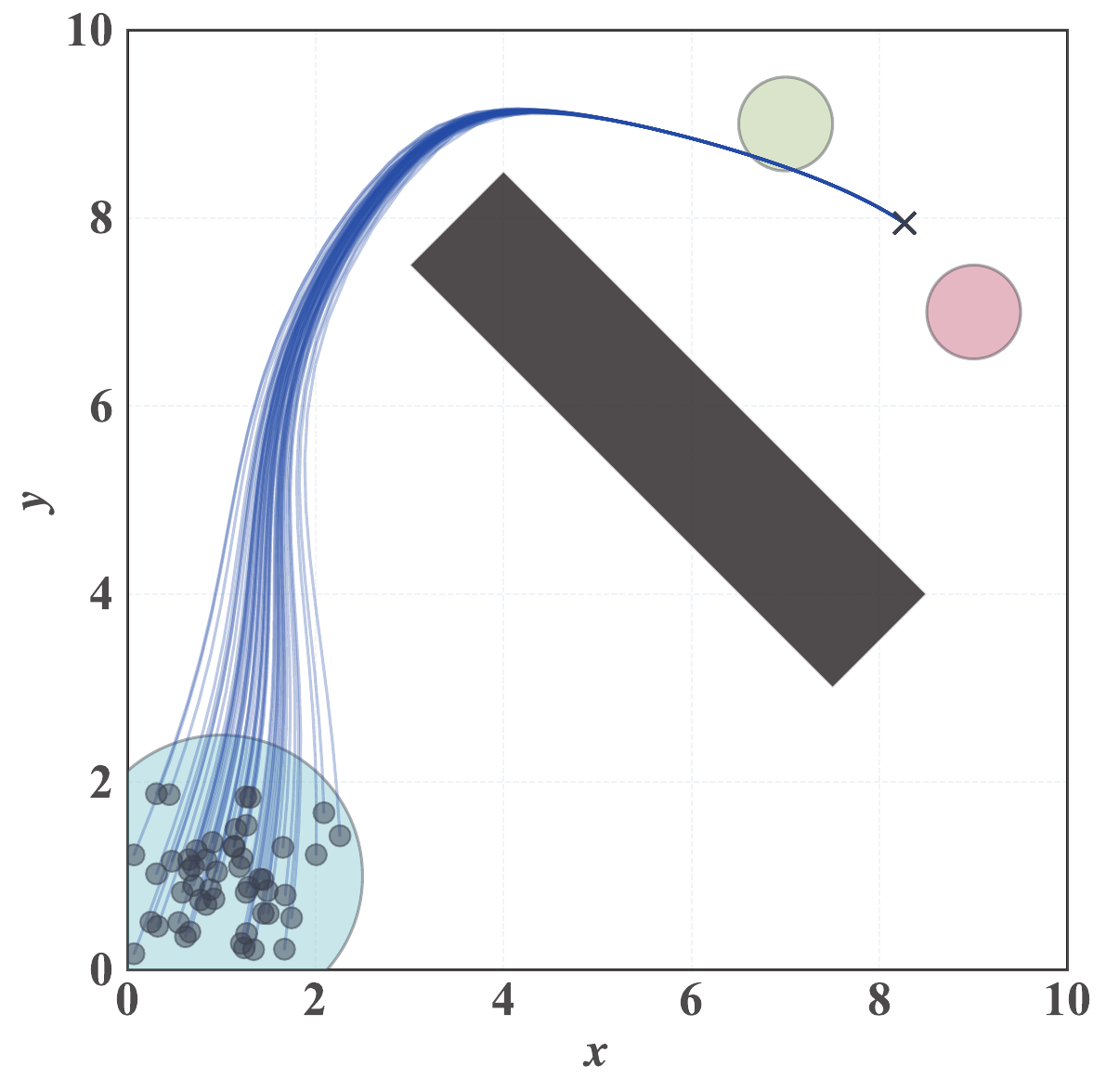}
        \end{minipage}
    }
    \subfigure[150000]{
        \begin{minipage}{0.48\columnwidth}
            \centering
            \includegraphics[width=\linewidth]{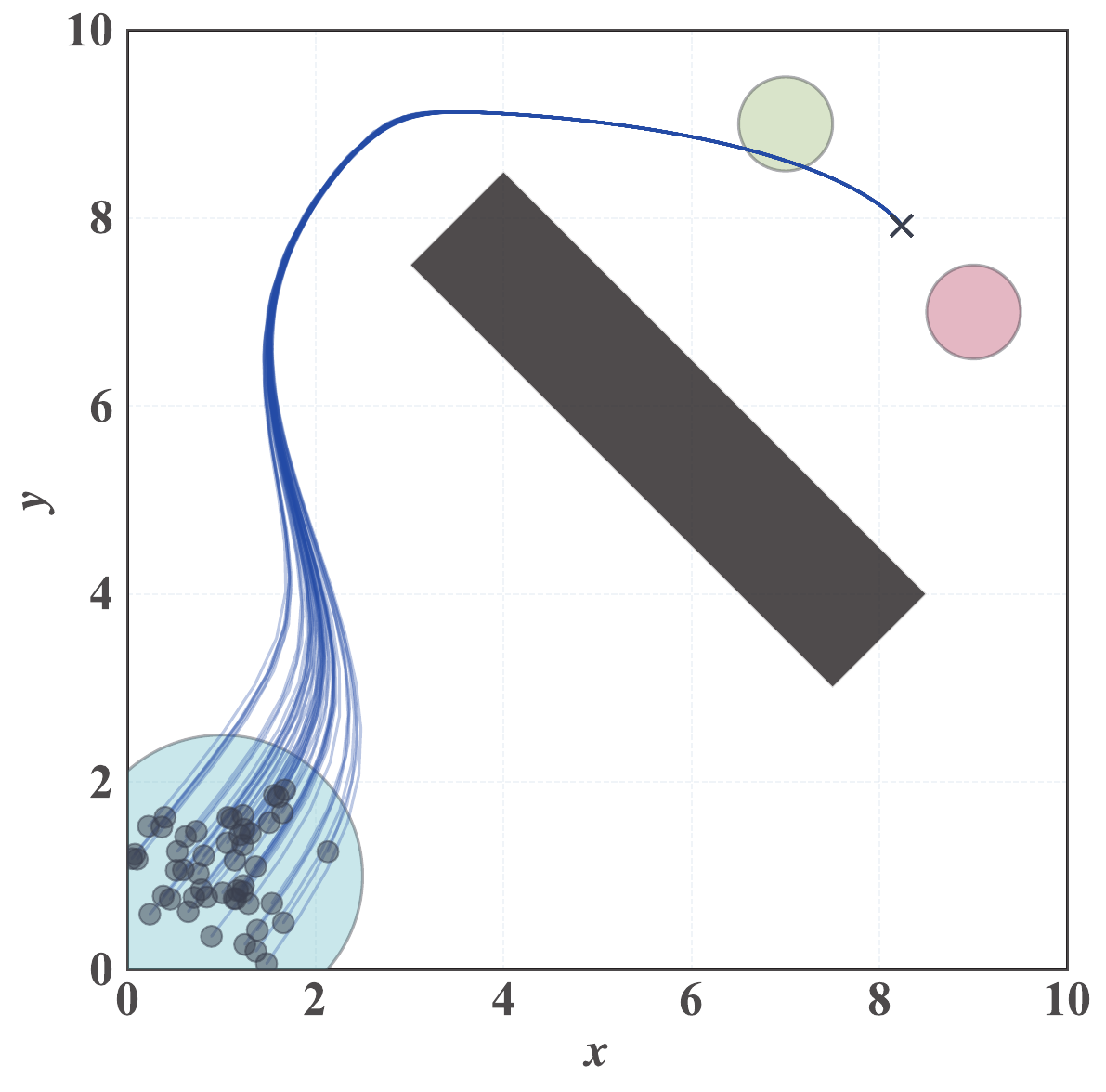}
        \end{minipage}
    }
    \subfigure[300000]{
        \begin{minipage}{0.48\columnwidth}
            \centering
            \includegraphics[width=\linewidth]{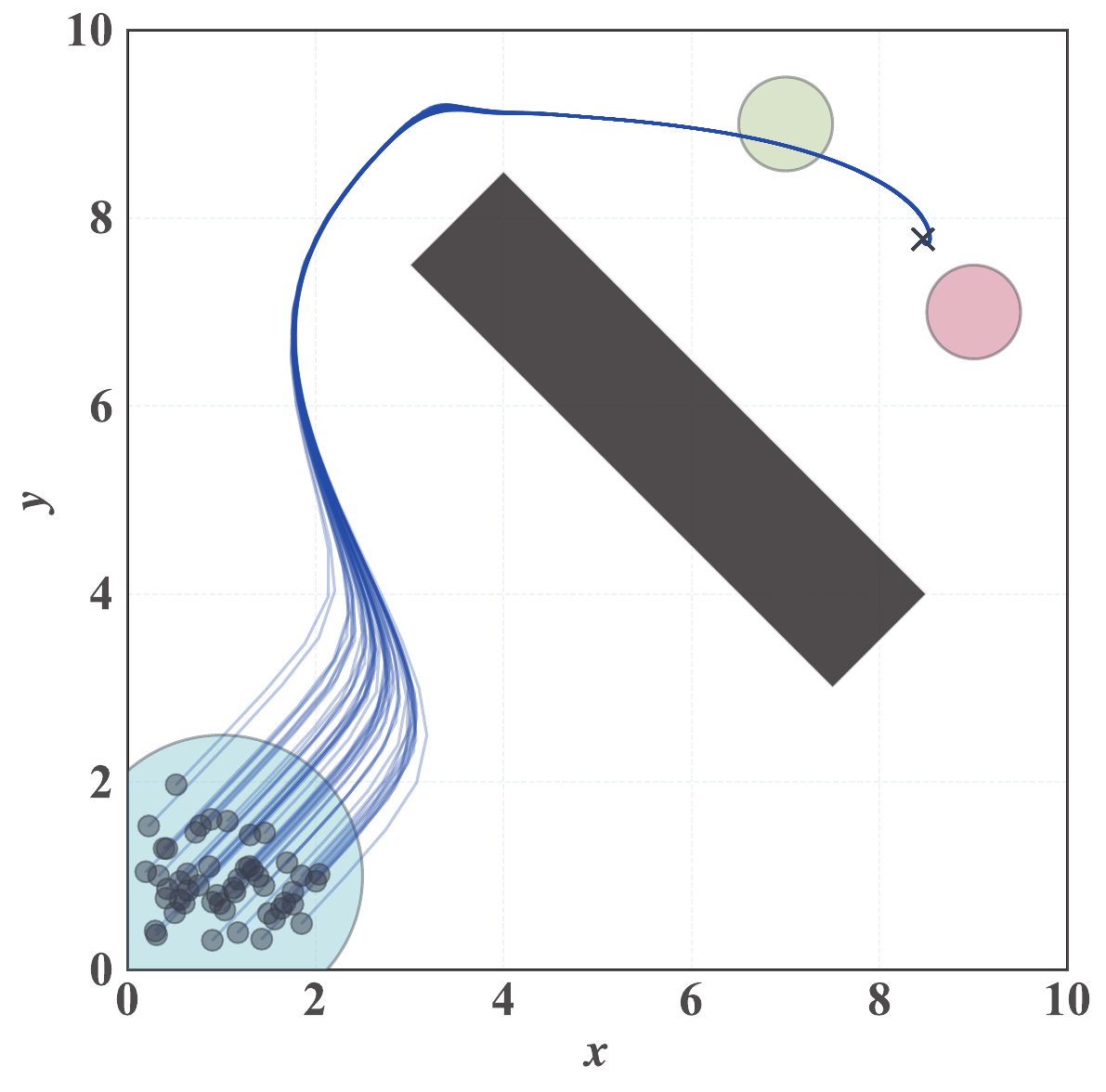}
        \end{minipage}
    }
    \subfigure[500000]{
        \begin{minipage}{0.48\columnwidth}
            \centering
            \includegraphics[width=\linewidth]{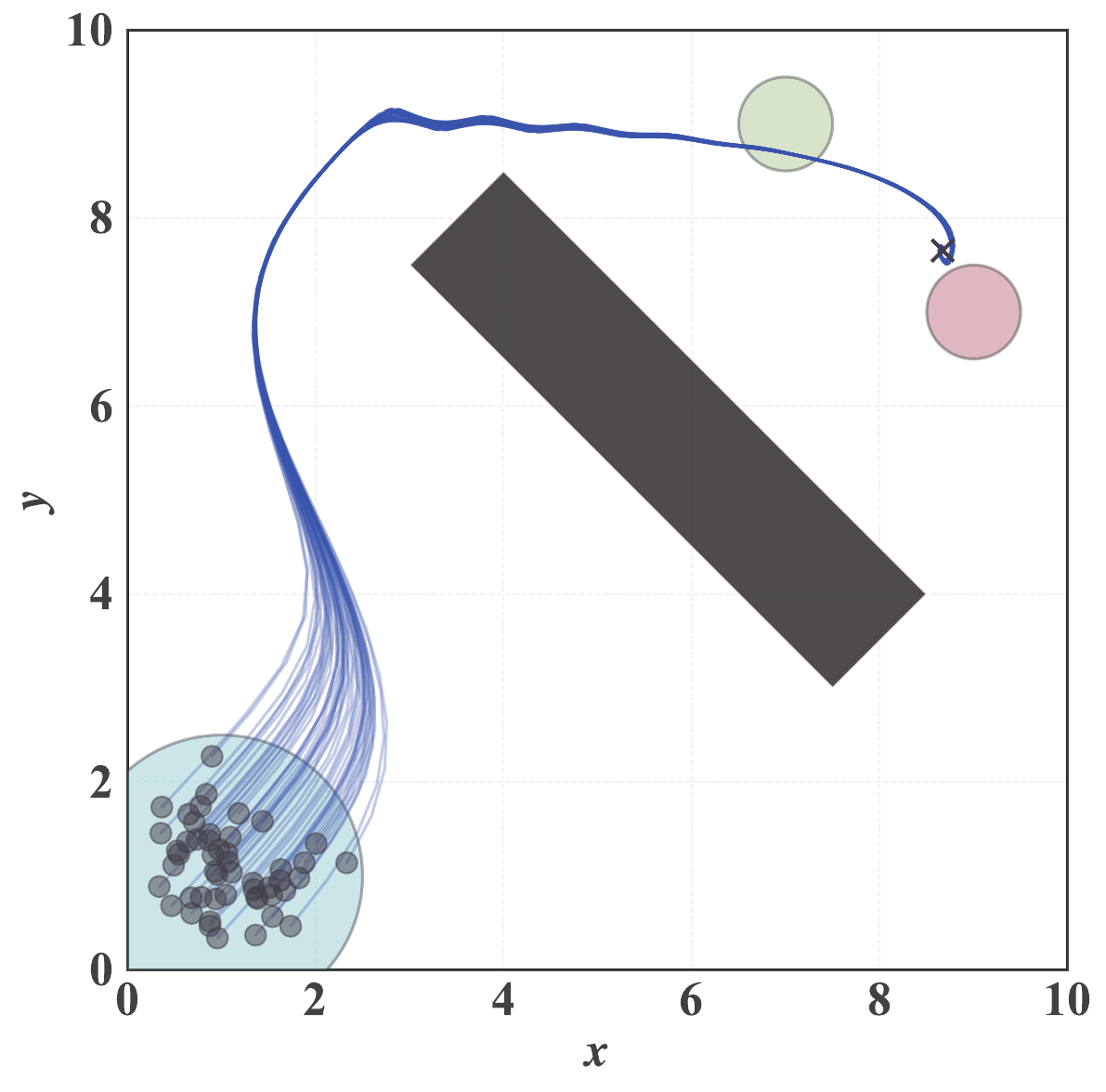}
        \end{minipage}
    }
    \subfigure[750000]{
        \begin{minipage}{0.48\columnwidth}
            \centering
            \includegraphics[width=\linewidth]{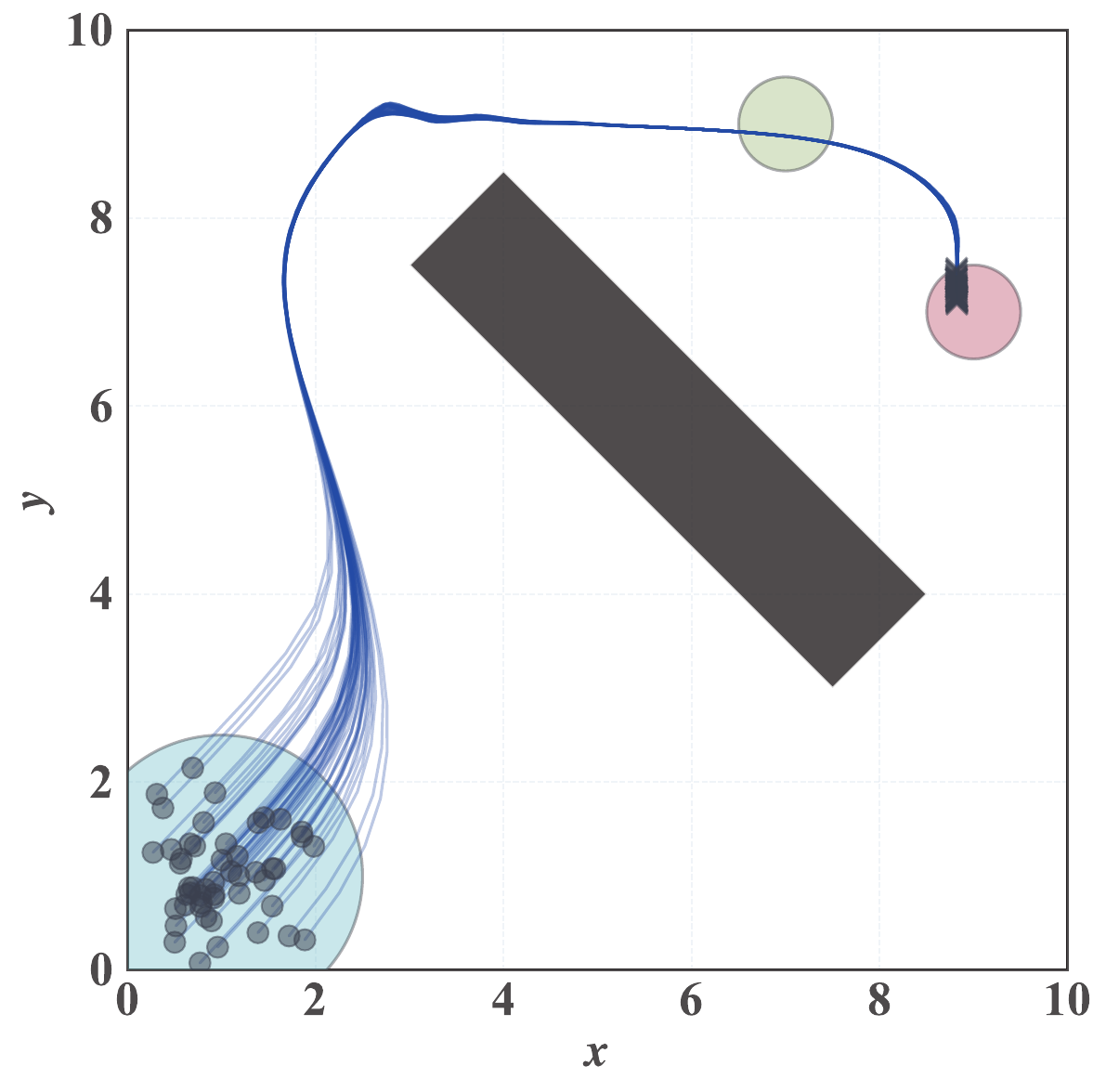}
        \end{minipage}
    }
    \subfigure[1000000]{
        \begin{minipage}{0.48\columnwidth}
            \centering
            \includegraphics[width=\linewidth]{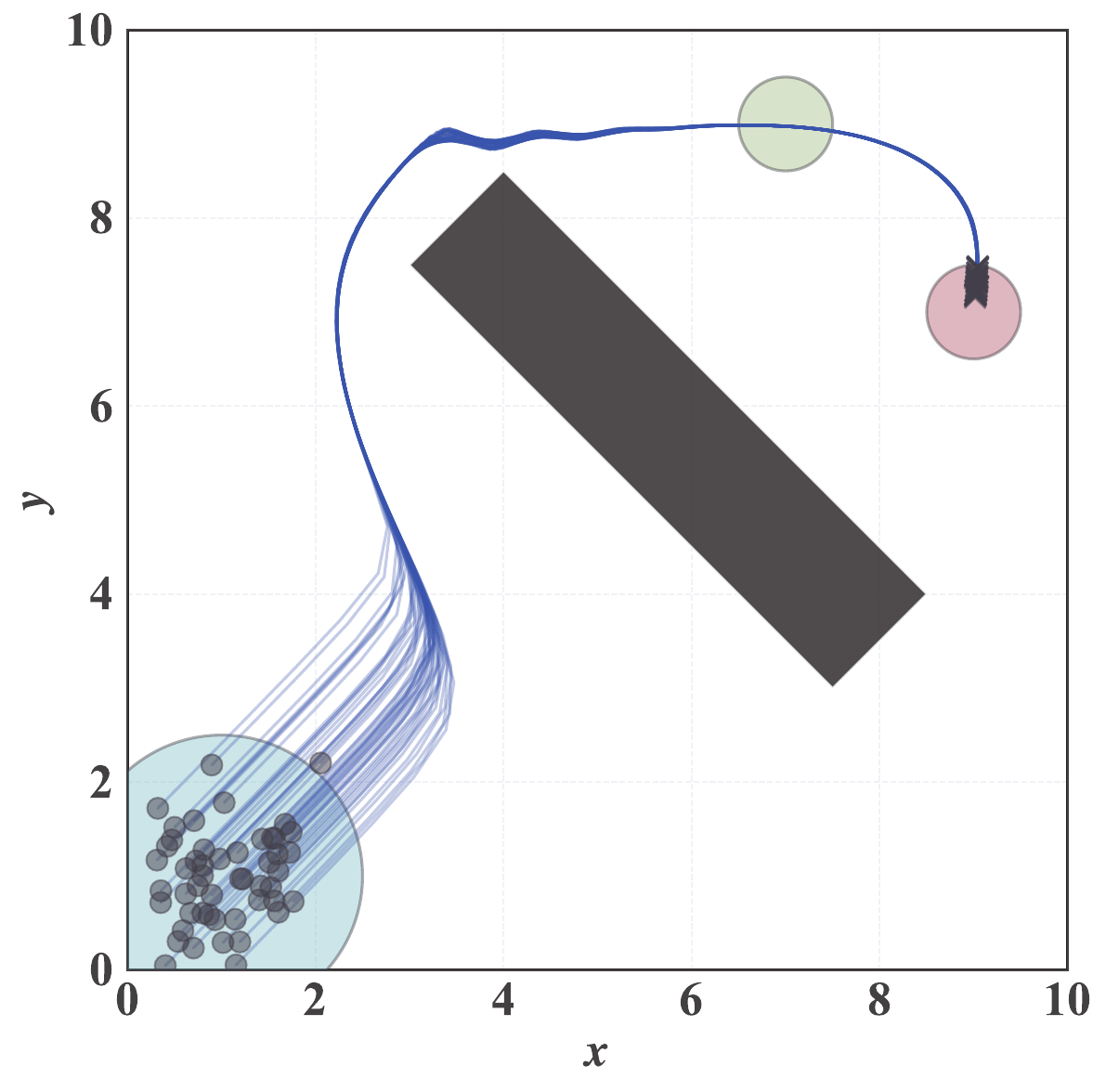}
        \end{minipage}
    }
    \caption{Training snapshots of \textbf{Nav2D-2G} environment from 10k steps to 1M steps. Green goal region is the first target, and red goal region is the second target.}    
    \label{app:fig:training snapshots of nav2d-2g}
\end{figure*}

\begin{figure*}[!t]
    \centering
    \subfigure[10000]{
        \begin{minipage}{0.48\columnwidth}
            \centering
            \includegraphics[width=\linewidth]{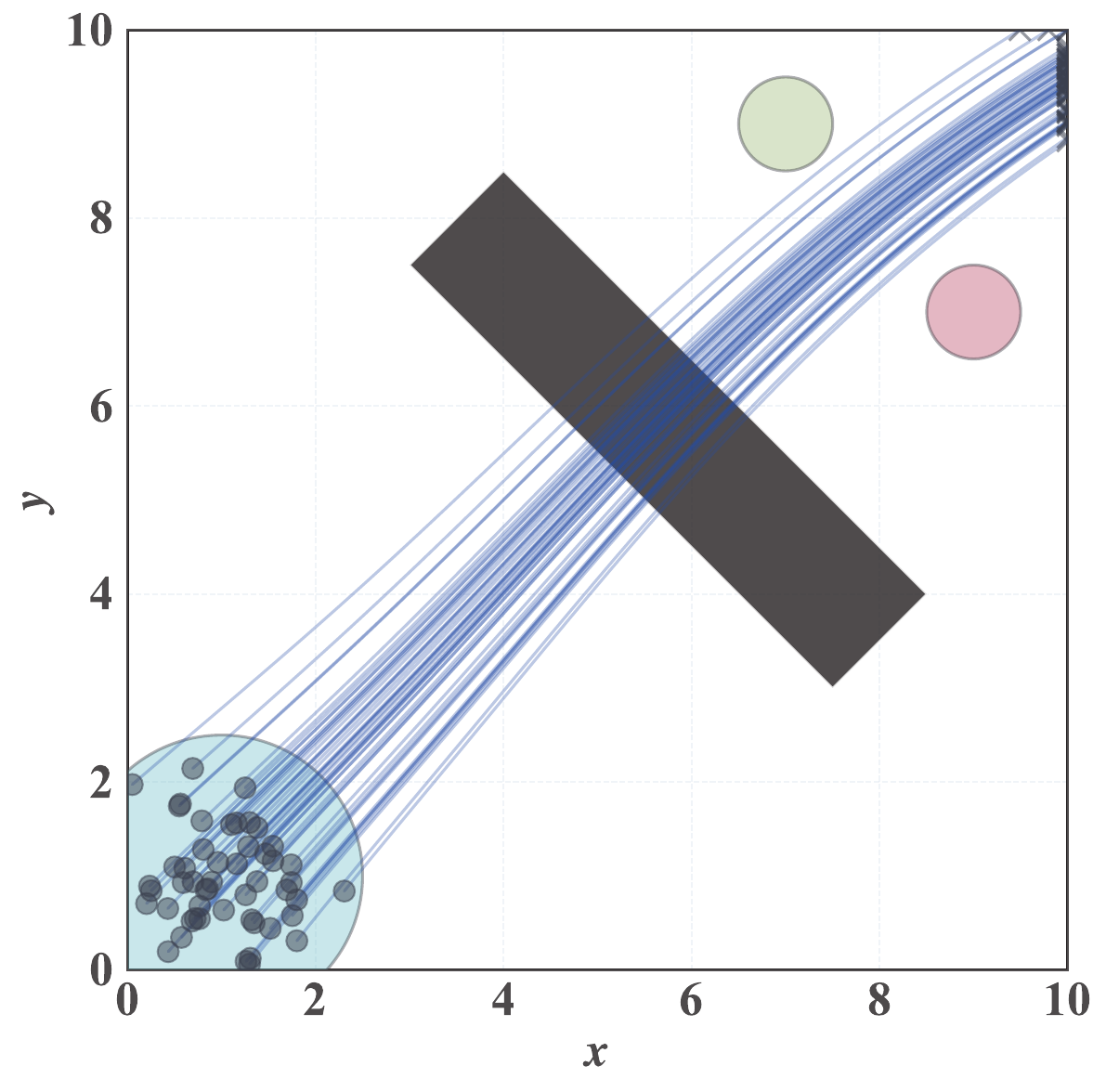}
        \end{minipage}
    }
    \subfigure[30000]{
        \begin{minipage}{0.48\columnwidth}
            \centering
            \includegraphics[width=\linewidth]{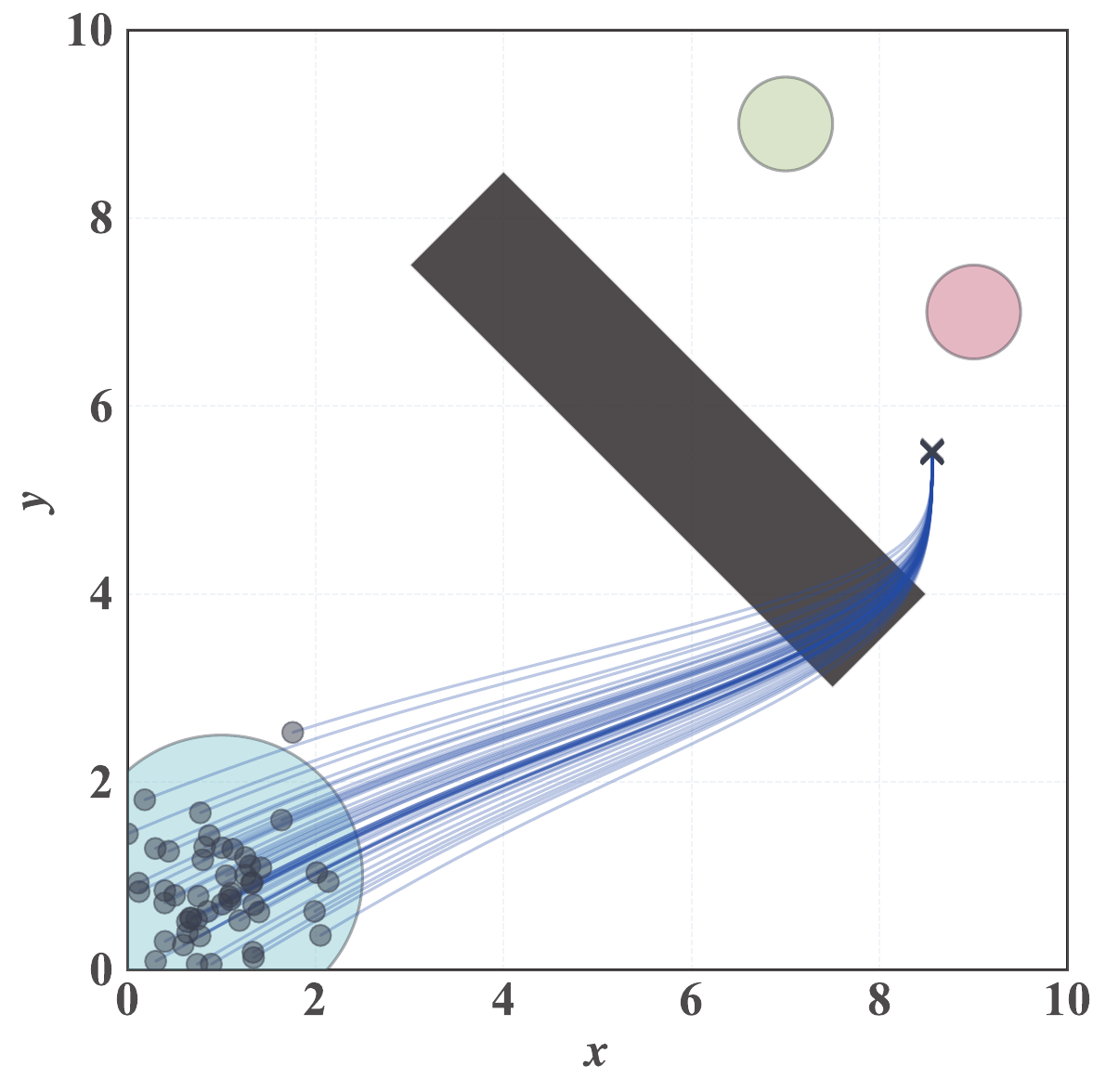}
        \end{minipage}
    }
    \subfigure[80000]{
        \begin{minipage}{0.48\columnwidth}
            \centering
            \includegraphics[width=\linewidth]{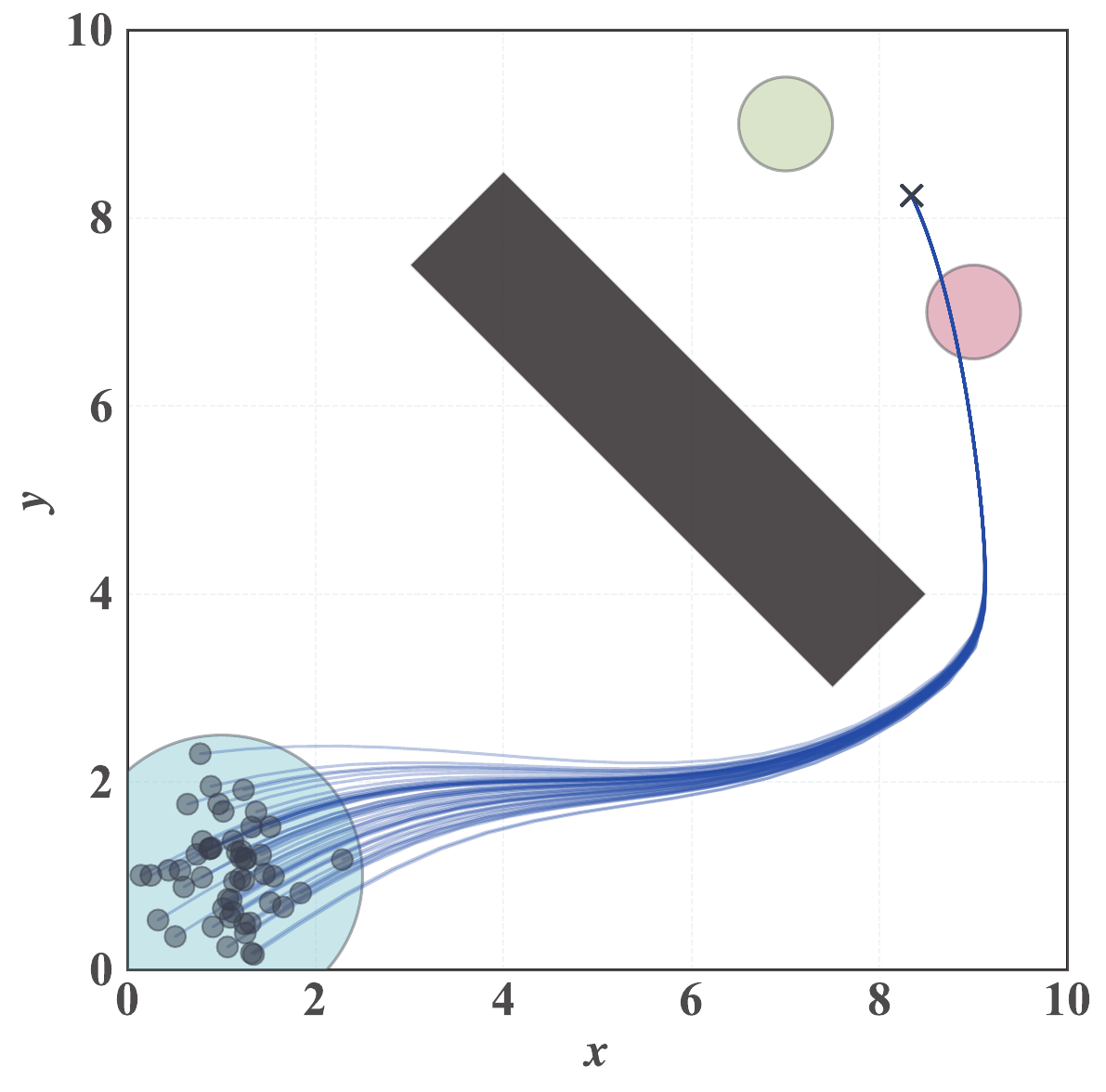}
        \end{minipage}
    }
    \subfigure[150000]{
        \begin{minipage}{0.48\columnwidth}
            \centering
            \includegraphics[width=\linewidth]{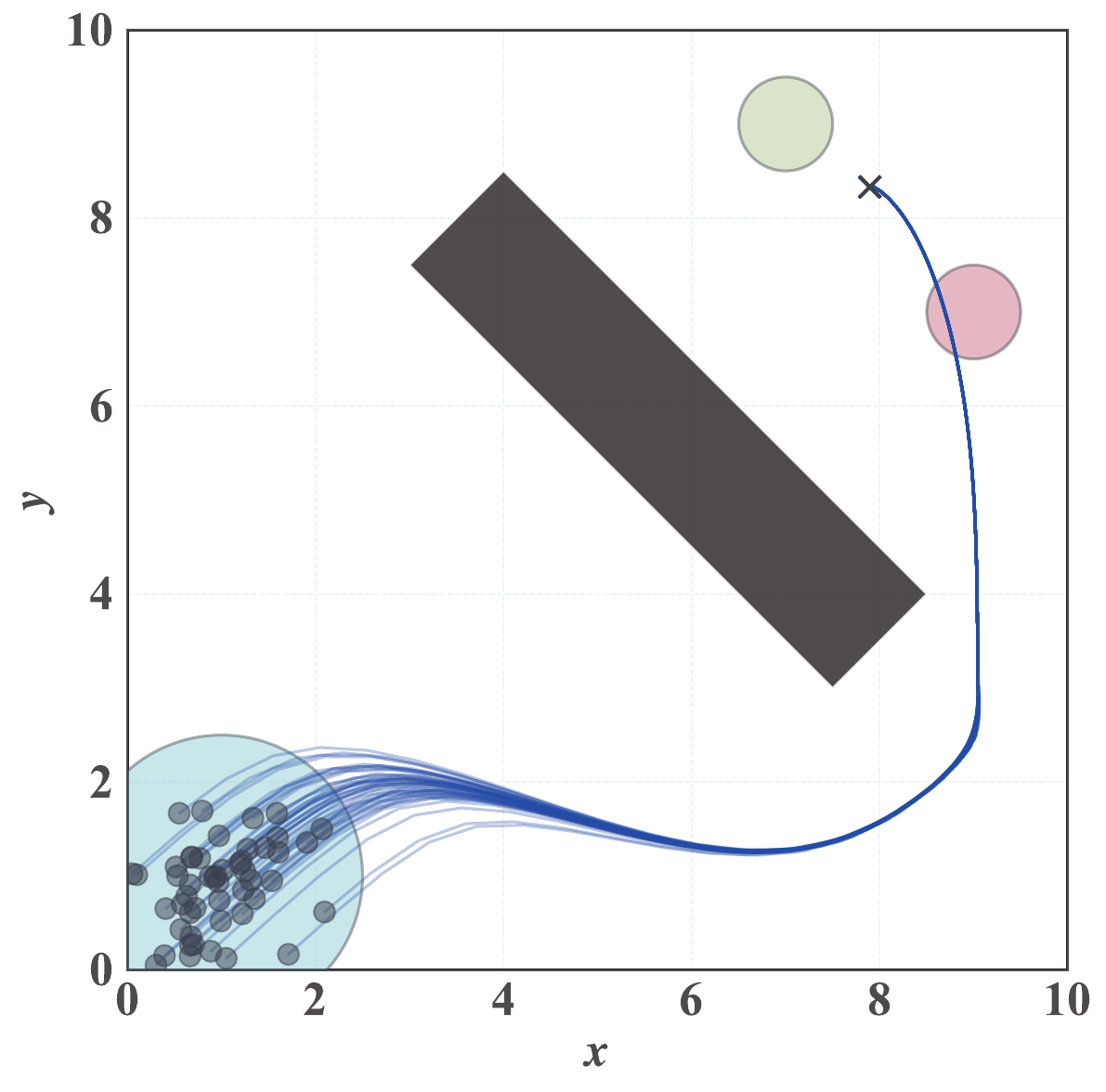}
        \end{minipage}
    }
    \subfigure[300000]{
        \begin{minipage}{0.48\columnwidth}
            \centering
            \includegraphics[width=\linewidth]{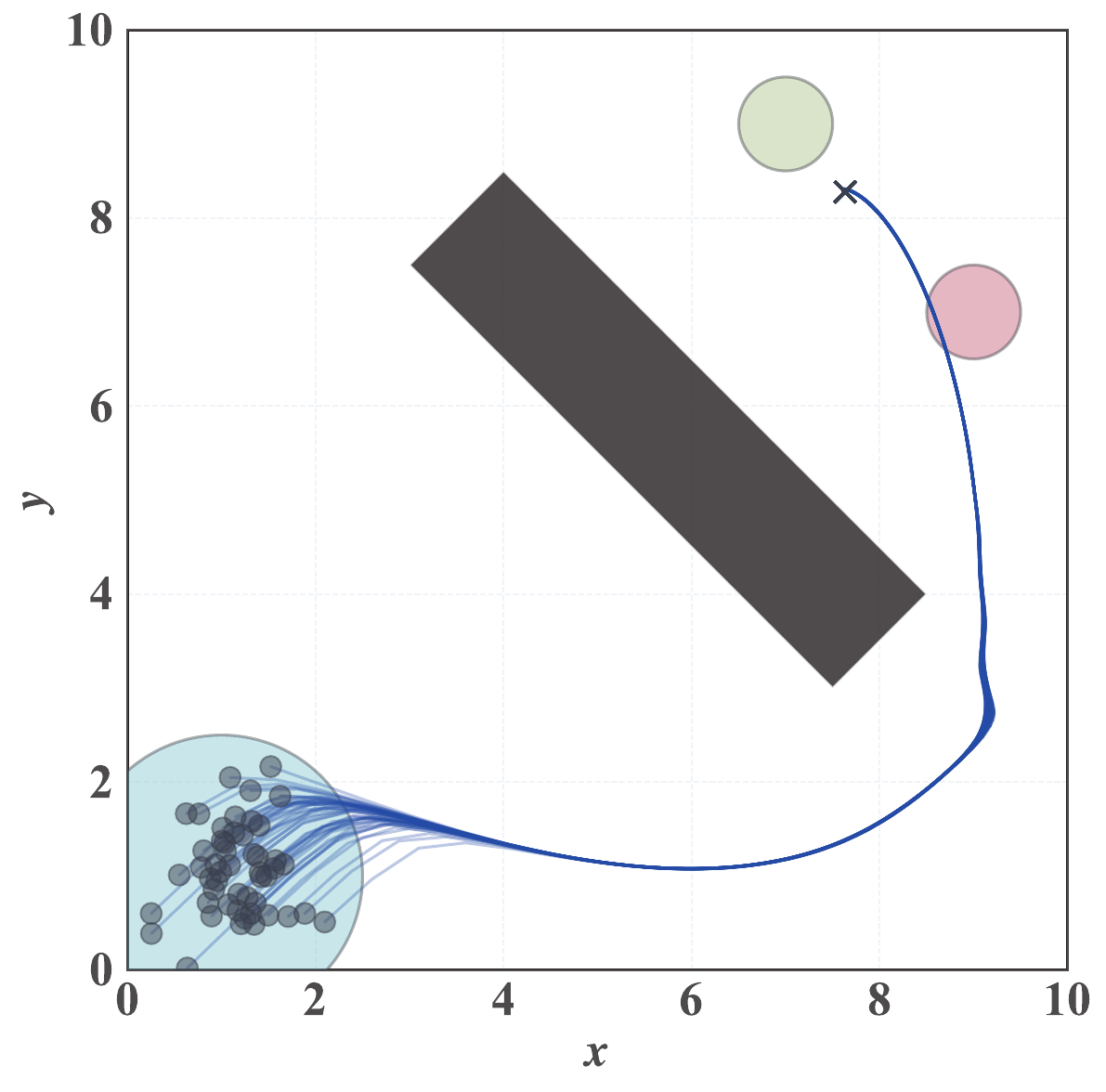}
        \end{minipage}
    }
    \subfigure[500000]{
        \begin{minipage}{0.48\columnwidth}
            \centering
            \includegraphics[width=\linewidth]{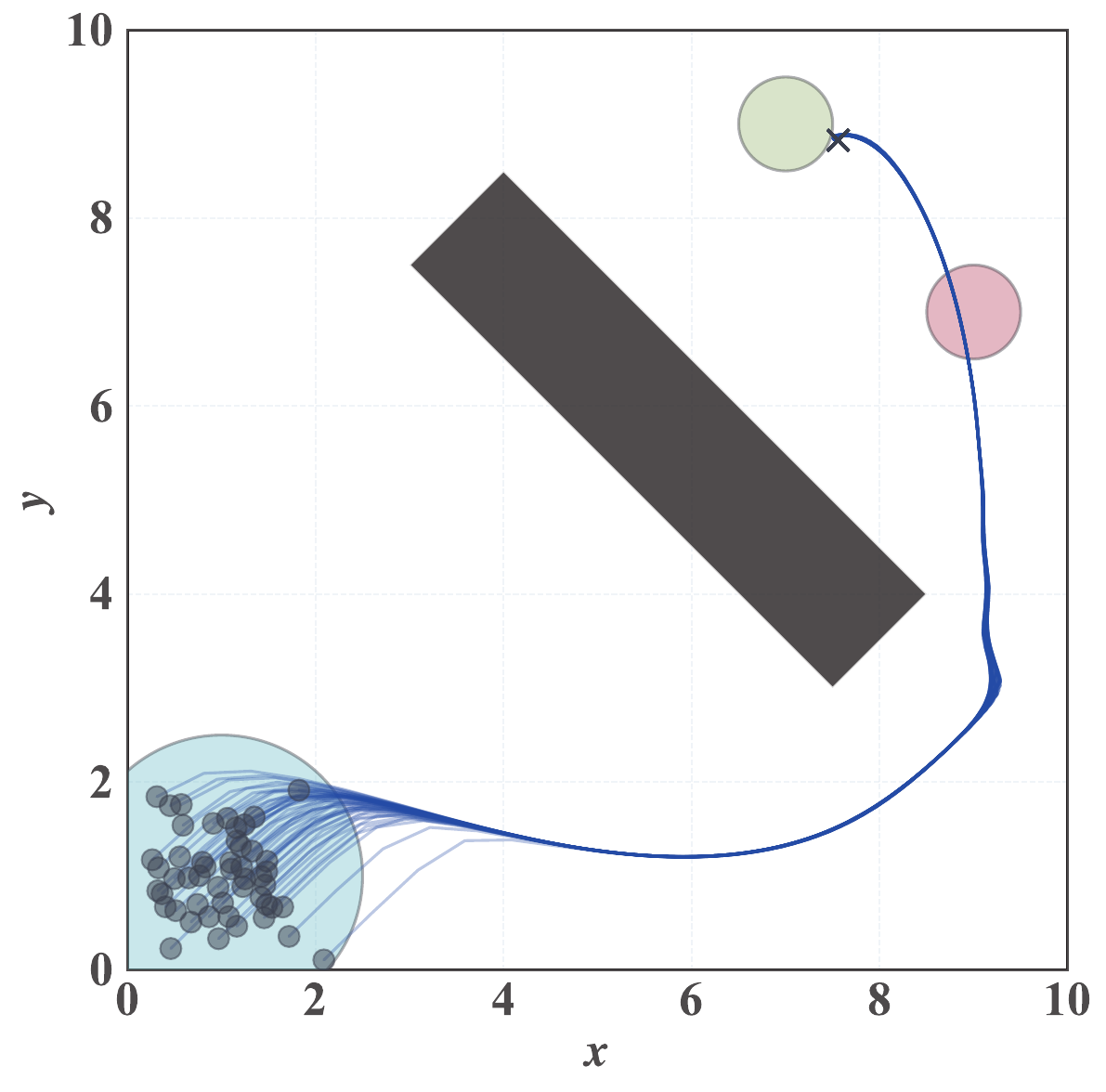}
        \end{minipage}
    }
    \subfigure[750000]{
        \begin{minipage}{0.48\columnwidth}
            \centering
            \includegraphics[width=\linewidth]{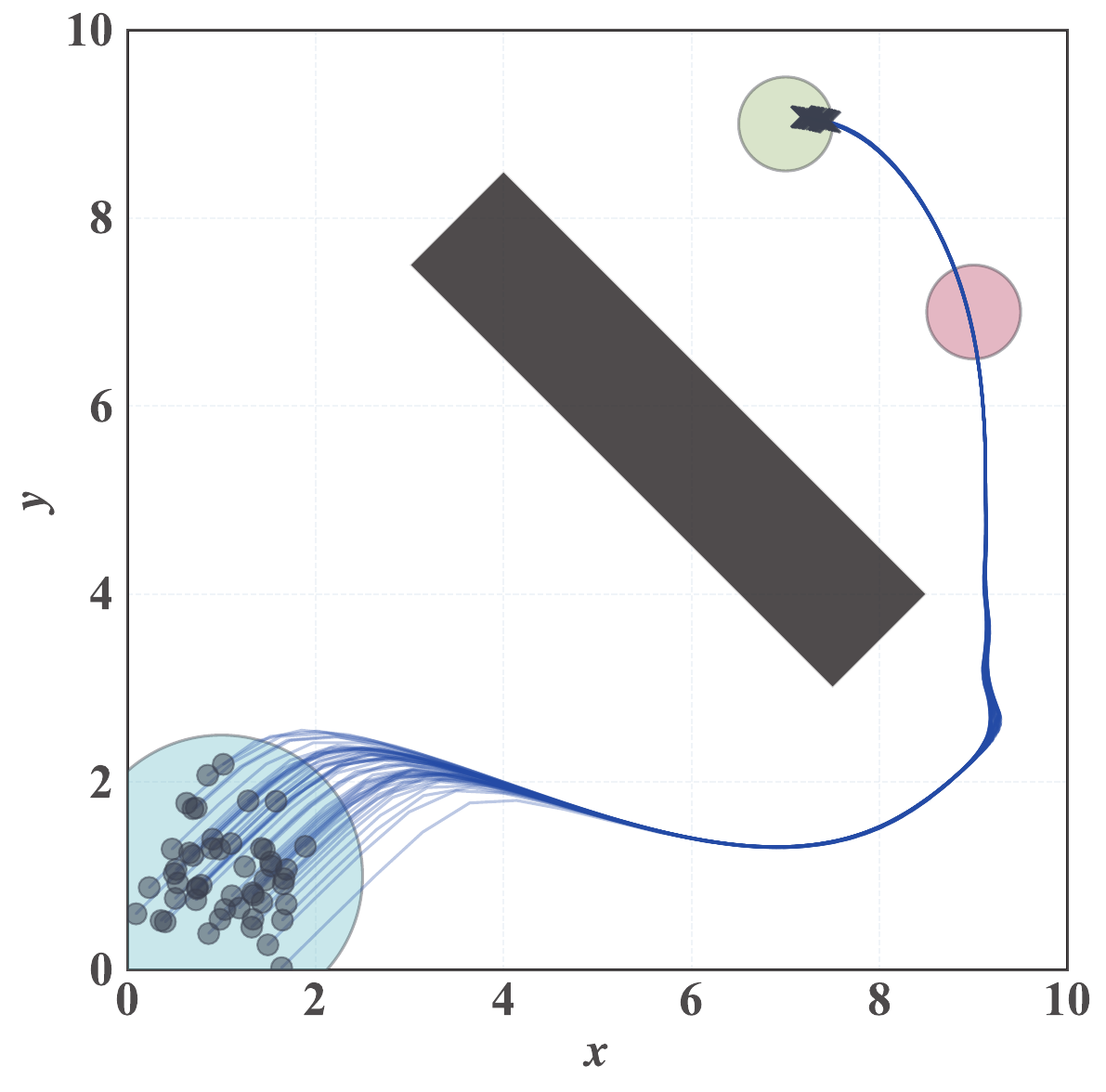}
        \end{minipage}
    }
    \subfigure[1000000]{
        \begin{minipage}{0.48\columnwidth}
            \centering
            \includegraphics[width=\linewidth]{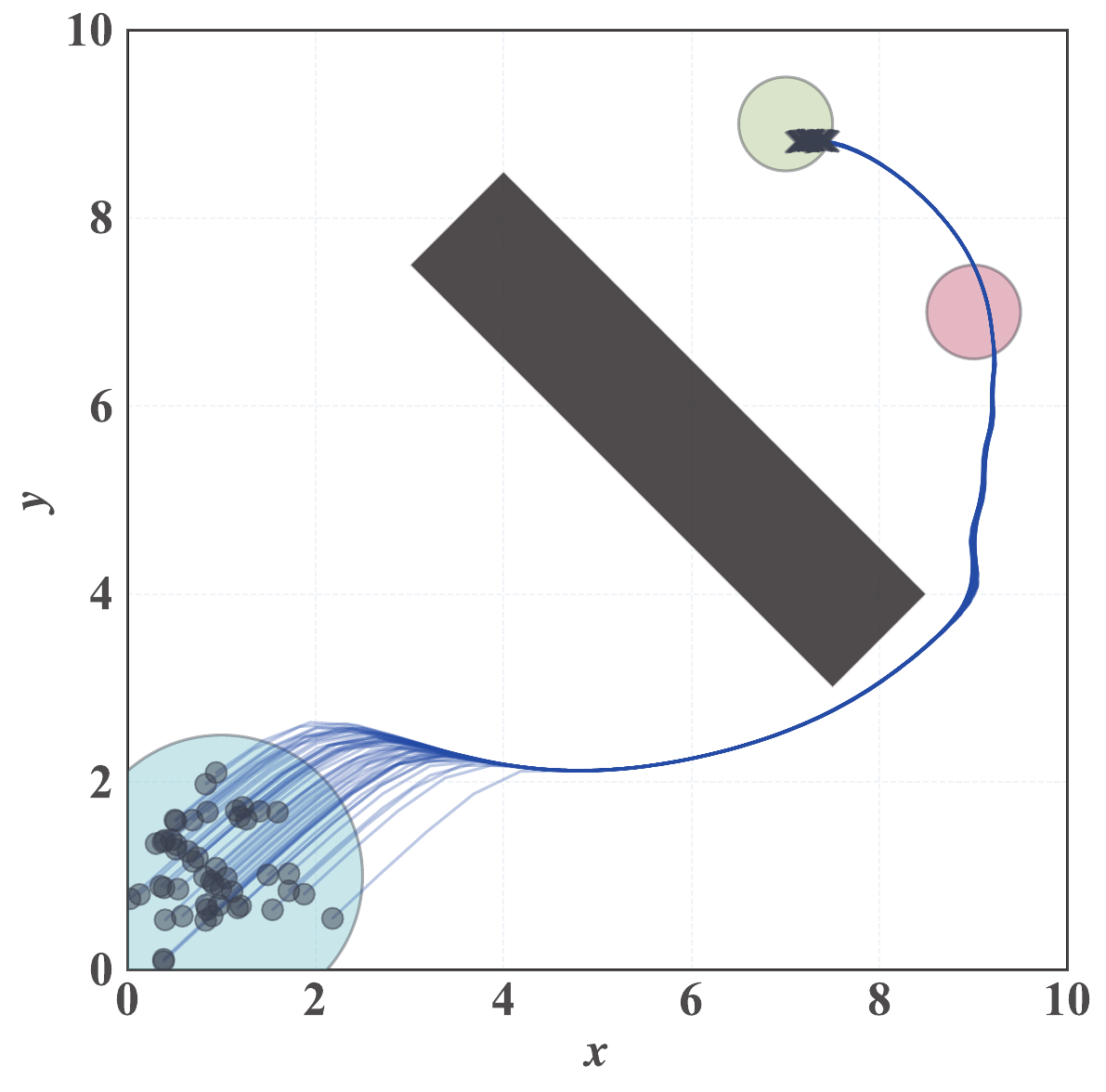}
        \end{minipage}
    }
    \caption{Training snapshots of \textbf{Nav2D-2G-rev} environment from 10k steps to 1M steps. Red goal region is the first target, and green goal region is the second target.}    
    \label{app:fig:training snapshots of nav2d-2g-rev}
\end{figure*}

\begin{figure*}[!t]
    \centering
    \subfigure[$r_0$]{
        \begin{minipage}{0.65\columnwidth}
            \centering
            \includegraphics[width=\linewidth]{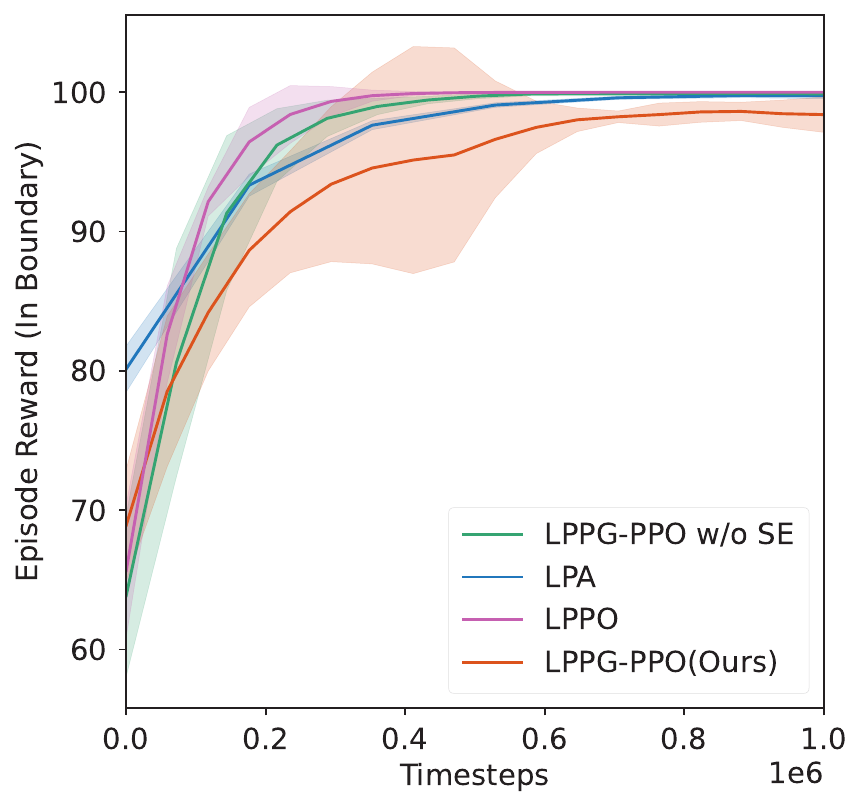}
        \end{minipage}
    }
    \subfigure[$r_1$]{
        \begin{minipage}{0.65\columnwidth}
            \centering
            \includegraphics[width=\linewidth]{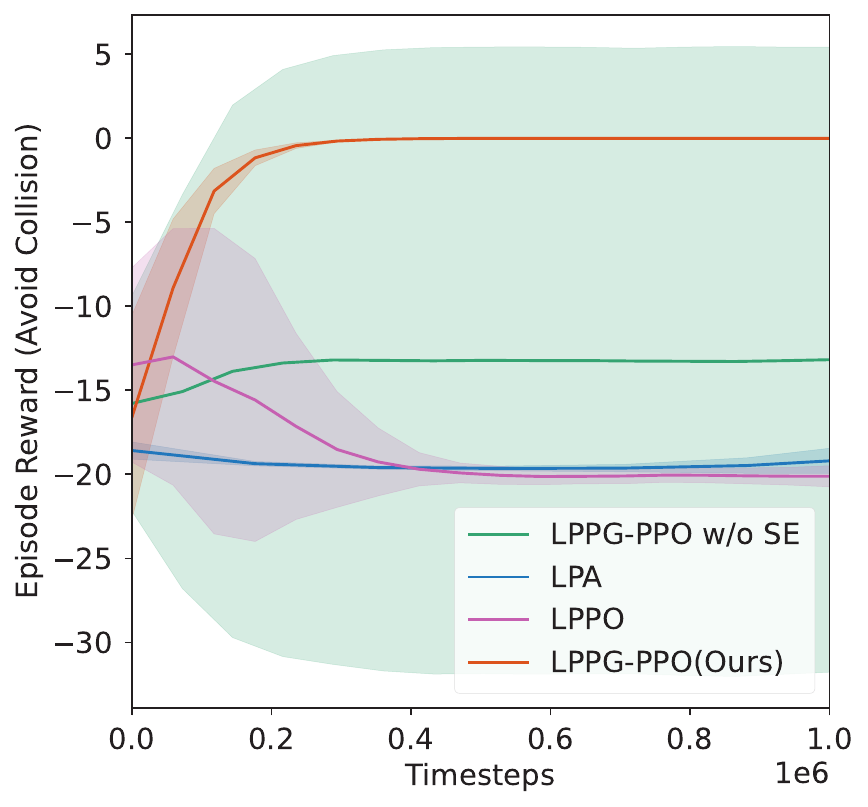}
        \end{minipage}
    }
    \subfigure[$r_2$]{
        \begin{minipage}{0.65\columnwidth}
            \centering
            \includegraphics[width=\linewidth]{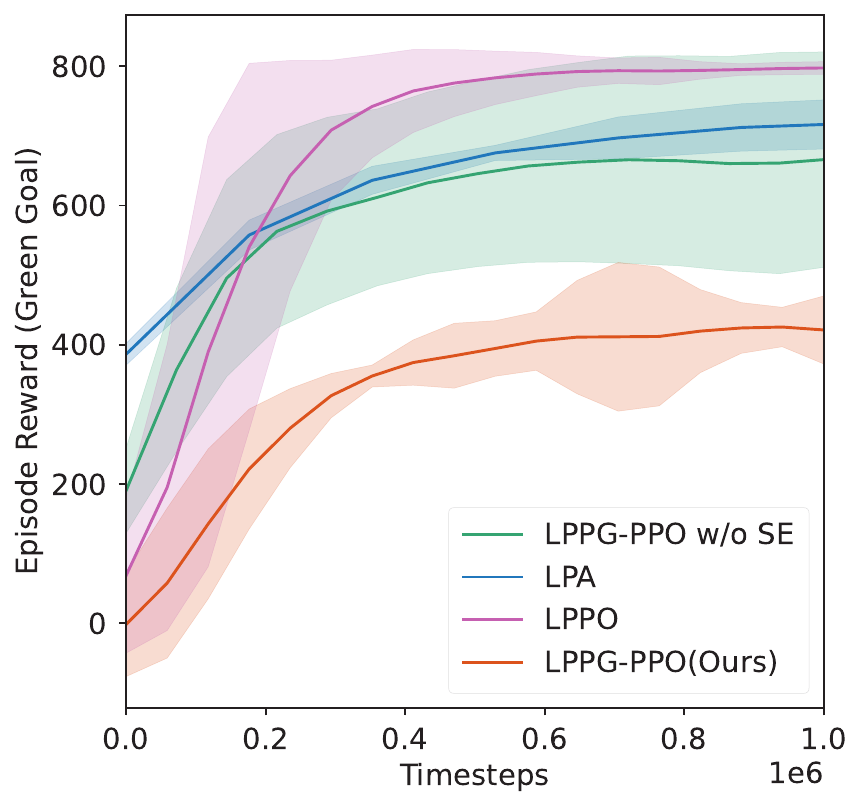}
        \end{minipage}
    }
    \caption{Training curves of different algorithms in Nav2D-1G} 
    \label{app:fig:reward nav2d-1g}
\end{figure*}

\begin{figure*}[!t]
    \centering
    \subfigure[$r_0$]{
        \begin{minipage}{0.48\columnwidth}
            \centering
            \includegraphics[width=\linewidth]{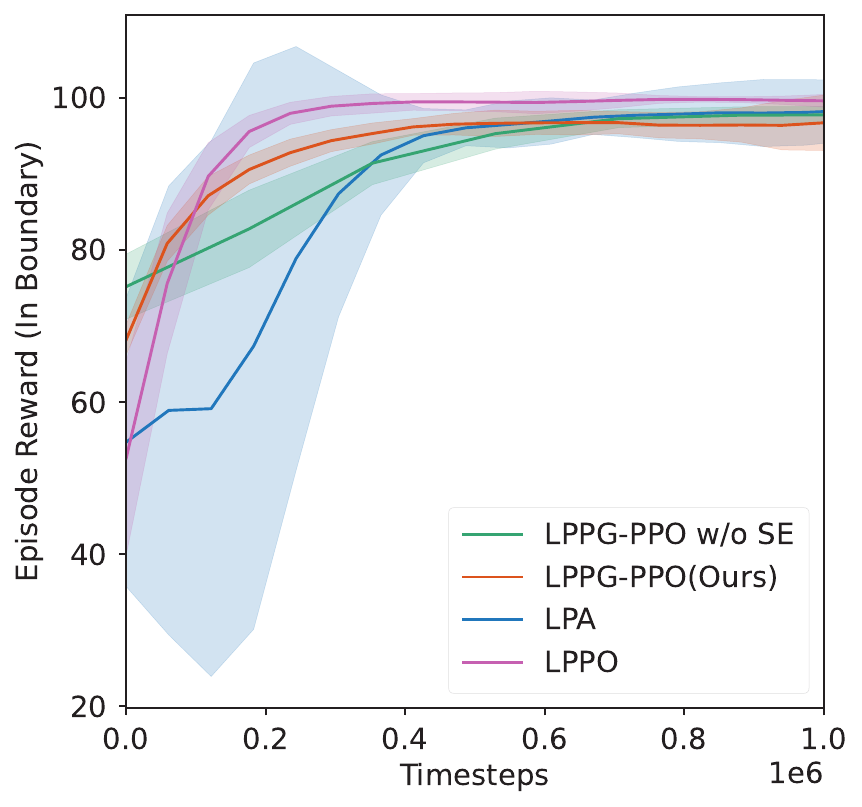}
        \end{minipage}
    }
    \subfigure[$r_1$]{
        \begin{minipage}{0.48\columnwidth}
            \centering
            \includegraphics[width=\linewidth]{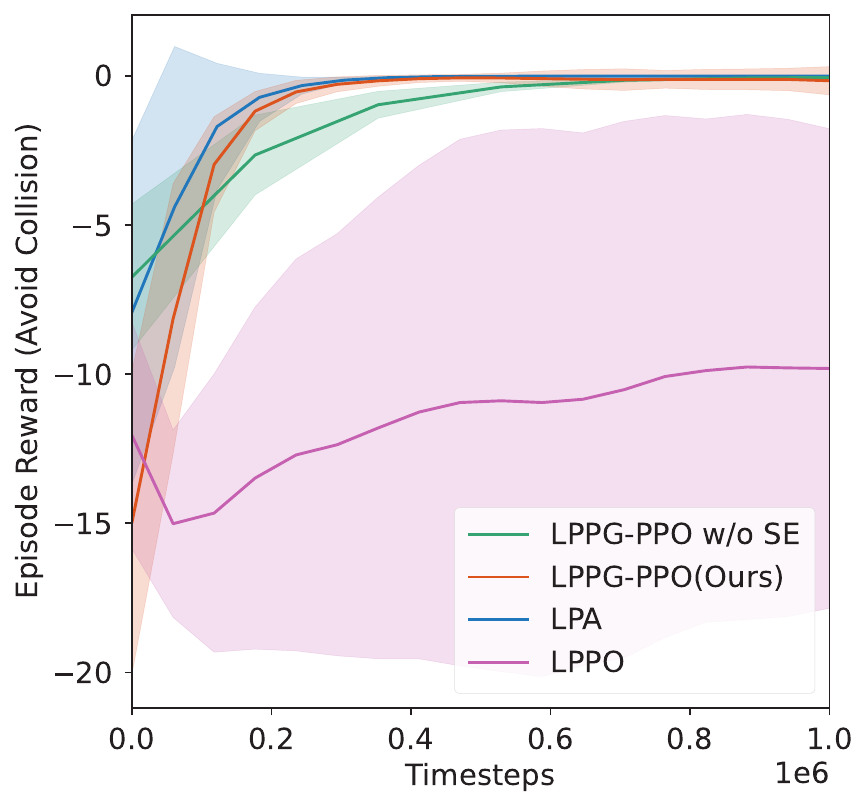}
        \end{minipage}
    }
    \subfigure[$r_2$]{
        \begin{minipage}{0.48\columnwidth}
            \centering
            \includegraphics[width=\linewidth]{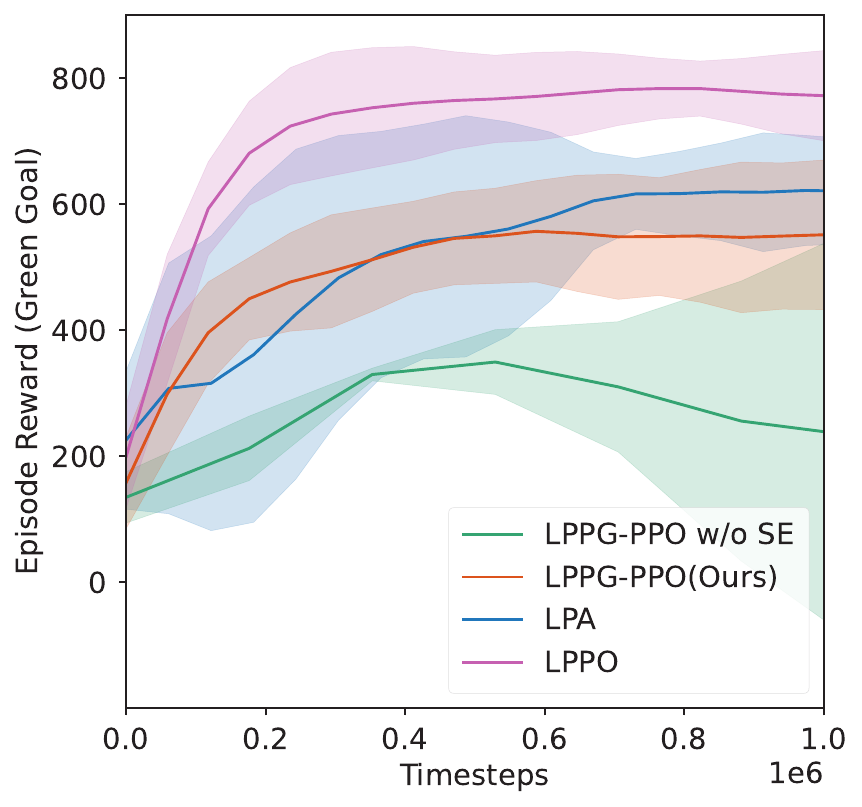}
        \end{minipage}
    }
    \subfigure[$r_3$]{
        \begin{minipage}{0.48\columnwidth}
            \centering
            \includegraphics[width=\linewidth]{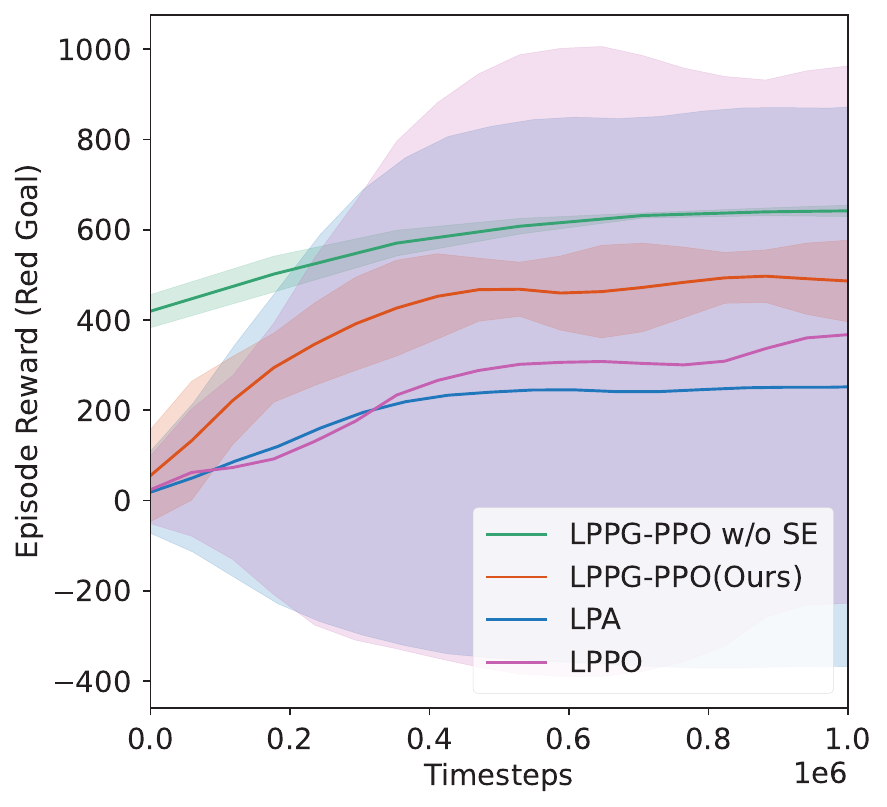}
        \end{minipage}
    }
    \caption{Training curves of different algorithms in Nav2D-2G} 
    \label{app:fig:reward nav2d-2g}
\end{figure*}

\begin{figure*}[!t]
    \centering
    \subfigure[$r_0$]{
        \begin{minipage}{0.48\columnwidth}
            \centering
            \includegraphics[width=\linewidth]{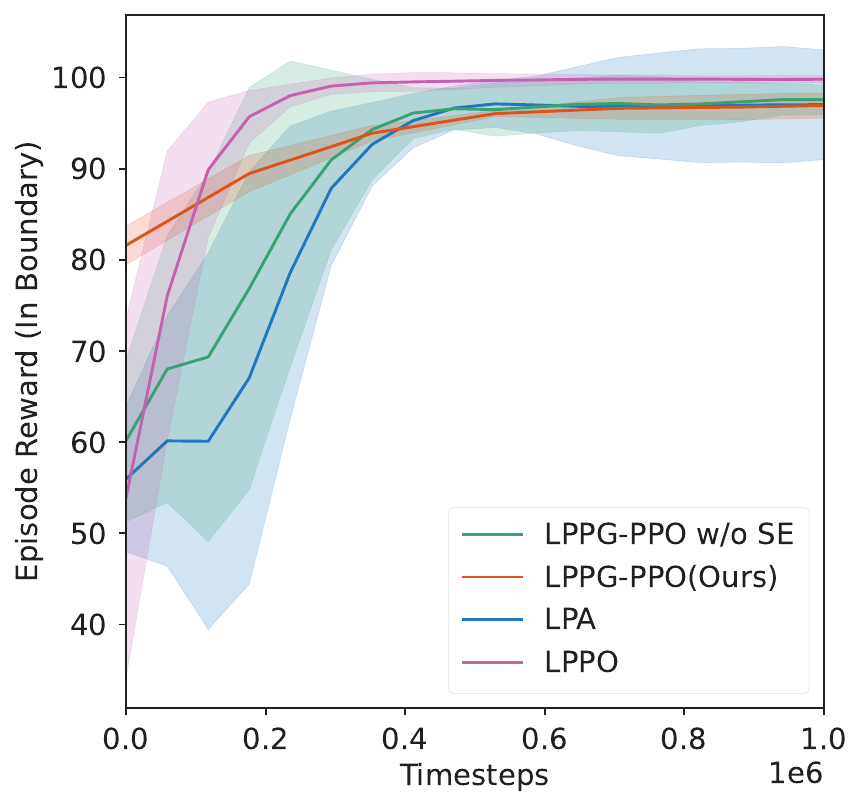}
        \end{minipage}
    }
    \subfigure[$r_1$]{
        \begin{minipage}{0.48\columnwidth}
            \centering
            \includegraphics[width=\linewidth]{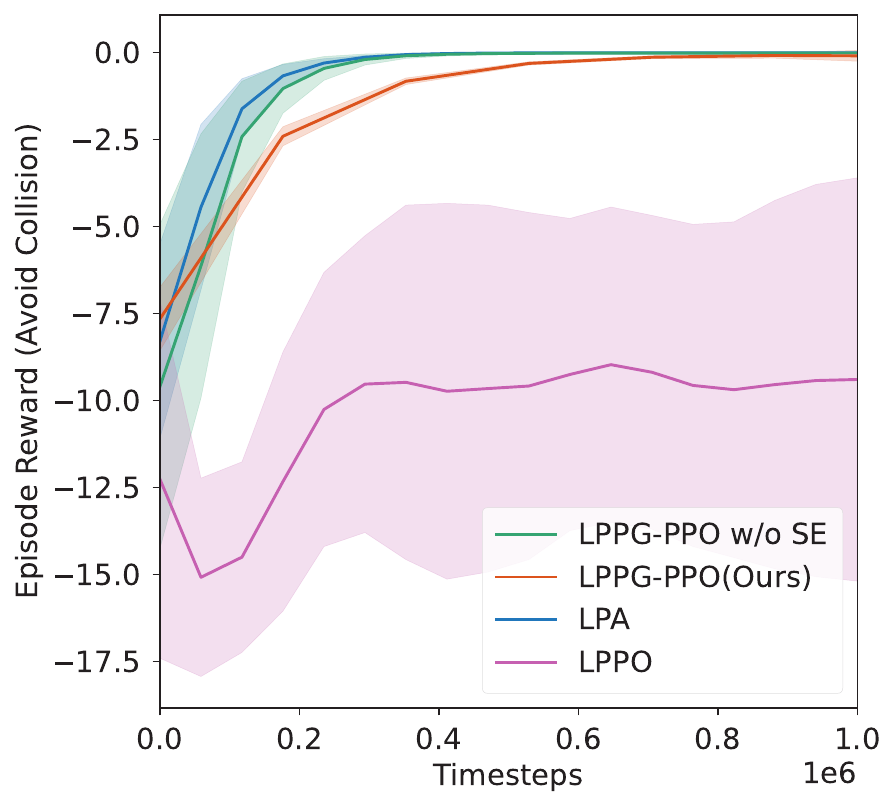}
        \end{minipage}
    }
    \subfigure[$r_2$]{
        \begin{minipage}{0.48\columnwidth}
            \centering
            \includegraphics[width=\linewidth]{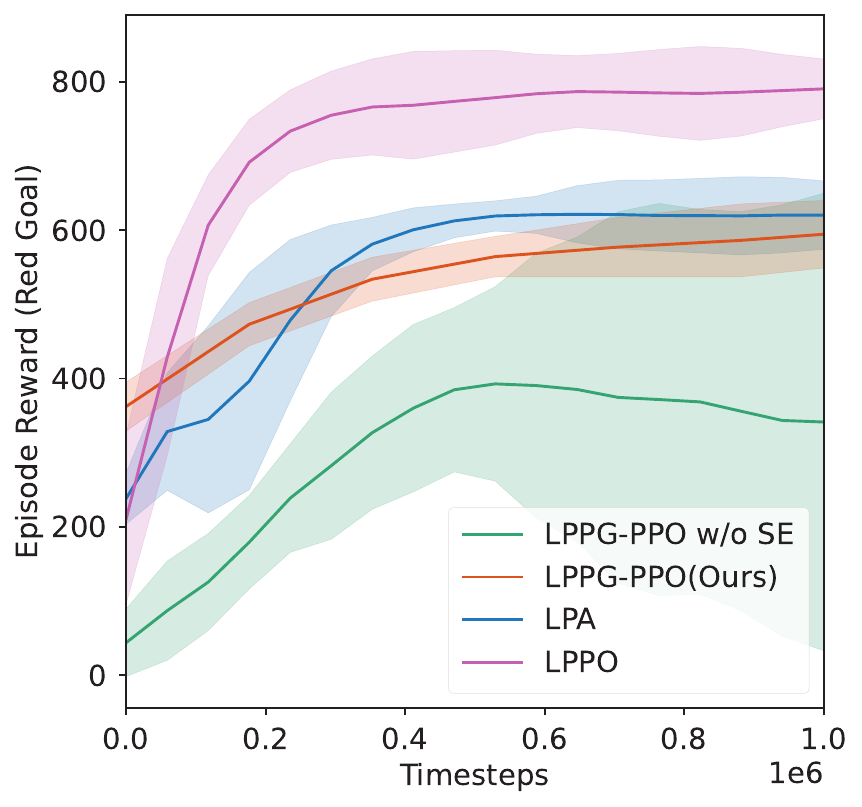}
        \end{minipage}
    }
    \subfigure[$r_3$]{
        \begin{minipage}{0.48\columnwidth}
            \centering
            \includegraphics[width=\linewidth]{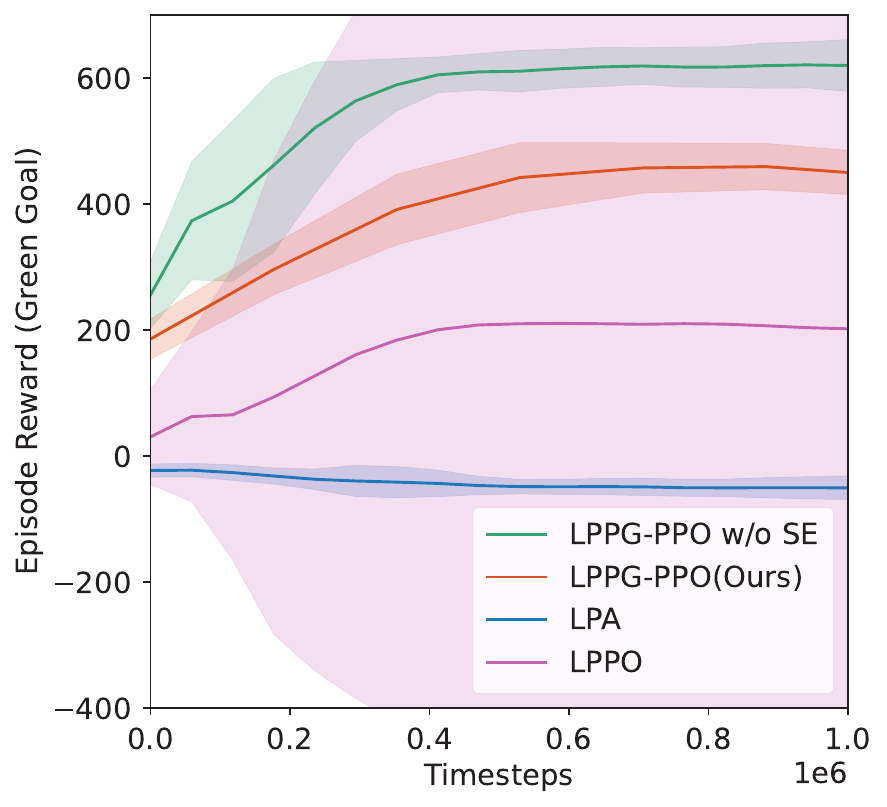}
        \end{minipage}
    }
    \caption{Training curves of different algorithms in Nav2D-2G-rev} 
    \label{app:fig:reward nav2d-2g-rev}
\end{figure*}

\newcolumntype{M}{r@{$\pm$}l}


\begin{table*}[!t]
\centering
\resizebox{\textwidth}{!}{
\begin{tabular}
{
l
*{3}{M c} 
*{4}{M c} 
*{4}{M c} 
}
\toprule
\multirow{2}{*}{Algorithm} &
\multicolumn{9}{c}{\textbf{Nav2D-1G}} &
\multicolumn{12}{c}{\textbf{Nav2D-2G}} &
\multicolumn{12}{c}{\textbf{Nav2D-2G-rev}} \\
\cmidrule(lr){2-10}\cmidrule(lr){11-22}\cmidrule(lr){23-34}
& \multicolumn{3}{c}{$K_1$} & \multicolumn{3}{c}{$K_2$} & \multicolumn{3}{c}{$K_3$}
& \multicolumn{3}{c}{$K_1$} & \multicolumn{3}{c}{$K_2$} & \multicolumn{3}{c}{$K_3$} & \multicolumn{3}{c}{$K_4$}
& \multicolumn{3}{c}{$K_1$} & \multicolumn{3}{c}{$K_2$} & \multicolumn{3}{c}{$K_3$} & \multicolumn{3}{c}{$K_4$} \\
\midrule
\textbf{LPA} &
100 & 0 & \yes & $-19$ & 1 & \no & 714 & 52 & \yes &
98 & 4 & \yes & 0 & 0 & \yes & 631 & 66 & \yes & 246 & 604 & \no &
97 & 10 & \yes & 0 & 0 & \yes & 603 & 48 & \yes & $-39$ & 11 & \no \\

\textbf{LPPO} &
100 & 0 & \yes & $-20$ & 1 & \no & 740 & 13 & \yes &
100 & 1 & \yes & $-10$ & 8 & \no & 753 & 62 & \yes & 382 & 612 & \no &
100 & 0 & \yes & $-10$ & 7 & \no & 785 & 30 & \yes & 197 & 231 & \no \\

\textbf{LPPG-PPO w/o SE} &
100 & 0 & \yes & $-12$ & 17 & \no & 704 & 164 & \yes &
98 & 2 & \yes & 0 & 0 & \yes & 199 & 391 & \no & 604 & 13 & \yes &
98 & 2 & \yes & 0 & 0 & \yes & 258 & 375 & \no & 612 & 22 & \yes \\

\textbf{LPPG-PPO (Ours)} &
98 & 2 & \yes & 0 & 0 & \yes & 427 & 58 & \yes &
96 & 4 & \yes & 0 & 1 & \yes & 587 & 93 & \yes & 459 & 71 & \yes &
97 & 2 & \yes & 0 & 0 & \yes & 591 & 42 & \yes & 426 & 38 & \yes \\
\bottomrule
\end{tabular}}
\caption{Completed comparison of baseline algorithms and ablation study of performances in three environments under 10 different seeds.}
\label{app:tab:full compare rewards}
\end{table*}

\begin{figure*}[!t]
    \centering
    \subfigure[Solver Runtime (log scale)]{
        \begin{minipage}{0.95\columnwidth}
            \centering
            \includegraphics[width=\linewidth]{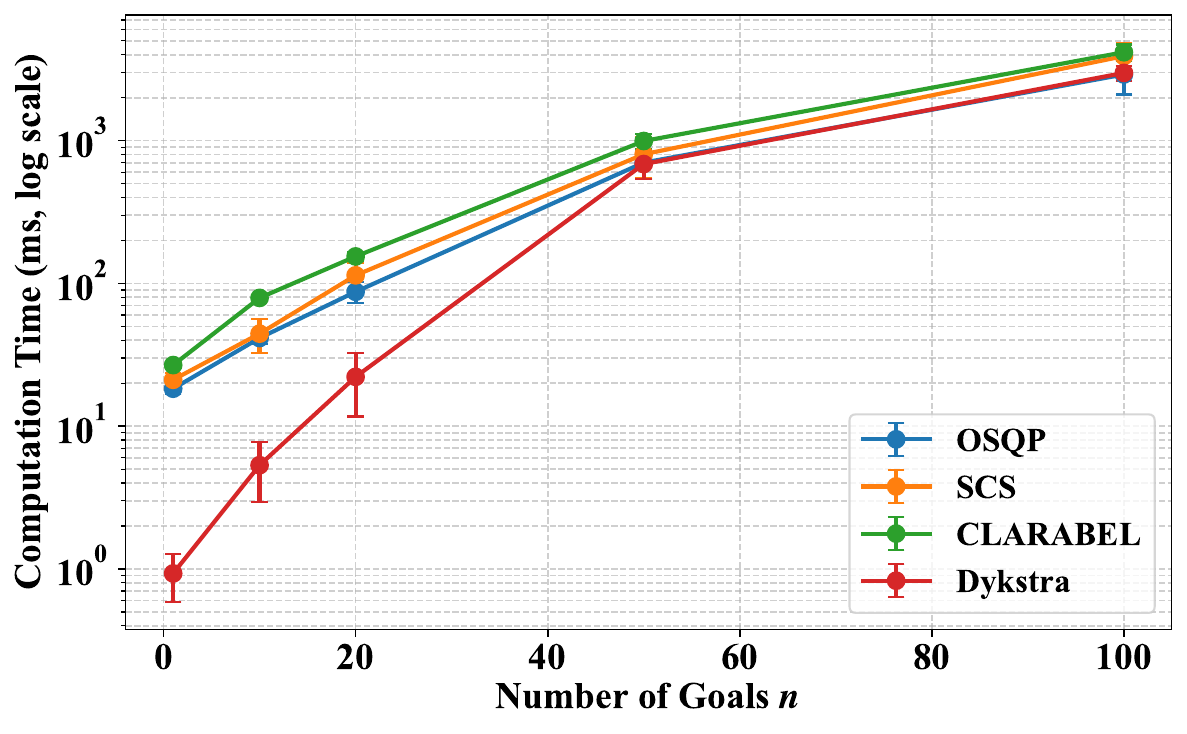}
        \end{minipage}
        \label{app:fig:solver runtime}
    }
    \subfigure[Relative Solver Speed-up (Higher = Faster)]{
        \begin{minipage}{0.95\columnwidth}
            \centering
            \includegraphics[width=\linewidth]{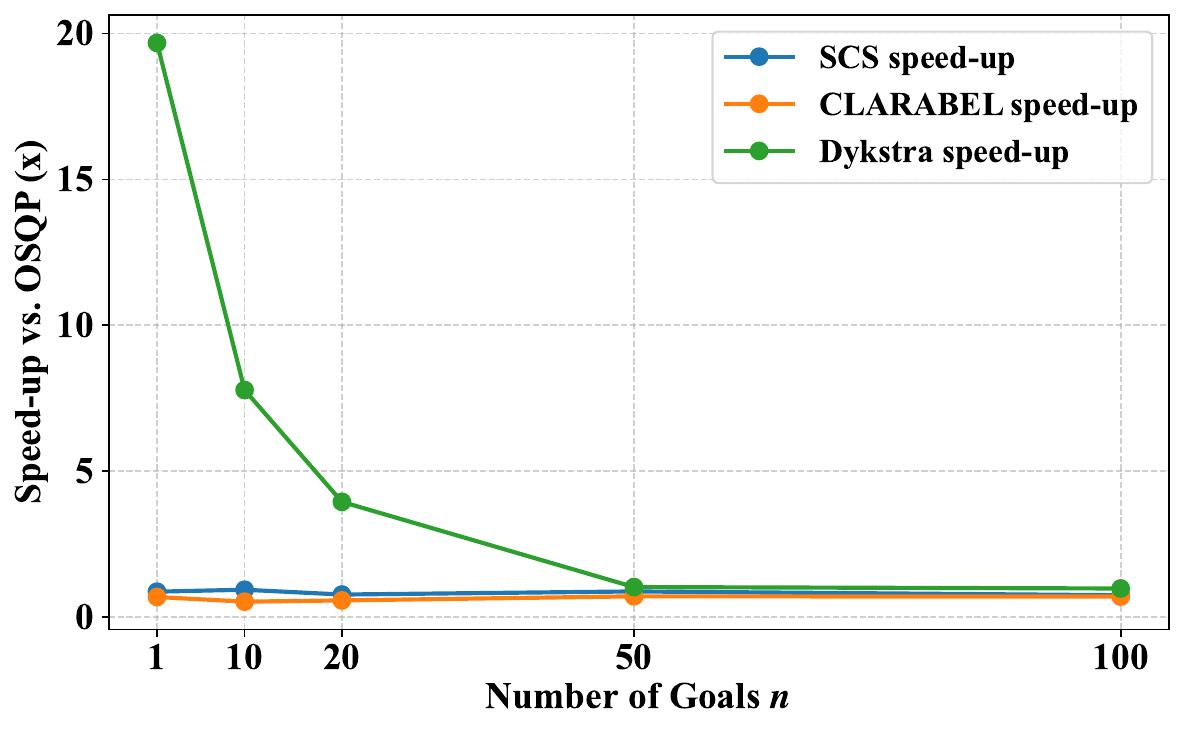}
        \end{minipage}
        \label{app:fig:solver speed-up}
    }
    \caption{Performance comparison between solvers for searching gradient}    
    \label{app:fig:solver comparison}
\end{figure*}

\subsection{C.2 Comparison Results}
To illustrate the performance comparison between different Dykstra's projection algorithm with other common solvers (OSQP, SCS, CLARABEL), we compare the solver runtime and relative solver speed-up, as shown in  Figure~\ref{app:fig:solver comparison}. To make the benchmark more comparative, SE technique is not used here, and all settings are trained with 100k steps, 5 different seeds, starting from the lowest priority subtask $K_M$ to search for a feasible policy gradient.  Specifically, Figure~\ref{app:fig:solver runtime} shows the average runtime of different solvers with different numbers of goals in a log scale. Figure~\ref{app:fig:solver speed-up} shows the relative solver speed-up using OSQP as the baseline. In all settings, these generic solvers shows a relatively stable speed with the baseline. In comparison, we observe that, as an iterative projection algorithm specifically designed for solving such types of optimization problems, Dykstra's projection algorithm has better efficiency when $n$ is small, and faster up to 20x than the baseline OSQP($\textbf{n}=1$, $|\mathcal{K}|=3$). Then, with $n$ being larger, Dykstra's algorithm needs more iterations to converge to the solution for the same precision, and the advantage over other solvers is not so obvious.

\section{D. Additional Experiment Results: MuJuCo Humanoid}
In this section, we evaluate LPPG-PPO using the MuJoCo Humanoid environment, a standard benchmark for high-dimensional continuous control tasks. We first present details about the environment configuration, followed by a direct training and performance comparison against vanilla to demonstrate the effectiveness of our approach. 

\subsection{D.1 Environment Details}
This section provides specific details for the MuJoCo Humanoid environment, including the state space, action space, reward composition, and how to specifically divide them into different subtasks. 
\begin{itemize}
    \item \textbf{State} The state consists of 348 elements, including 22 body parts positions (qpos), 23 body part velocities (qvel), 130 center-of-mass based body inertia values (cinert), 78 center-of-mass based velocity values, and 78 center-of-mass based external force values. 
    \item \textbf{Action} The action consists of 17 elements, representing the torques applied at the hinge joints. The raw policy action $a_t$ is $(-1, 1)$, and is clipped to the true torque bound $[-0.4,0.4]$. 
    \item \textbf{Reward: Healthy} This reward is designed to keep survival. The agent receives a constant positive reward $w_{healthy}$ (default 5) for each timestep it remains alive (i.e., does not enter a terminal state).
    \begin{equation} \notag
            r^{healthy}_t = 
                \begin{cases}
                    w_{healthy} \quad & \text{if} \quad  \text{alive} \\ 
                    0 & \text{else}
                \end{cases}
        \end{equation}
    \item \textbf{Reward: Forward} This reward is designed for forward locomotion. It is proportional to the humanoid velocity in the positive $x$ direction, based on the change in the center-of-mass position per timestep. 
    \begin{equation} \notag
        r^{forward}_t = w_{forward} \frac{x'-x}{dt}
    \end{equation}
    where $w_{forward}$ is the weight for this reward (default is 1.25), $x'$ and $x$ is the center position of mass after action and before action, respectively. 
    \item \textbf{Cost: Control Cost} This cost is to penalize the humanoid for high-magnitude actions to encourage energy efficiency. In practice, it is implemented as a negative reward (a cost), proportional to the squared L2-norm of the action $a_t$. 
    \begin{equation} \notag
        r^{control}_t = -c_t^{control} = w_{control} \|a_t\|^2
    \end{equation}
    where $w_{control}$ is the weight for the cost (default is 0.1).
\end{itemize}

The total reward $r_t$ at each timestep is the summation of these components,
\begin{equation}
    r_t = r^{healthy}_t + r^{forward}_t + r^{control}_t
\end{equation}

To apply LPPG-PPO, we decompose this reward and define three separate prioritized subtasks, ordered from high to low priority: (1) $K_1$ is to maximize the healthy reward $r_t^{healthy}$. (2) $K_2$ is to maximize forward reward $r_t^{forward}$. (3) $K_3$ is to maximizing the control reward $r_t^{control}$, which is equivalent to minimizing the control cost. 

A key distinction of the MuJoCo Humanoid from Nav2D environment is that, Humanoid task lacks explicit hard constraints, and all subtasks are formulated as objectives to be maximized. However, a natural priority exists among these objectives. Our defined priority order actually reflects a logical learning curriculum. The agent must first learn to survive (stay healthy or upright), then learn to achieve forward locomotion, and finally learn to optimize for energy efficiency.

\subsection{D.2 Training Details}
The hyperparameter settings used for our experiments is presented in Table~\ref{app:tab:hyperparameters of lppg-ppo humanoid}. For the Humanoid benchmark, the agent is trained for total $7\times 10^6$ steps, which ensures convergence and allows for a fair comparison of performance. Similarly, we set $\epsilon_i=0$ for all subtasks. 

\begin{table}[t]
    \centering
    \begin{tabular}{lc}
        \toprule
         & \textbf{Humanoid} \\
        \midrule
        \textbf{Total number of steps} & $7\times10^6$ \\ 
        \textbf{Learning rate for actor} & $1\times10^{-4}$\\
        \textbf{Learning rate for critic} & $1\times10^{-4}$ \\
        \textbf{Discount factor} & $0.99$ \\ 
        \textbf{GAE discount factor} & $0.95$ \\ 
        \textbf{Batch size} & $2048$\\ 
        \textbf{Mini batch size} & $64$\\ 
        \textbf{Update epoch} & $10$ \\
        \textbf{Actor hidden layer numbers} & $3$ \\
        \textbf{Actor hidden neuron numbers} & $64$ \\
        \textbf{Critic hidden layer numbers} & $3$ \\
        \textbf{Critic hidden neuron numbers} & $64$ \\
        \textbf{Lexicographic relaxation value }  \bm{$\epsilon_i$} & $[0,0,0]$ \\
        \textbf{Dykstra's convergence tolerance} & $1\times10^{-6}$ \\ 
        \textbf{Dykstra's maximum iteration} & 500 \\
        \bottomrule
    \end{tabular}
    \caption{Hyperparameter settings of LPPG-PPO in the Humanoid environment.}
    \label{app:tab:hyperparameters of lppg-ppo humanoid}
\end{table}

\subsection{D.3 Comparison Results}
In this section, we present a comparative analysis of LPPG-PPO against the vanilla PPO baseline. The vanilla PPO uses the recommended default weights for different subtask scalarization. LPPG-PPO dispenses with manual weight tuning and instead enforces the explicit priority defined above. 

The training curves for subtasks are presented in Figure~\ref{app:fig:humanoid healthy}-\ref{app:fig:humanoid control cost}, and the training curves for total reward are presented in Figure~\ref{app:fig:humanoid total reward}. The final converged values are summarized in Table~\ref{app:tab:full compare rewards humanoid}. First, For the healthy subtask $K_1$, we observe that the vanilla PPO achieves a marginally higher reward than LPPG-PPO. The main reason is that, LPPG-PPO uses the SE technique with uniform sampling, which makes the training process only partially focus on being healthy. Second, LPPG-PPO demonstrates a significant and clear advantage in the forward reward. However, the vanilla PPO attempts to learn all subtasks concurrently, and fails to find the balance between these subtasks. This result also indicates that the default PPO weights are suboptimal and that we should assign $K_2$ with a higher weight. Third, the training curve of $K_3$ also demonstrates the effectiveness of priorities. The control cost of LPPG-PPO begins to meaningfully decrease after around $4\times 10^6$ timesteps. This aligns with the priority, because it occurs only after the agent has achieved enough high performance on more critical tasks $K_2$ and $K_3$. In contrast, the vanilla PPO obtains a much larger control cost throughout training, which indicates an inefficient policy characterized by larger torques. Finally, despite LPPG-PPO does not optimize for a simple weighted sum, it ultimately achieves a superior total reward compared to the fine-tuned vanilla PPO. This demonstrates that our LPPG-RL framework does not simply trade one objective for another. Instead, it guides the agent to a more robust, balanced, and effective policy that finds a better solution according to the priority to solve a complex and high-dimensional task. 

\begin{figure*}[!t]
    \centering
    \subfigure[$r_0$]{ \label{app:fig:humanoid healthy}
        \begin{minipage}{0.48\columnwidth}
            \centering
            \includegraphics[width=\linewidth]{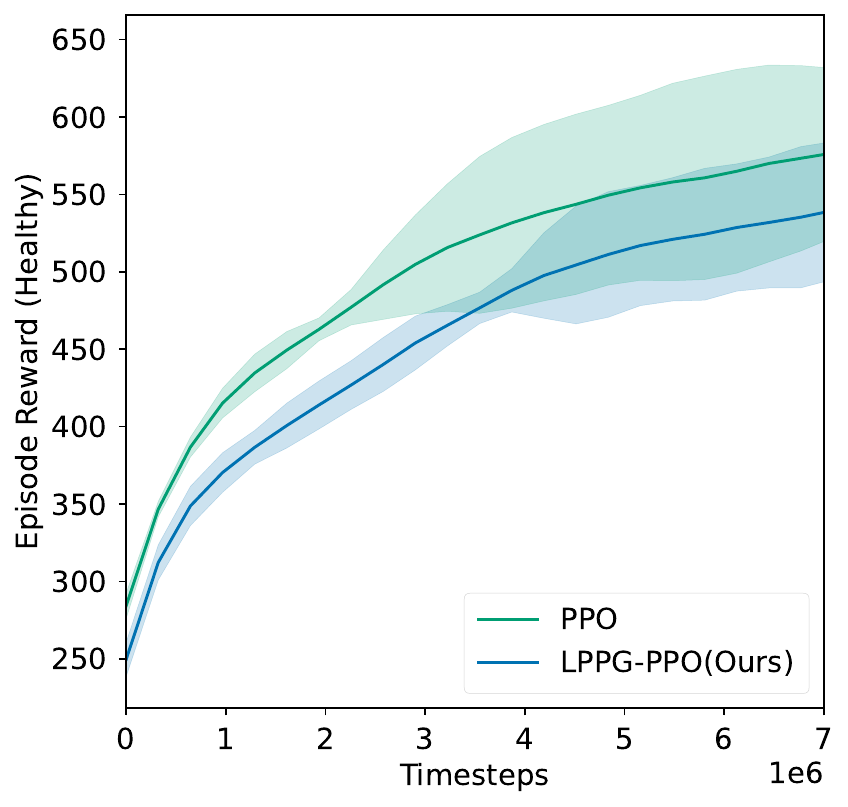}
        \end{minipage}
    }
    \subfigure[$r_1$]{ \label{app:fig:humanoid forward}
        \begin{minipage}{0.48\columnwidth}
            \centering
            \includegraphics[width=\linewidth]{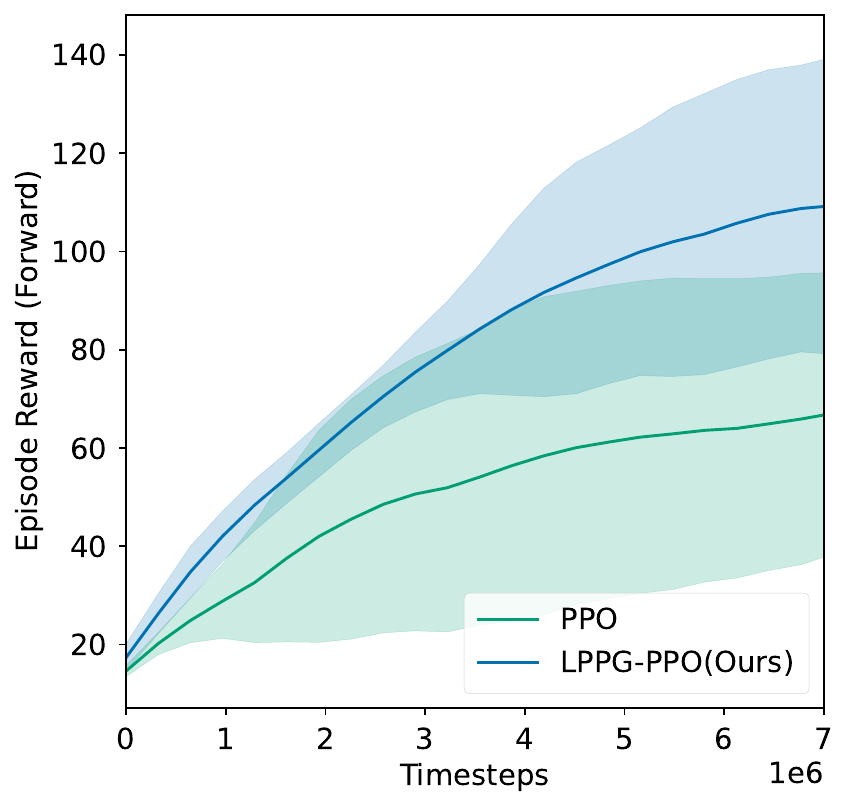}
        \end{minipage}
    }
    \subfigure[$r_2$]{ \label{app:fig:humanoid control cost}
        \begin{minipage}{0.48\columnwidth}
            \centering
            \includegraphics[width=\linewidth]{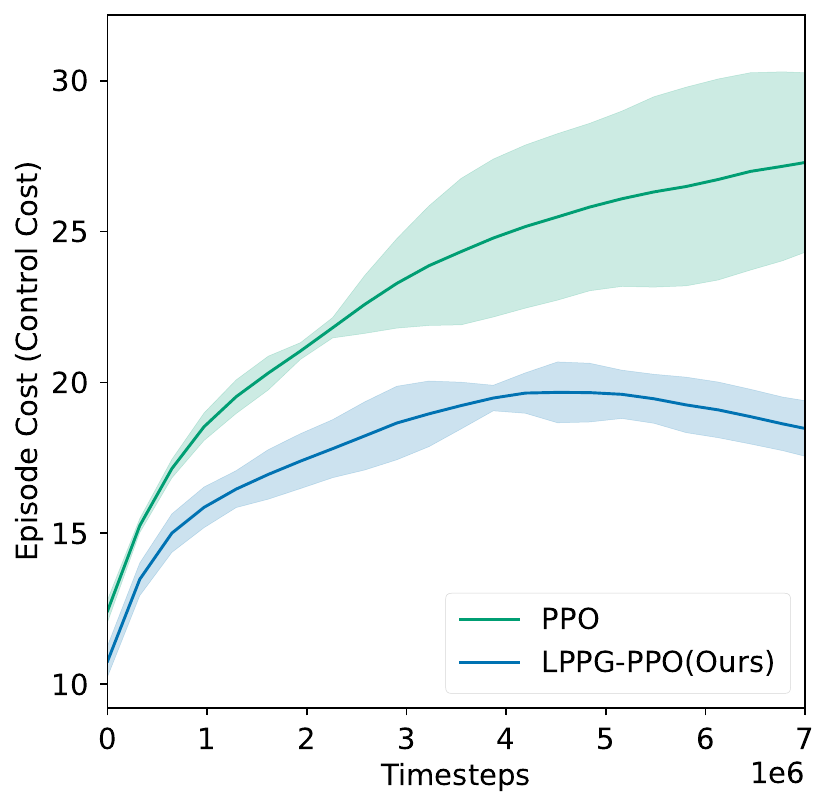}
        \end{minipage}
    }
    \subfigure[total reward]{ \label{app:fig:humanoid total reward}
        \begin{minipage}{0.48\columnwidth}
            \centering
            \includegraphics[width=\linewidth]{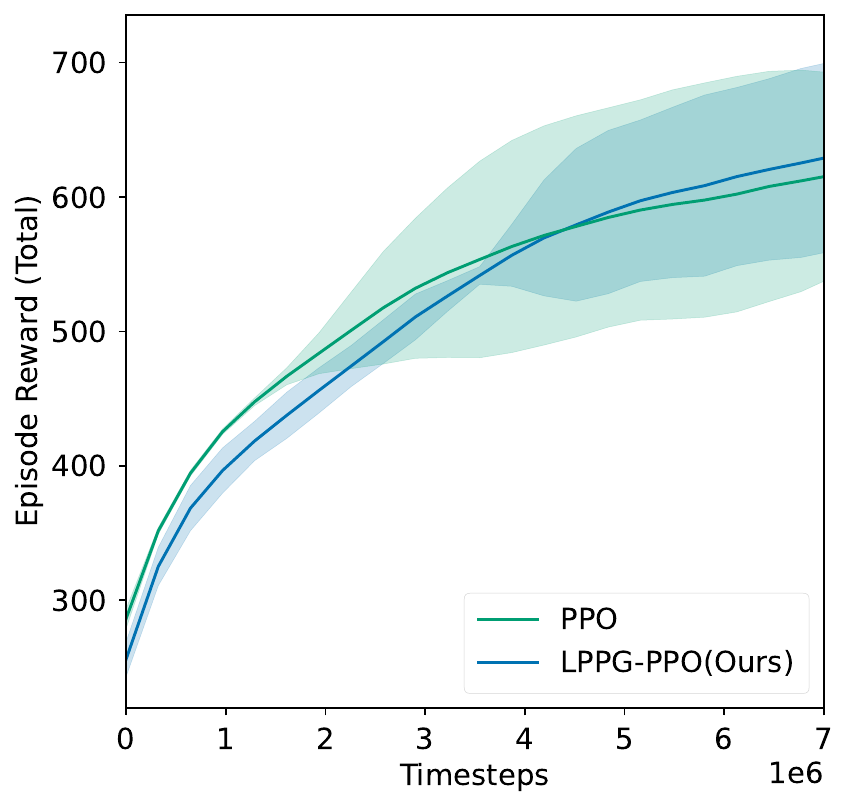}
        \end{minipage}
    }
    \caption{Training curves of LPPG-PPO and PPO in Humanoid} 
    \label{app:fig:reward humanoid}
\end{figure*}

\begin{table*}[!t]
\centering
\begin{tabular}
{
l
*{4}{M c} 
}
\toprule
\multirow{2}{*}{Algorithm} &
\multicolumn{12}{c}{\textbf{Humanoid}} \\
\cmidrule(lr){2-13}
& \multicolumn{3}{c}{$K_1$ (Healthy Reward) $\uparrow$} & \multicolumn{3}{c}{$K_2$ (Forward Reward) $\uparrow$} & \multicolumn{3}{c}{$K_3$ (Control Cost) $\downarrow$} & \multicolumn{3}{c}{Total Reward $\uparrow$} \\
\midrule
\textbf{PPO} &
\textbf{571} & \textbf{53} & \yes & 62 & 22 & \yes & 27 & 3 & \yes & 620 & 68 & \yes \\

\textbf{LPPG-PPO (Ours)} &
543 & 30 & \yes & \textbf{108} & \textbf{32} & \yes & \textbf{16} & \textbf{1} & \yes & \textbf{624} & \textbf{54} & \yes \\
\bottomrule
\end{tabular}
\caption{Completed comparison of PPO and LPPG-PPO in Humanoid environment under 5 different seeds.}
\label{app:tab:full compare rewards humanoid}
\end{table*}

\section{E. Proof of Theorem~\ref{thm:lexico convergence}}
In this section, we provide a complete proof for Theorem~\ref{thm:lexico convergence}. First, we present a general convergence Lemma with two-timescale methods and stochastic approximation methods for actor-critic architecture. This is based on some past works \cite{konda1999, borkar2008}. Here we will simply restate the Lemma and give the short proof.
Second, we give the upper bound of stepsize to ensure the lexicographic monotonic improvement for the policy update  Then, we give the full proof of the convergence of our method under general actor-critic architecture. 

\subsection{E.1 Preliminaries}
\begin{lemma}[Actor-Critic convergence] \label{app:lem:ac convergence}
    Let the learning rate of actor $\alpha_\theta$ and critic $\alpha_\phi$ satisfy the following conditions, 
    \begin{equation} \notag
        \begin{aligned}
            & \sum_t a_{\theta,t} = \sum_t a_{\phi,t} = \infty \\ 
            & \sum_t (a_{\theta,t}^2 + a_{\phi,t}^2) < \infty \\ 
            & \lim_{t\rightarrow\infty} \frac{\alpha_{\theta,t}}{\alpha_{\phi,t}} = 0
        \end{aligned}
    \end{equation}
Then, the actor and critic network parameter converges to a local optimal point $(\theta^*, \phi^*(\theta^*))$, with the criterion $J$ satisfies $\|\nabla J(\theta^*)\|=0$.
\end{lemma}
\begin{proof}
    In an Actor-Critic architecture, suppose the learning rate of actor $\alpha_\theta$ and $\alpha_\phi$ satisfy all conditions above, the actor and critic is updated with two different timescales, and critic is updated faster than actor. Under this assumption, actor parameters $\theta$ can be viewed as fixed when critic is updated. According to Theorem~2 of \cite{borkar2008}, standard stochastic approximation algorithms can be used for the proof that, for each $\theta_k$, critic reach a local optimal point $\phi^*(\theta_k)$, and finally, $\lim_{k\rightarrow\infty} \theta_k = \theta^*$, where the criterion $J$ satisfies $\|\nabla J(\theta^*)\| = 0$. Then, we complete the proof. 
\end{proof}

\begin{lemma}[Lexicographic feasibility] \label{app:lem:lex stepsize feasibility}
    Let $\delta_i = g_i^Td, \forall i\in \{1,\cdots,N\}$, where $g_i$ is the gradient of the $i$-th subtask, $d$ is the final update direction calculated by LPPG, and $N$ is the uniformly chosen subtask index. Suppose the sub-objective function $J_i(\theta)$ is differentiable and L-smooth, and $\alpha_\theta$ is the update step of policy network. Then, 
    \begin{equation} \notag
        \alpha_\theta \leq \min \left\{\frac{2\delta_i}{L_i \|d\|^2}, \quad i=\{1,\cdots, N\} \right\}
    \end{equation}
\end{lemma}
\begin{proof}
    Suppose the sub-objective function of $i$-th subtask w.r.t policy parameter $J_i(\theta)$ is differentiable and L-smooth. From the second-order Taylor expansion along direction $d$, we have,  
    \begin{equation} \notag
        J_i(\theta+\alpha_\theta d) = J_i(\theta) + \alpha_\theta g_i^Td + \frac{\alpha_\theta^2}{2}d^TH_id
    \end{equation}
    From the property of L-smooth, we have $\|H_i\| \leq L_i$, with $L_i > 0$. From the result of our LPPG method, let $\delta_i = g_i^T d \geq 0, \forall i\in \{1,\cdots,N\}$. Then, 
    \begin{equation} \notag
        \begin{aligned}
            J_i(\theta+\alpha_\theta d) - J_i(\theta) & = \alpha_\theta \delta_i + \frac{\alpha_\theta^2}{2}d^TH_id \\ 
            & \geq \alpha_\theta \delta_i - \frac{L_i}{2}\alpha_\theta^2 \|d\|^2
        \end{aligned}
    \end{equation}
    To make $J_i(\theta)$ is nonincreasing for all $i$, we should make the right side positive holds for all $i$. Then, 
    \begin{equation} \notag 
        \begin{aligned}
            & && \alpha_\theta \delta_i - \frac{L_i}{2} \alpha_\theta^2 \|d\|^2 \geq 0  \\ 
            & \Rightarrow && \alpha_\theta \leq \frac{2\delta_i}{L_i\|d\|^2} \\ 
            & \Rightarrow && \alpha_\theta \leq \underbrace{\min\left\{\frac{2\delta_i}{L_i\|d\|^2},i \in \{1,\cdots,N\} \right\}}_{\alpha_{\theta,\max}}
        \end{aligned}
    \end{equation}
    Therefore, if the learning rate of policy is within the given range, the lexicographic monotonic improvement holds for $i = \{1,\cdots,N\}$, which ensures all higher priorities at lease will not get worse. Then we complete the proof. 
\end{proof}

\subsection{E.2 Proof of Theorem~\ref{thm:lexico convergence}}
Now, we are ready to give the full proof of Theorem~\ref{thm:lexico convergence}.

\begin{proof}
    Suppose at each update step, we extract a subproblem with $N\sim \text{Uniform}(\{1,\cdots,M\})$. From Lemma~\ref{app:lem:lex stepsize feasibility}, we have that, for each single step, the lexicographic property holds for all optimized priorities, if the stepsize of actor is chosen as, 
    \begin{equation} \notag
        \alpha_\theta \leq \min\left\{\frac{2\delta_i}{L_i\|d\|^2},i \in \{1,\cdots,N\} \right\}
    \end{equation}
    where $\delta_i = g_i^Td$, and the lexicographic monotonic improvement per step is guaranteed. 
    
    Then, according to our LPPG algorithm, $d$ is selected from the intersection of all subtask gradient half-spaces $g_i^Td \geq 0$. When $N=M$, $d$ will only be zero if and only if $\|g_i\|=0$ or $g_i$ is conflicted with other higher priorities. This means, we can always find a $\|d\| > 0$, and a small enough stepsize $\alpha_\theta$ to update our policy before convergence. Therefore, with Lemma~\ref{app:lem:ac convergence}, we have our actor and critic parameters finally converge to the local optimal point $(\theta^*, \phi^*(\theta^*))$, and we complete our proof.
\end{proof}
\begin{remark}
    The quality of the reached fixed point is ultimately bounded by the expressive power of the policy class.
    
    If the ground-truth optimal policy $\pi^\star$ is contained in the parameterized family $\Pi:=\{\pi_\theta \mid \theta \in \Theta \}$, the projected updates can, in principle, converge to that global lexicographic optimum. However, when $\pi^\star\notin \Pi$, the algorithm can do no better than the best policy admissible under the chosen architecture, yielding a class-induced local optimum.
    
    In addition, stochastic gradient noise and finite-sample estimation error further bias the search, so in practice the procedure typically settles at a local stationary point even when $\pi^\star$ is representable.
\end{remark}

\section{F. Proof of Theorem~\ref{thm:update lower bound}}
In this section, we will give a analyze the lower bound of our policy in each update. We apply the Lemma by \cite{kakade2002} to give the general lower bounds. 

\subsection{F.1 Preliminaries}
\begin{lemma}[Policy update bound] \label{app:lem:policy update bound}
    Given two policies $\pi'$ and $\pi$ with parameters $\theta'$ and $\theta$, we have the performance different between these two policies, 
    \begin{equation} \notag
    \begin{aligned}
        J(\pi') - J(\pi) \geq & \frac{1}{1-\gamma} \underset{\begin{subarray}{c} s \sim D^{\pi} \\ a \sim \pi' \end{subarray}}{\mathbb{E}} \left[ A^{\pi}(s,a) \right] \\ 
        & - \frac{2\gamma C^{\pi',\pi}}{(1-\gamma)^2} \underset{s \sim D^{\pi}}{\mathbb{E}} \left[\operatorname{TV}(\pi', \pi)(s) \right]
    \end{aligned}
    \end{equation}
    where $D^\pi$ is the discounted future state distribution, defined by, $D^\pi = (1-\gamma) \sum_{t=0}^{\infty} \gamma^t P(s_t=s|\pi)$. $C^{\pi',\pi} = \max_{s\in\mathcal{S}} \left| \mathbb{E}_{a\sim \pi'}\left[A^{\pi}(s,a) \right]\right|$, and $\operatorname{TV}(\pi', \pi)$ is the total variance distance between two distribution $\pi'$ and $\pi$. 
\end{lemma}

\subsection{F.2 Proof of Theorem~\ref{thm:update lower bound}}
Now, we are ready to give the proof of Theorem~\ref{thm:update lower bound}.

\begin{proof}
    For each objective $J_i$ with advantage $A_i^{\pi}$, we set the constant,
    \begin{equation} \notag 
        C_i^{\pi',\pi} := \max_s \left| \underset{a\ \sim \pi}{\mathbb{E}} \left[ A^{\pi}(s,a) \right] \right|
    \end{equation}
    For the first item, with importance sampling, we have,
    \begin{equation} \notag
        \begin{aligned}
            & \underset{\begin{subarray}{c} s \sim D^{\pi} \\ a \sim \pi' \end{subarray}}{\mathbb{E}} \left[ A_i^{\pi}(s,a) \right] = \underset{\begin{subarray}{c} s \sim D^{\pi} \\ a \sim \pi \end{subarray}}{\mathbb{E}} \left[\frac{\pi'(a|s)}{\pi(a|s)} A_i^{\pi}(s,a) \right] \\ 
            & = \underset{\begin{subarray}{c} s \sim D^{\pi} \\ a \sim \pi \end{subarray}}{\mathbb{E}} \left[ \left(1 + \nabla_\theta \log{\pi_\theta}(a|s)\Delta \theta+ \mathcal{O}(\|\Delta \theta\|^2) \right) A_i^{\pi}(s,a) \right] \\ 
            & = \underset{\begin{subarray}{c} s \sim D^{\pi} \\ a \sim \pi \end{subarray}}{\mathbb{E}} \left[\underbrace{\nabla_\theta\log{\pi_\theta(a|s)}A_i^{\pi}(s,a)}_{g_i} \Delta\theta \right] + \mathcal{O}(\|\Delta \theta\|^2)
        \end{aligned}
    \end{equation}
    Replace $\Delta\theta = \alpha_\theta d$. When $\alpha_\theta$ is small enough, we omit the second-order item. Then we have,
    \begin{equation} \notag
        \underset{\begin{subarray}{c} s \sim D^{\pi} \\ a \sim \pi' \end{subarray}}{\mathbb{E}} \left[ A_i^{\pi}(s,a) \right] =
        \alpha g_i^Td = \alpha_\theta \delta_i
    \end{equation}
    For the second item, suppose we have the update distance $\operatorname{KL}(\pi\|\pi') \leq \eta$, we have the following relationship with Pinsker's inequality, 
    \begin{equation} \notag
        \operatorname{TV}(\pi', \pi)(s) \leq \sqrt{\frac{\eta}{2}}
    \end{equation}
    Then, we have the lower bound for improvement, 
    \begin{equation} \notag
        J_i(\pi') - J_i(\pi) \geq \frac{\alpha_\theta \delta_i}{1-\gamma} - \frac{2\gamma C_i^{\pi',\pi}}{(1-\gamma)^2}\sqrt{\frac{\eta}{2}}
    \end{equation}
    and we complete the proof. 
\end{proof}

\end{document}